\documentclass[10pt]{article}
\usepackage[ngerman, english]{babel}
\usepackage[utf8]{inputenc}
\usepackage[T1]{fontenc}
\usepackage[gen]{eurosym}
\usepackage{geometry}
\usepackage{graphicx}
\usepackage{amsmath, amsfonts, amssymb}
\usepackage{amsthm}
\usepackage{mathrsfs}
\usepackage{xcolor}
\usepackage{lmodern}
\usepackage{mathtools}
\usepackage{framed}
\usepackage{relsize}
\usepackage{exscale}
\usepackage{tabularx}
\usepackage{caption}
\usepackage{subcaption} 
\usepackage{microtype} 
\usepackage{a4}        
\usepackage{setspace} 
\usepackage{verbatim}   
\usepackage{enumerate}
\usepackage{enumitem}   
\usepackage{tocloft}
\usepackage{float}
\usepackage{pifont}
\usepackage{longtable}
\usepackage{booktabs}
\usepackage{bbm}
\usepackage[]{authblk}
\usepackage{fancyhdr}
\usepackage{makecell}

\definecolor{ubtgreen}{RGB}{0,146,93}

\definecolor{lightgray}{RGB}{245,245,245}

\newcommand{\anf}[1]{\enquote{#1}}

\newcommand{\R}{\mathbb{R}}

\newcommand{\N}{\mathbb{N}}

\newcommand{\XY}{X \times Y}

\DeclareMathOperator{\e}{e}

\newcommand{\Lcal}{\mathcal{L}}
\newcommand{\B}  {\mathcal{B}}   

\newcommand{\AB}{\mathcal{A} \otimes \mathcal{B}}

\newcommand{\RLP}{\mathcal{R}_{L,P}}
\newcommand{\RLD}{\mathcal{R}_{L,D}}

\newcommand{\RregLP}{\mathcal{R}^{\mathrm{reg}}_{L,P,\lambda}}
\newcommand{\dd}{\mathrm{d}}

\newcommand{\xmark}{\textbf{--}}
\newcommand{\email}[1]{\href{mailto:#1}{#1}}

\newtheorem{theorem}{Theorem}[section]
\newtheorem{lemma}[theorem]{Lemma}
\newtheorem{remark}[theorem]{Remark}
\newtheorem{proposition}[theorem]{Proposition}
\newtheorem{example}[theorem]{Example}

\newtheorem{corollary}[theorem]{Corollary}
\newtheorem{definition}[theorem]{Definition}

\usepackage[backend = biber, style = alphabetic, sorting = nyt, maxnames=10]{biblatex}
\usepackage{csquotes}
\DeclareFieldFormat{postnote}{#1}

\usepackage[hidelinks]{hyperref}
\usepackage{cleveref}

\title{\textbf{Ratio-based Loss Functions}}
\author[1]{Lena Helgerth\thanks{corresponding author, e-mail: \email{lena.helgerth@uni-bayreuth.de}}}  
\author[1]{Andreas Christmann}
\affil[1]{Department of Mathematics, University of Bayreuth, Chair of Stochastics and Machine Learning, 95440 Bayreuth, Germany}

\allowdisplaybreaks

\begin{document}
\selectlanguage{english}

\pagestyle{headings}

\maketitle

\textbf{Abstract~} {\small 
	Algorithms in machine learning and AI do critically depend on at least three key components: (i) the risk function, which is the expectation of the loss function, (ii) the function space, which is often called the hypothesis space, and (iii) the set of probability measures, which are allowed for the specified algorithm. 
	This paper gives a survey of a certain class of loss functions, which we call ratio-based. 
	In supervised learning, margin-based loss functions for classification tasks depending on the \emph{product} of the output values $y_i$ and the predictions $f(x_i)$  as well as distance-based loss functions depending on the \emph{difference} of $y_i$ and $f(x_i)$ for regression are common.
	Distance-based loss functions are in particular useful, if an additive model assumption seems plausible, i.e. the common signal plus noise assumption.
	However, in the literature, several loss functions proposed for regression purposes have a multiplicative error structure in mind and pay attention to relative errors, i.e. to the \emph{ratio} of $y_i$ and $f(x_i)$.
	In this survey article, we systematically investigate such ratio-based loss functions and propose a few new losses, which may be interesting for future research. 
	We concentrate on investigating general properties of ratio-based loss functions like continuity, Lipschitz-continuity, convexity, and differentiability, because these properties play a central role in most machine learning algorithms.
	Therefore, we do not focus on some specific machine learning algorithm to derive universal consistency, learning rates, or stability results.
	Instead, we want to enable future research in this direction.
}

\section{Introduction} \label{Kap_Intro}
Machine learning (ML) and AI are well established in modern life.
Besides choosing an appropriate function space or hypothesis space $\mathcal{F}$ and a set of probability measures $\mathcal{P}$, the choice of the loss function $L$ is a crucial component of ML. Of course, many algorithms also need a good determination of hyperparameters such as a regularization parameter or the depth of a decision tree or a random forest, or the depth and the structure of
a deep neural network.
Two main goals of ML methods are (i) minimizing the risk or the regularized risk over $\mathcal{F}$ and (ii) finding the optimal function $f^\star \in \mathcal{F}$.
Here, we focus on the loss function for supervised learning.
Denote the given data set by $D:=((x_1,y_1), \ldots, (x_n,y_n))\in (\XY)^n$, where $n\in\N$ is the sample size, $x_i$ denotes an input value in the set $X$, and $y_i$ denotes the corresponding output value in the set $Y\subset \R$.
As predictive accuracy plays a key role when comparing ML methods, loss functions, i.e. measurable mappings $L: \XY \times \R \to [0, \infty)$, quantify how good a prediction $t \in \R$ captures an input's $x \in X$ output value $y \in Y$.
The higher the loss, the poorer the prediction.
In many machine learning methods minimizing the expected loss, i.e. the risk, is a central component.  
The input space $X$ is typically required to be a complete separable metric space or even a Polish space, whereas the output space $Y$ is often assumed to be a subset of $\R$. 
Because loss functions are a critical component of ML methods, there is a variety of survey articles (see e.g. \cite{TianEtAl2022, TervenEtAl2025, CiampiconiEtAl2024}). \par
It is a useful and well-known fact that general properties of the loss function like continuity, Lipschitz continuity, convexity, or differentiability are inherited by the risk, i.e. the expectation of the loss function, under suitable conditions, see e.g. \cite[Chapter 2.2]{SC08}. \par 
In binary classification, we have $Y = \{-1,1\}$ and common loss functions for such methods are \textbf{margin-based}, i.e. there exists a representing function $\varphi: \R \to [0, \infty)$ that fulfills
\begin{equation*}
	L(x,y,t) = \varphi(y \cdot t),  \qquad (x,y,t) \in \XY \times \R.
\end{equation*}
Popular examples for margin-based loss functions are hinge loss $L_{\mathrm{hinge}}(x,y,t)=\max\{0, 1-y\cdot t\}$, logistic loss $L_{\mathrm{c-logist}}(x,y,t)=\ln(1+\exp{(y\cdot t)})$, and least squares loss $L_{\mathrm{LS}}(x,y,t) = (1-y\cdot t)^2$.
\par
In regression, it is common to assume $Y = \R$ and use a \textbf{distance-based} loss function, i.e. $L$ can be written as 
\begin{equation*}
	L(x,y,t) = \psi(y-t),  \qquad  (x,y,t) \in \XY \times \R,
\end{equation*}
where the representing function satisfies $\psi: \R \to [0, \infty)$ with $\psi(0) = 0$.
Distance-based loss functions are of particular usefulness if an \emph{additive} error structure is assumed, i.e. if the classical signal \emph{plus} random noise assumption seems plausible. Obviously, distance-based loss functions satisfy $L(x,y,t) = L(x,y+c, t+c)$ for all $(x,y,t) \in \XY \times \R$ and $c \in \R$.
Hence, such loss functions and their risk functions are invariant w.r.t. additive shifts by some constant $c$.
Examples for distance-based loss functions are the least squares loss $L_{\mathrm{LS}}(x,y,t)=(y-t)^2$ and the absolute error loss $L_{\mathrm{abs}}(x,y,t)=|y-t|$ both for regression and the pinball loss $L(x,y,t)=\max\{\tau(y-t), -(1-\tau)(y-t)\}$ for $\tau$-quantile regression, where $\tau\in(0,1)$ is fixed. 
\par
However, in many real-world regression tasks, the output space $Y$ is often a subset of $\R$, e.g. $Y=(0,\infty)$.  
Furthermore, a \emph{multiplicative} error model, i.e. signal \emph{times} random noise, is often at least as plausible as a signal \emph{plus}  random noise assumption. 
This gives the motivation to systematically investigate loss functions which depend on the \emph{ratio} of predicted and observed values instead of their difference. 
Especially for multiplicative models, relative errors are widely used (see e.g. \cite{ChenEtAl2010LARE, ChenEtAl2016LPRE, ZhangEtAl2021}).
A perfect ratio-based approach gives a scale-invariant version of predictive accuracy.
Note that this property cannot be ensured through an appropriate distance-based approach as long as a constant loss function $L\equiv 0$ (which is totally uninteresting in ML) is neglected.
For this, assume $L: X \times \R \times \R \to [0, \infty)$ is a distance-based loss function satisfying $L(x,y,t) = L(x,c y, c t)$ for all $(x,y,t) \in X \times \R \times \R$ and $c \in \R$, i.e. $L$ is scale-invariant.
Let $\psi$ denote its representing function. 
Then, 
\begin{equation*}
	\psi(y-t) = \psi(c (y-t)) , \qquad (x,y,t) \in X \times \R \times \R, c \in \R.
\end{equation*}
Therefore, $\psi\equiv 0$ is a constant function. \par 
In a scale-invariant approach, evaluation becomes independent of absolute output values and its units.
\cite[(2)]{ChenEtAl2010LARE} and \cite[(1)]{ChenEtAl2016LPRE} consider \anf{accelerated failure time models}, which reduce to additive regression models after a logarithmic transformation.
Therefore, these models only handle positive data (cf. \cite[p. 2]{ChenEtAl2010LARE}).
Later, the question arises whether transformation to the multiplicative case is necessary. 
We will see that back-transformation to the additive scenario is possible but does not give any additional information. 
Since results also need to be transformed back, one either gets information about the difference or about their ratio, but not both. \par 
Let us now briefly motivate why ratio-based loss functions can be of interest in applications. Medical fields as well as financial modeling or quality control are just a few examples in which a ratio-based error can be relevant and useful. This is one reason why logistic regression, odds ratio estimation  in dose-response studies, and Poisson regression to model count data are often used in practice.
E.g., when predicting a person's weight, absolute deviations often provide limited interpretation compared to relative ones.
If the aim of prediction is to calculate a dose of medicine, a deviation of 2 kg from the true value may be irrelevant for a 80 kg grown-up man's dose. But in case of an 8 kg child, whose weight was measured as 10 kg, the medication dose was 25 \% too high.
Here, the ratio is probably of greater interest than the difference. \par 
Speaking of inflation, for a daily newspaper, which cost \euro 3 before, an increase of \euro 5 would be unacceptable for most customers and have a major impact on their buying decision.
For a car instead, lifting the price by \euro 5 would not make a noticeable impact on the customer's buying decision. \par	
Therefore, ratio-based methods are an important addition to traditional distance-based techniques to measure predictive accuracy.\par 
Speaking of ratios, one usually intuitively assumes positive values because of interpretability issues, although mathematically observed values and predicted values may both be negative in some applications. 
Trivially, if both the observed values and the predicted values are negative, we can just multiply by $-1$.
Moreover, we have already seen before that in many real-world applications only positive outputs are possible, thinking of body height, body weight, measuring of time, prices, damage costs, fuel consumption of cars, etc.
This is why we do not use $Y = \R$ as an output space in this paper, but a real subset of it.
To guarantee the former and to avoid dividing by zero, we take $Y$ as a part of the positive real line.
Examples for positive outputs are also used to explain the significance of log-normal distribution or rather its advantages over the normal distribution in scenarios where only positive outputs can be expected (cf. \cite{LimpertStahel2001, LimpertStahel2017}).\par
The paper is organized as follows. 
Section \ref{Kap_Def} introduces the class of ratio-based loss functions.
Section \ref{Kap_Eigenschaften} presents the main results explaining and systematically analyzing how certain properties of the loss function can be achieved. 
Section \ref{Kap_Risiko} transfers these properties to the risk.
In Section \ref{Kap_Bsp}, we give some examples for ratio-based losses from the literature as well as some new ones and investigate their properties. 
At the end, in Sections \ref{Kap_Distance} and \ref{Kap_Alternative} we shortly discuss ratio-based loss functions' connection to distance-bases loss functions and a further approach to ratio-based.
Some basic proofs are given in the appendix.

\section{Definition} \label{Kap_Def}
Let us first define ratio-based loss functions. 
If not otherwise specified, in the remainder  of the paper $(X, \mathcal{A})$ denotes a measurable space and $Y \subseteq \R$, while $\R$ and all of its subsets are equipped with the Borel $\sigma$-algebras denoted by $\B=\B(\R)$ and
$\B(Y)$, respectively.
\begin{definition} \label{Def_ratiobasedLoss}
	A supervised loss function $L: \XY \times \R \to [0, \infty)$ is called \textbf{ratio-based (rb)} if there exists a representing measurable function $\ell: (0, \infty) \to [0, \infty)$ as well as a constant $c\geq 0$, and a link function $u:\R \to Y$ that is measurable, surjective, and monotonic, such that $\ell(1) = 0$ and 
	\begin{equation} \label{ratiobased}
		L(x,y,t) = \ell\biggl( \frac{u(t) + c}{y + c} \biggr), \qquad (x,y,t) \in \XY \times \R,	
	\end{equation}
	hold, where $Y = (a,b), 0 \leq a < b \leq \infty$. 
	Moreover, we call $L$ \textbf{strictly ratio-based} if $c=0$.
\end{definition}
The constant $c$ helps to guarantee that the ratio is in $\bigl(\frac{a+c}{b+c}, \frac{b+c}{a+c}\bigr)$.
Since many statistical learning algorithms use Hilbert spaces or Banach spaces of functions $f:X\to\R$ to predict $\hat{y}$ for $y$, one can in general not prevent the third variable $t=f(x)$ from being negative.
Therefore, it is necessary to transform the predicted values to $Y$ to be able to compare them accurately.
The link function transforms real values appropriately to $Y$.
Some examples for such link functions are given in Table \ref{table_u}. \par
\begin{table}[ht]
	\centering
	\renewcommand{\arraystretch}{1.2}
	\caption{Examples for link functions $u$ depending on output space $Y$}
	\begin{tabular}{l|l}
		$Y$ & $u(t)$ \\
		\midrule
		$(a, \infty)$, $0 \leq a < \infty$ & $\exp(t) + a$ \\
		& $\exp(-t) + a$ \\
		$(a,b)$, $0 \leq a < b < \infty$ & $(b-a)\frac{1}{1+\exp(-t)} + a$, \\
		& $(b-a)(\frac{1}{2} + \frac{1}{\pi}\arctan(t))+a$\\  
		& $(b-a)\exp(-\exp(-t))+a$ 
	\end{tabular}
	\label{table_u}
\end{table} 
The idea to define a loss function via this relation has also been used before.
For example, \cite[Chapter III.D]{JadonEtAl2022} define \textit{Relative Absolute Error (RAE)} through
\begin{equation*}
	RAE = \frac{\sum_{i = 1}^N |y_i-\hat{y}_i|}{\sum_{i = 1}^N \Bigl|y_i - \frac{1}{N} \sum_{i =1}^N y_i\Bigr| },
\end{equation*}
where $\hat{y}_i=u(f(x_i))$. 
In this context, the question arises whether RAE is well-defined. 
In case $y_i = K$ for all $i \in \{1, \dots, N\}$, RAE can be undefined (\cite[Tab. IV]{JadonEtAl2022}).
\cite{TervenEtAl2025} also considered loss functions which use certain ratios, e.g. \textit{Mean Absolute Relative Error (AbsRel)} 
\begin{equation*}
	AbsRel = \frac{1}{N} \sum_{i = 1}^N \frac{|\hat{y}_i-y_i|}{y_i}
\end{equation*}
(\cite[Chapter 8.2.1]{TervenEtAl2025}) and \textit{Logarithmic RMSE (LRMSE)}
\begin{equation*}
	LRMSE = \sqrt{\frac{1}{N}\sum_{i = 1}^N \bigl(\log(\hat{y}_i)-\log(y_i)\bigr)^2}
\end{equation*}
(\cite[Chapter 8.2.3]{TervenEtAl2025}) as well as \textit{Mean Log10 error}
\begin{equation*}
	Mean Log10 = \frac{1}{N} \sum_{i=1}^{N} \bigl|\log_{10}(\hat{y}_i)-\log_{10}(y_i)\bigr|
\end{equation*}
(\cite[Chapter 8.2.5]{TervenEtAl2025}). 
The last ones can be expressed through a ratio because of the logarithm's rules. 
A similar version of LRMSE is also discussed in \cite[Chapter 3.1.8]{LiEtAl2025}.
Aside from that, \cite{ChenEtAl2016LPRE} discuss ratio-based approaches for their \textit{least absolute relative error (LARE)} (\cites[(3)]{ChenEtAl2010LARE}[(2)]{ChenEtAl2016LPRE}), \textit{least product relative error (LPRE)} (\cite[(3), (4)]{ChenEtAl2016LPRE}), and \textit{general relative error (GRE)} (\cite[(8)]{ChenEtAl2016LPRE}). \par
Speaking of distance-based loss functions, symmetry is sometimes of concern. 
Considering a dis\-tance\--based loss with representing function $\psi$, symmetry is expressed via $\psi(r) = \psi(-r)$, see e.g. \cite[Def. 4.11 (ii)]{Steinwart2007}. Such functions penalize overestimation ($\hat{y} = y + c$) and underestimation ($\hat{y} = y - c$) of $y$, where $\hat{y}=f(x)$, by some $c \geq 0$ equally, because $L(x,y,t-c) = L(x,y,t+c)$ for all $(x,y,t) \in \XY \times \R$ and $c \in \R$.
Therefore, in a ratio-based setting, it makes sometimes sense to use a loss function such that overestimation $\hat{y} = \lambda y$ and underestimation $\hat{y} = \frac{y}{\lambda}$, where $\hat{y}=u(f(x))$ and $\lambda \geq 1$, yields the same loss.
This can be obtained by requiring 
\begin{equation*} 
	\ell(r) = \ell(r^{-1}), \qquad r \in (0, \infty).
\end{equation*}
This idea was already discussed in \cite{LimpertStahel2001}.
We call this property \textbf{ratio-symmetry}.
Roughly speaking, a ratio-based loss function is defined to realize the idea $L(x,y,t) = L(x,\lambda y, \lambda t)$ for all $(x,y,t) \in \XY \times \R$ and $\lambda > 0$ (modulo modification of the prediction through the link function), whereas ratio-symmetry ensures $L(x,y,\lambda t) = L(x,y,\frac{1}{\lambda}t)$ for all $(x,y,t) \in \XY\times \R$ and $\lambda > 0$.
\cite{ChenEtAl2010LARE, ChenEtAl2016LPRE} already mentioned that a loss without one type of relative error could lead to biased estimation and one should, therefore, take both errors into account, one relative to the response and one relative to its prediction.
Clearly, ratio-symmetry for ratio-based loss functions can be useful in some applications and not be useful in others, as is true also for symmetric distance-based loss functions. 
Thinking of loans, for the lending bank, it is worse to not get their money back, once they gave someone credit, than to not give someone credit, who would have paid all his money back (cf. \cite[credit-scoring]{FahrmeirHamerleTutz1996}).
Moreover, \anf{it is [\dots] worse to predict that a person will not have a heart attack when he or she actually will, than vice versa} (\cite[Chapter 9.2.4]{HastieTibshiraniFriedman2017}).
In these cases, one does not wish a ratio-symmetric loss. \par 
In binary classification with output space $Y = \{-1, 1\}$, a margin-based loss function
\begin{equation*}
	L(x,y,t) = \varphi(yt), \qquad (x,y,t) \in \XY \times \R,
\end{equation*}
with representing function $\varphi: \R \to [0, \infty)$ can be expressed in a ratio-based way.
Since $yt = \frac{t}{y}$ holds for $y \in \{-1,1\}$ and $t \in \R$,
\begin{equation*}
	L(x,y,t) = \varphi\Bigl(\frac{t}{y}\Bigr), \qquad (x,y,t) \in \XY \times \R.
\end{equation*}
For that reason, all margin-based loss functions can be presented through a ratio-based approach if we do not take every specification from Definition \ref{Def_ratiobasedLoss} into account.

\section{Properties of ratio-based loss functions} \label{Kap_Eigenschaften}
In this section, we systematically investigate properties of ratio-based loss functions and give some basic examples. These results will be used to investigate properties of the 
corresponding risk functional in Section \ref{Kap_Risiko}.\par
For distance-based loss functions, it is well known that properties of the representing function $\psi$ determine properties of the loss function and the risk function to a great extent.  	
Let us now investigate, whether this is true for ratio-based loss functions, too.\par
If not otherwise mentioned, $L$ is a ratio-based loss function (for regression) in this section. 
As is common in the literature, we say that a rb loss function $L$ has property 
$A$ if $L$ has this property w.r.t. the third argument uniformly for all $(x,y)\in\XY$ (cf. \cite[Chapter 2.2]{SC08}).

\subsection{Continuity}
A distance-based loss function is continuous if and only if the representing function $\psi$ is continuous.
For ratio-based loss functions it is in general necessary to also claim continuity of $u$. 
\begin{lemma} \label{LemmaLstetig}
	Let $L$ be a ratio-based loss function. If the functions $\ell$ and $u$ are continuous, then $L$ is a continuous loss function.
\end{lemma}	
Of course, this is only sufficient, as e.g. if $u$ is measurable but discontinuous and $\ell \equiv 0$, then $L$ is still a continuous loss function.

\subsection{Differentiability} \label{chapter_diff}
If all of the considered derivatives exist, with the common notation $L'(x,y,t)$ for the partial derivative $\frac{\partial}{\partial t}L(x,y,t)$, a straightforward calculation yields that 
\begin{equation*} 
	L'(x,y,t) = \ell'\biggl(\frac{u(t)+ c}{y + c}\biggr) \frac{u'(t)}{y + c}, \qquad (x,y,t) \in \XY \times \R.
\end{equation*}
\begin{lemma} \label{LemmaLdiff}
	Let $L$ be a ratio-based loss function. 
	Given two differentiable functions $\ell$ and $u$, $L$ is a differentiable loss function.
\end{lemma}

\subsection{Convexity} 
Convexity is a strong, but mathematically very nice condition for loss functions, as it implies convexity of the risk functional independent of the probability measure, see e.g. \cite[Lem. 2.13]{SC08}. 
The convexity of the risk functional is often useful to show uniqueness of the learning algorithm, as is well-known e.g. for many regularized kernel based methods. \par
A twice differentiable ratio-based loss function with twice differentiable $\ell$ and $u$ has second derivative (w.r.t. $t$)
\begin{equation*} 
	L''(x,y,t) = \ell''\biggl( \frac{u(t) + c}{y + c}\biggr) \cdot \biggl(\frac{u'(t)}{y + c}\biggr)^2 + \ell'\biggl(\frac{u(t) + c}{y + c}\biggr) \cdot \frac{u''(t)}{y + c} ~.
\end{equation*}
Hence, the assumption that $\ell$ and $u$ are convex functions is -- even under this smoothness assumption on $\ell$ and $u$ -- in general \emph{not} sufficient to guarantee a convex rb loss, as $\ell(r)$ should decrease for $r < 1$. 
From an applied point of view, it does not make sense to assume an increasing function $\ell$, since $\ell(1)=0$ and underestimation, i.e. $\ell(r)$ with $r \ll 1$, should in general yield positive losses.
Furthermore, the assumption of $u$ being convex is not automatically fulfilled, e.g. for logistic distribution function.
In fact, for $Y = (0, 1)$ with corresponding link function $u$ no matter of $c$, a non-constant convex ratio-based loss $L$ is not even possible.
The main reason for that is the following. 
Since $u$ is monotone and surjective, either $\lim_{t\to \infty} u(t) = 1$ or $\lim_{t \to -\infty} u(t) = 1$ holds. 
Without loss of generality, we assume that the first equation holds.
Then, for continuous $\ell$, $\lim_{t \to \infty} L(x,y,t) = \ell \Bigl(\frac{1+c}{y+c}\Bigr) < \infty$ for all $(x,y) \in \XY$, $c \geq 0$.
Under these assumptions, if $c > 0$, $\lim_{t \to -\infty} L(x,y,t) = \ell\Bigl(\frac{c}{y+c}\Bigr) < \infty$, too.
\begin{proposition} \label{PropNoConvexLoss}
	Let $L$ be a non-constant rb loss function. 
	We put that into concrete terms assuming additionally the existence of $\tilde{t} < t < t'$ with $L(x,y,t) < L(x,y,t')$ and $L(x,y,t) < L(x,y,\tilde{t})$. 
	Let $Y = (0,1)$ and consider logistic link function $u$. 
	Then, $L$ is not a convex loss function.
\end{proposition}
The same holds true for intervals  $Y = (a, \infty), a > 0$, or $Y =(a,b), 0 \leq a < b$, and continuous functions $\ell$ as either $\lim_{t \to -\infty}L(x,y,t)$ or $\lim_{t \to \infty}L(x,y,t)$ is bounded which gives a contradiction to the loss function's convexity no matter whether $c$ equals zero or not.
Note that convex functions on $\R$ are continuous automatically (\cite[Cor. 10.1.1]{Rockafellar1997}).
Moreover, when $Y = (0, \infty)$ with $c > 0$, again, a convex loss function is not possible since $\lim_{t \to -\infty} L(x,y,t)$ is bounded for $(x,y) \in \XY$. \par 
When using parameter $c = 0$ and link function $u(t) = \e^t$ from $\R$ to $Y = (0, \infty)$, a convex loss is, at least, possible.
Therefore, we will focus on such strict rb loss functions for the moment.
\begin{proposition} \label{Prop_konvexeVerlustfkt_tildeEll}
	Let $L$ be a ratio-based loss function.
	Let $Y = (0, \infty), u(t) = \e^t$, and $c = 0$.
	Define
	\begin{equation*}
		\ell(r) := \widetilde{\ell}(r) + \widetilde{\ell}(r^{-1})-2\widetilde{\ell}(1)
	\end{equation*}
	with $\widetilde{\ell}: (0, \infty) \to [0, \infty)$ being a twice differentiable function satisfying
	\begin{equation} \label{Bedingung_Konvexität_tildeell}
		\widetilde{\ell}'(r) + r\widetilde{\ell}''(r) \geq 0
	\end{equation}
	for all $r \in (0, \infty)$.
	Then, 
	\begin{equation} \label{convexLoss}
		L(x,y,t) = \ell\Bigl(\frac{\e^t}{y}\Bigr) = \widetilde{\ell}\Bigl(\frac{\e^t}{y}\Bigr)+\widetilde{\ell}\Bigl(\frac{y}{\e^t}\Bigr) - 2\widetilde{\ell}(1)
	\end{equation}
	is a convex ratio-based loss function.
\end{proposition}
Note that $u' = u$ yields easy calculations here and that the statement can be improved to strict convexity by claiming strict inequality in \eqref{Bedingung_Konvexität_tildeell}.
\begin{example}
	In the context from Proposition \ref{Prop_konvexeVerlustfkt_tildeEll},  $\widetilde{\ell}(r):= r^\alpha, \alpha \geq 0$ satisfies
	\begin{equation*}
		\widetilde{\ell}'(r) + r\widetilde{\ell}''(r) = \alpha r^{\alpha -1} + r \alpha(\alpha-1)r^{\alpha-2} = \alpha^2 r^{\alpha-1} \geq 0
	\end{equation*}
	for all $r \in (0, \infty)$.
	Thus, using Proposition \ref{Prop_konvexeVerlustfkt_tildeEll}, we get the  convex ratio-based loss function
	\begin{equation*}
		L_\alpha(x,y,t) := \Bigl(\frac{\e^t}{y}\Bigr)^\alpha + \Bigl(\frac{y}{\e^t}\Bigr)^\alpha - 2.
	\end{equation*}
\end{example}
If $\alpha = 0$, the loss function is constantly zero. 
This is, of course, the trivial case and not interesting for applications.
The special case  $\alpha = 1$ will be discussed in Chapter \ref{Kap_Bsp}.
Interestingly, $\ell$ itself does not need to be convex.
Especially, if $\alpha < 1$,
\begin{equation*}
	\ell''(r) = \alpha r^{-2}\bigl((\alpha-1)r^\alpha + (\alpha+1)r^{-\alpha}\bigr), \; r \in (0, \infty).
\end{equation*}
Choosing $r = \biggl(2\sqrt{\frac{1+\alpha}{1-\alpha}}\biggr)^{\frac{1}{\alpha}} > 0$ gives $\ell''(r) < 0$. 
This means, $\ell$ is not convex for $\alpha < 1$.
Instead, $\ell$ is convex if $\alpha \geq 1$.
Nevertheless, the corresponding loss function $L$ is always convex. \par
We are now interested in a certain choice of $\widetilde{\ell}$ to ensure the claims of Proposition \ref{Prop_konvexeVerlustfkt_tildeEll}.
\begin{proposition} \label{Prop_BspKonvexeVerlustfkt_Integral}
	Let $L$ be a ratio-based loss function. 
	Let $Y = (0, \infty), u(t) = \e^t$, and $g \in C^1((0, \infty))$ with $g$ being an increasing function.
	Then,
	\begin{equation*}
		\widetilde{\ell}(r) := C + \int_{r_0}^{r} \frac{g(t)}{t} \; \mathrm{d}t, \qquad r \in (0, \infty)
	\end{equation*}
	with $r_0 \geq 0$ and $C \in \R$ such that $\widetilde{\ell}$ is well-defined, non-negative, and finite, fulfills \eqref{Bedingung_Konvexität_tildeell} for all $r \in (0, \infty)$. 
	Thus, 
	\begin{equation*}
		L(x,y,t) = \widetilde{\ell}\Bigl(\frac{\e^t}{y}\Bigr) + \widetilde{\ell}\Bigl(\frac{y}{\e^t}\Bigr) - 2\widetilde{\ell}(1)
	\end{equation*}
	is a convex loss function.
\end{proposition}
\begin{example} \label{Bsp_log(1+r)}
	In the notation from Proposition \ref{Prop_BspKonvexeVerlustfkt_Integral},  $\widetilde{\ell}(r) := \log(1+r)$ gives a convex rb loss function. 
	Integrating $g(t) = \frac{t}{t+1}$ on the open interval $(0,r)$ yields
	\begin{equation*}
		\int_0^r \frac{g(t)}{t} \; \mathrm{d}t = \int_{0}^{r} \frac{1}{t+1} \; \mathrm{d}t = \log(1+r) = \widetilde{\ell}(r).
	\end{equation*}
	Here, $g$ is a differentiable and increasing function on $(0, \infty)$.
	Because of Proposition \ref{Prop_BspKonvexeVerlustfkt_Integral} we obtain that  
	\begin{equation*}
		L(x,y,t) = \log\Bigl(1+\frac{\e^t}{y}\Bigr) + \log\Bigl(1+\frac{y}{\e^t}\Bigr) - 2\log(2)
	\end{equation*}
	is a convex ratio-based loss function with minimum $L(x,y,\log(y)) = 0$ in $t = \log(y)$. 
	Furthermore, $L(x,y,\cdot)$ is symmetric around $\log(y)$.
	Additionally, $L$ can be written in a distance-based manner.
	Using the logistic loss function (cf. e.g. \cite[Ex. 2.40]{SC08})
	\begin{equation*}
		\psi(r) = 2\log(1+\e^{-r})+r-2\log(2) = \log\Bigl(\frac{(1+\e^r)^2}{4\e^r}\Bigr), \qquad \psi(0) = 0,
	\end{equation*}
	one can write $L(x,y,t) = \psi(\log(y)-t)$.
	Note that $\ell$ also yields a similar representation, i.e.
	\begin{equation} \label{log(1+r)Loss}
		\ell(r) = \log\biggl(\frac{(1+r)^2}{4r}\biggr).
	\end{equation}
	\begin{figure}[ht]
		\centering
		\begin{subfigure}{0.31\textwidth}
			\centering
			\includegraphics[width=\textwidth]{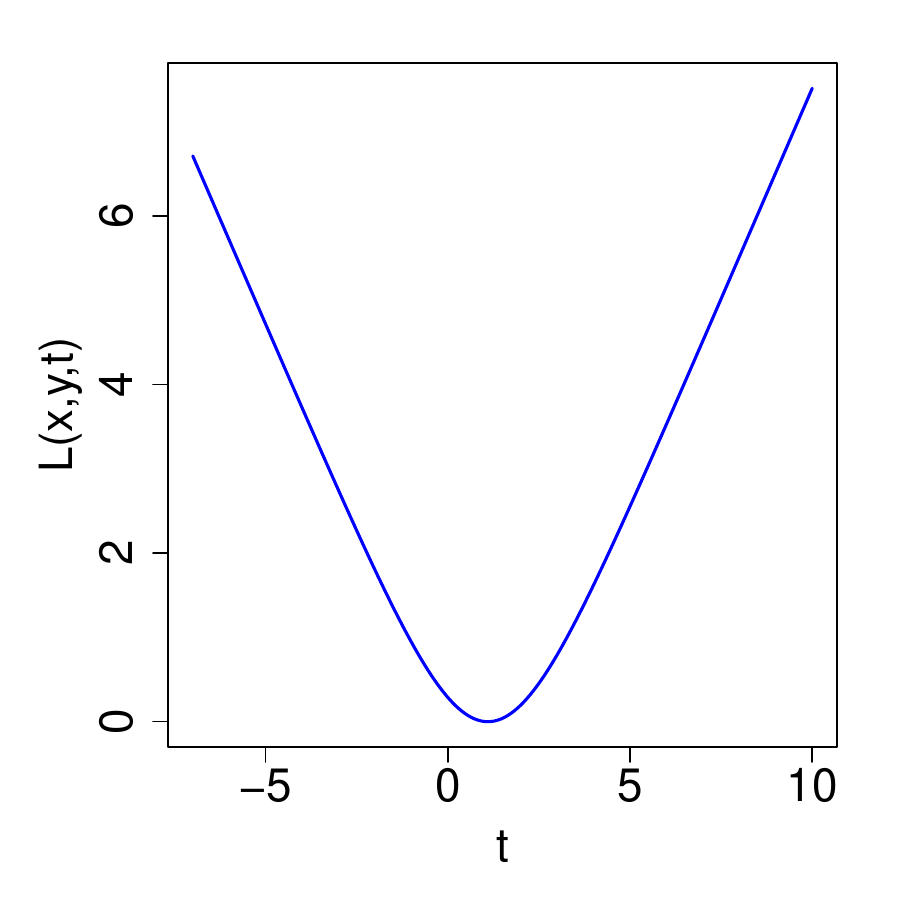}
			\caption{Loss function $L$ in $t$ for $y = 3$}
		\end{subfigure}
		\hspace*{1cm}
		\begin{subfigure}{0.31\textwidth}
			\centering
			\includegraphics[width=\textwidth]{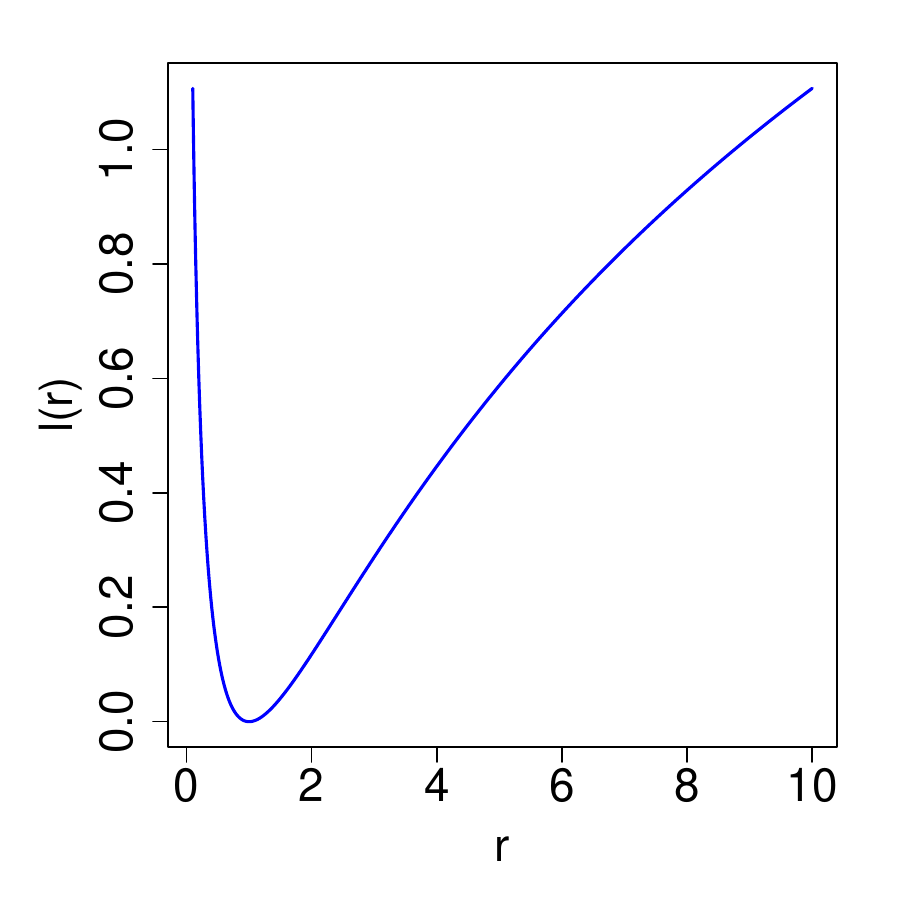}
			\caption{Representing function $\ell$}
		\end{subfigure}
		\caption{Convex and Lipschitz continuous loss function $L$ (left) with representing function $\ell$ (right) from Example \ref{Bsp_log(1+r)}}
	\end{figure}
\end{example}
\begin{remark}
	Using either $g(t) = \frac{t}{t+1}+\frac{t}{(1+t)^2}$ with $\widetilde{\ell}(r) = \log(1+r)+\frac{r}{r+1}$ or $g(t) = (\frac{t}{t+1})^2$ with $\widetilde{\ell}(r) = \log(1+r)+\frac{1}{r+1}$, one can apply Proposition \ref{Prop_BspKonvexeVerlustfkt_Integral} as well. 
	Note that these auxiliary functions produce the same loss function as Example \ref{Bsp_log(1+r)}. 
\end{remark}
\begin{example} \label{bsp_konvex_Lipschitz_3}
	$g(t) = \frac{\sqrt{t}}{2\sqrt{t+1}}$ is a differentiable, increasing function on $(0, \infty)$ which satisfies $0 \leq g(t) \leq \frac{1}{2}$ for all $t \in (0, \infty)$.
	Hence, rb loss function
	\begin{equation*}
		L(x,y,t) = \log \biggl(\sqrt{\frac{\e^t}{y}} + \sqrt{1+\frac{\e^t}{y}}\biggr) + \log\biggl(\sqrt{\frac{y}{\e^t}} + \sqrt{1+\frac{y}{\e^t}}\biggr) - 2 \log\Bigl(1 + \sqrt{2}\Bigr),
	\end{equation*}
	based on 
	\begin{equation*}
		\widetilde{\ell}(r) := \int_{0}^{r} \frac{1}{2\sqrt{t(t+1)}} \; \mathrm{d}t = \log(\sqrt{r} + \sqrt{1+r}),
	\end{equation*}
	which yields representing function
	\begin{equation}\label{bsp_3_Loss}
		\ell(r) = \log\biggl(\frac{(1+\sqrt{1+r})(\sqrt{r}+\sqrt{1+r})}{(3+2\sqrt{2})\sqrt{r}}\biggr),
	\end{equation}
	is a convex rb loss according to Proposition \ref{Prop_BspKonvexeVerlustfkt_Integral} which differs from the previous one.
	$L(x,y,\cdot)$ attains its minimum $L(x,y,\log(y))=0$ in $t = \log(y)$.
	Moreover, $L(x,y,\cdot)$ is symmetric around it.
	In a distance-based approach, $L(x,y,t) = \psi(\log(y)-t)$ holds where
	\begin{equation*}
		\psi(r) = \log(\sqrt{\e^r}+\sqrt{1+\e^r}) + \log(\sqrt{\e^{-r}}+\sqrt{1+\e^{-r}}) - 2\log(1+\sqrt{2}), \qquad \psi(0) = 0.
	\end{equation*}
\end{example}
\begin{corollary}
	Let $L$ be a ratio-based loss function.
	Let $Y = (0, \infty)$, $u = \exp$, $c = 0$, and $g \in C^1((0, \infty))$ increasing.
	Let 
	\begin{equation*}
		\ell(r) := C + \int_{r_0}^r \frac{g(t)}{t} \; \dd t, \qquad r \in (0, \infty)
	\end{equation*}
	with $r_0 \geq 0$, $C \in \R$ such that $\ell$ is well-defined, non-negative, finite, ratio-symmetric, and $\ell(1) = 0$.
	Then, we get a convex rb loss function
	\begin{equation*}
		L(x,y,t) = \ell\biggl(\frac{\e^t}{y}\biggr) .
	\end{equation*}
\end{corollary}

\subsection{Lipschitz Continuity}
For the sake of convexity, one could think about replacing $Y = (0,1)$ by its superset $Y = (0,\infty)$ and using $c = 0$.
However, choosing such output space can be disadvantageous in terms of Lipschitz continuity.
\begin{lemma} \label{LipschitzL} 
	Let $L$ be a ratio-based loss function. 
	Let $\ell$ and $u$ be Lipschitz-continuous and $c \geq 0$. 
	Consider the output space $Y = (a, b)$, $0 \leq a < b \leq \infty$  with $a + c > 0$. 
	Then, rb loss \eqref{ratiobased} is a Lipschitz continuous loss function with Lipschitz constant $|L|_1 \leq \frac{|\ell|_1 |u|_1}{a+c}$.
\end{lemma}
Here, both of the critical values $a = 0$ and $b = \infty$ are possible.
In case $a = 0$, we only require $c > 0$. 
Another possibility to get a Lipschitz continuous loss is to bound its derivative.
\begin{lemma} \label{Lemma_Lipschitz_beschrAbl}
	Let $L$ be a differentiable (ratio-based) loss function and $M \in (0, \infty)$. Suppose,
	\begin{equation*}
		\sup_{(x,y,t) \in \XY \times \R} |L'(x,y,t)| \leq M < \infty.
	\end{equation*}
	Then, $L$ is a Lipschitz continuous loss function.
\end{lemma}
\begin{proof}
	Any differentiable function with bounded first derivative is Lipschitz continuous (cf. e.g. \cite[p. 147]{Königsberger2004}).
\end{proof}
Indeed, there is a equivalence (\cite[Chapter 1, Thm. 7.3]{ClarkeEtAl1998}).
Note that we do not need a \emph{ratio-based} loss function here. 
\begin{lemma} \label{Lemma_Lipschitz_beschrAbl2}
	Let $L$ be a differentiable (ratio-based) loss function such that $|L'|$ is not bounded. Then, $L$ is not a Lipschitz continuous loss function.
\end{lemma}
\begin{example}[Continuation of Example \ref{Bsp_log(1+r)}]
	One can show that
	\begin{equation*}
		L(x,y,t) = \log\Bigl(1+\frac{\e^t}{y}\Bigr) + \log\Bigl(1+\frac{y}{\e^t}\Bigr) - 2\log(2)
	\end{equation*}
	is a Lipschitz continuous loss function. 
	For this, we apply Lemma \ref{Lemma_Lipschitz_beschrAbl} since
	\begin{equation*}
		|L'(x,y,t)| = \Bigl|\frac{1}{1+\frac{\e^t}{y}} \frac{\e^t}{y} - \frac{1}{1+\frac{y}{\e^t}}\frac{y}{\e^t}\Bigr| = \Bigl|\frac{\e^t}{y+\e^t} - \frac{y}{\e^t+y}\Bigr|\leq 2.
	\end{equation*}
\end{example}
This result can be summarized in the following proposition.
\begin{proposition}[Continuation of Proposition \ref{Prop_BspKonvexeVerlustfkt_Integral}] \label{Prop_Lipschitz_Integral}
	Let the assumptions of Proposition \ref{Prop_BspKonvexeVerlustfkt_Integral} be satisfied, i.e. let $L$ be a ratio-based loss function with $Y = (0, \infty), u(t) =\e^t$, and $g \in C^1((0, \infty))$ being increasing. 
	Additionally, assume $-M \leq g(r) \leq M$ holds for some $M \in (0, \infty)$.
	Besides, define
	\begin{equation*}
		\widetilde{\ell}(r) := C + \int_{r_0}^r \frac{g(t)}{t} \; \mathrm{d}t, \qquad r \in (0, \infty)
	\end{equation*}
	for $r_0 \geq 0$ and $C \in \R$ such that $\widetilde{\ell}$ is well-defined, non-negative, and finite.
	Then,
	\begin{equation*}
		L(x,y,t) = \widetilde{\ell}\Bigl(\frac{\e^t}{y}\Bigr)+\widetilde{\ell}\Bigl(\frac{y}{\e^t}\Bigr)-2\widetilde{\ell}(1)
	\end{equation*}
	is a Lipschitz continuous ratio-based loss function with Lipschitz constant $|L|_1 \leq 2M$.
\end{proposition}
\begin{corollary}
	Let $Y = (0, \infty)$, $u = \exp$, and $c = 0$. Let $g \in C^1((0, \infty))$ be an increasing function fulfilling $-M \leq g(r) \leq M$ for some $M \in (0, \infty)$.
	Let 
	\begin{equation*}
		\ell(r) := C + \int_{r_0}^r \frac{g(t)}{t} \; \dd t, \qquad r \in (0, \infty)
	\end{equation*}
	with $r_0 \geq 0$, $C \in \R$ such that $\ell$ is well-defined, non-negative, finite, ratio-symmetric, and $\ell(1) = 0$.
	Then, 
	\begin{equation*}
		L(x,y,t) = \ell\biggl(\frac{\e^t}{y}\biggr) 
	\end{equation*}
	is a Lipschitz continuous rb loss function.
\end{corollary}
\begin{example}[Continuation of Example \ref{bsp_konvex_Lipschitz_3}]
	According to Proposition \ref{Prop_Lipschitz_Integral}, $L$ from Example \ref{bsp_konvex_Lipschitz_3} is a Lipschitz continuous loss function.
\end{example}
Most examples for $\ell$ from Chapter \ref{Kap_Bsp}, however, are not globally, but only locally Lipschitz continuous (cf. Table \ref{TablePropertiesl}).
Nevertheless, there are functions which are globally Lipschitz continuous, e.g. $\ell(r) = |r-1|$, for which we can apply Lemma \ref{LipschitzL}.
Such functions bound the penalization of underestimation.
Moreover, not all link functions $u$ are globally Lipschitz continuous.
For example, the exponential function does not satisfy this condition.
Anyway, the logistic distribution function $u$ has Lipschitz constant $|u|_1 = \frac{1}{4}$. 
\begin{lemma} \label{LipschitzY01}
	Let $Y = (0,1), u(t) = \frac{1}{1+\exp(-t)}$, and $c > 0$. Additionally, let $\ell$ be locally Lipschitz continuous.
	Then, $L$ is a (globally) Lipschitz continuous ratio-based loss function with Lipschitz constant $|L|_1 \leq \frac{|\ell|_{I,1}|u|_1}{c}$ and $I := \Bigl(\frac{c}{1+c}, \frac{1+c}{c}\Bigr)$, where $|\ell|_{I,1}$ denotes $\ell$'s Lipschitz constant on $I$.
\end{lemma}
Note that here the specific link function $u$ does not matter as long as $u$ is globally Lipschitz continuous. 
Every example of link function $u$ for $Y = (0,1)$ from Table \ref{table_u} is indeed Lipschitz continuous.
\begin{lemma} \label{LocalLipschitzL}
	Let $L$ be a ratio-based loss, such that either $a > 0$ or $c > 0$. Let $b < \infty$ and  let $\ell$ and $u$ be locally Lipschitz continuous. 
	Then, $L$ is a locally Lipschitz continuous rb loss function.
\end{lemma}
Note that $b < \infty$ plays a key role in this proof.
Again, $a +c >0$ is necessary here.
Otherwise we would require local Lipschitz continuity arbitrarily close to zero which cannot be assumed in general.
\begin{lemma} \label{LocalLipschitzL2}
	Consider $Y = (a, \infty)$, $a \geq 0$, and a ratio-based loss $L$, such that $a + c > 0$ holds. 
	Moreover, assume locally Lipschitz continuous functions $\ell: (0, \infty) \to [0, \infty)$ and $u: \R \to Y$. 
	Assume that $\ell$ can be continuously continued in 0. 
	Furthermore, let the continuation $\ell: [0, \infty) \to [0, \infty)$ be locally Lipschitz continuous.
	Under those circumstances, $L$ is a locally Lipschitz continuous ratio-based loss function.
\end{lemma}
\begin{corollary}
	Let $Y = (a,\infty), a \geq 0$, and $c \geq 0$ with $a+c>0$. 
	Let $\ell: (0, \infty) \to [0, \infty)$ and $u: \R \to Y$ be locally Lipschitz continuous.
	Furthermore, let $\ell \in C^1((0, \infty))$ and assume $\lim_{r \to 0, r > 0}\ell(r) < \infty$ as well as $\lim_{r \to 0, r > 0} \ell'(r) < \infty$. 
	Then, $L$ is a locally Lipschitz continuous ratio-based loss function.
\end{corollary}

\subsection{Finite Risk} \label{Kap_endlRisiko}
Let $P$ be a probability measure on $\XY$.
The $L$-risk of a measurable function $f: X \to \R$ is defined by
\begin{equation*}
	\RLP(f) := \int_{\XY} L(x,y,f(x)) \; \mathrm{d}P(x,y).
\end{equation*}
Note, that this integral is always well-defined, because $L$ is non-negative, but it can happen that $\RLP(f)=\infty$ for some $P$ and some $f$.\par 
As $P$ is unknown in machine learning, many ML algorithms approximate
the true risk $\RLP(f)$ by the empirical risk 
\begin{equation*}
	\RLD(f) = \frac{1}{n} \sum_{i=1}^n L(x_i,y_i,f(x_i)),
\end{equation*}
where $D:=D_n := \frac{1}{n} \sum_{i =1}^n \delta_{(x_i, y_i)}$ denotes the empirical measure for the data set $\bigl((x_1,y_1),\dots,(x_n,y_n)\bigr)$.\par 
To avoid the danger of overfitting, many ML algorithms minimize a regularized  risk, i.e. one minimizes
\begin{equation*}
	\RregLP(f):= \RLP(f) + \lambda R(f)
\end{equation*}
over a suitable subset $\mathcal{F}$ of (all) measurable functions $f: X \to \R$ with some regularization parameter $\lambda > 0$ and regularization term $R(f)$. 
One special case are support vector machines, see e.g. \cite{Vapnik1995, Vapnik1998, SC08} with a general loss function $L$, $\mathcal{F}$ chosen to be a reproducing kernel Hilbert space $H$, and $R(f)=\|f\|_H^2$. 
There are too many generalizations to mention them here in detail. 
E.g. $\mathcal{F}$ chosen as a reproducing kernel Banach space and other regularization functions like $\|f\|^p$ with $p\ge 1$ or $\alpha_1 \|f\|_{\mathcal{F}}^1 + \alpha_2\|f\|_{\mathcal{F}}^2$ have been proposed, see e.g. \cite{ZhangXuZhang2009,DeMolDeVitoRosasco2009,RosascoEtAl2013, CuckerZhou2007} and the references therein.
For mathematically even more advanced function spaces used in deep distributed convolutional neural networks, we refer to  
\cite{Zhou2018},\cite{LinWangWangZhou2022}, and \cite{ZhangShiZhou2024}.
\begin{remark}
	Let $L$ be a rb loss function.
	If $Y = (a,b), 0 \leq a < b < \infty$ (note that $b \neq \infty$), and $c \geq 0$, but $a+c>0$,
	\begin{equation*}
		\frac{u(t)+c}{y+c} \in \biggl(\frac{a+c}{b+c}, \frac{b+c}{a+c}\biggr) \subsetneq (0, \infty).
	\end{equation*}
	So even the closed interval is contained in the domain of $\ell$.
	Thus, if $\ell$ is continuous, $\ell\Bigl(\frac{u(f(x))+c}{y+c}\Bigr)$ is contained in $[0, M]$ for some $M > 0$ (cf. e.g. \cite[p. 90]{Königsberger2004}).
	Hence, 
	\begin{equation*}
		\mathcal{R}_{L, P}(f) \leq M < \infty
	\end{equation*}
	holds for any measurable function $f: X \to \R$.
\end{remark}
Considering $Y = (0, \infty)$, the upper fraction can become arbitrarily large and arbitrarily close to zero.
\begin{remark}
	Let $L$ be a rb loss function.
	Using a bounded function $\ell$ (e.g. robust loss function), a measurable function $f$ has a finite risk for all probability measures $P$, even if $Y = (0, \infty)$ and $c = 0$.
\end{remark}
As many machine learning algorithms minimize the risk (or the regularized risk) over an appropriate function space $\mathcal{F}$, we need to avoid the case that $\RLP(f)=\infty$ or  $\RLD(f)=\infty$ for \emph{all} $f\in \mathcal{F}$.
In this case, it is often sufficient to know that there exists at least one function $f\in\mathcal{F}$ whose (regularized) risk is finite such that the minimal (regularized) risk is finite. Often this is easily checked for the function $f \equiv 0$. 
Let $L$ be a rb loss function. 
Then, 
\begin{equation} \label{BruchAbschaetz}
	\frac{u(0)+c}{y+c} \in \biggl(0, \frac{u(0)+c}{a+c}\biggr) \subseteq (0, \infty),
\end{equation}
where $u(0) \in Y$ is a certain value, that gives an upper bound for the quotient.
This yields
\begin{equation*}
	\mathcal{R}_{L, P}(0) = \int_{X \times Y} \biggl|\ell\biggl(\frac{u(0)+c}{y+c}\biggr) - \ell(1)\biggr| \; \mathrm{d}P(x,y).
\end{equation*}
\begin{remark} \label{finiteRisk_locLipschitz}
	Let $Y = (0, \infty)$ and $c > 0$. 
	Let $\ell$ be locally Lipschitz continuous and continuously continuable in 0, such that the continuation $\ell: [0, \infty) \to [0, \infty)$ is locally Lipschitz continuous.
	Using \eqref{BruchAbschaetz} with $m := \frac{u(0)+c}{c}$, we obtain for 
	an rb loss $L$ that 
	\begin{align*}
		\mathcal{R}_{L, P}(0) &= \int_{X \times Y} \biggl|\ell\biggl(\frac{u(0)+c}{y+c}\biggr) - \ell(1)\biggr| \; \mathrm{d}P(x,y) \\
		&\leq \int_{X \times Y} |\ell|_{[0,m],1} \biggl(\frac{u(0)+c}{c} + 1\biggr) \; \mathrm{d}P(x,y) = |\ell|_{[0,m],1} \biggl(\frac{u(0)}{c} + 2\biggr) < \infty
	\end{align*} 
	for all probability measures $P$.
\end{remark}
\begin{remark}
	Under the assumptions from Remark \ref{finiteRisk_locLipschitz}, one can show, in the very same way, that the risk $\RLP(f)$ of a bounded measurable function $f: X \to \R$ is finite as well for all probability measures $P$ where we take the monotonicity of $u$ into account. 
\end{remark}
\subsection{Nemitski Loss Functions}
For distance-based loss functions and their risks, the Nemitski property plays an important role. The reason is that this property can help to achieve a finite $L$-risk, which is necessary for many ML methods.
Additionally, it helps to transfer continuity and differentiability from the loss function to $\RLP$.
The \textbf{Nemitski} property separates the influence of variables $(x, y)$ and $t$:
\begin{equation*}
	L(x, y,t) \leq \tilde{b}(x, y) + \tilde{h}(|t|),
\end{equation*}
with increasing $\tilde{h}: [0, \infty) \to [0, \infty)$ and measurable $\tilde{b}: \XY \to [0, \infty)$.
Furthermore, such loss functions are called $\boldsymbol{P}$\textbf{-integrable Nemitski} if and only if $\tilde{b} \in \Lcal_1(P)$ with $P$ being the probability measure on $X \times Y$; for details, we refer to \cite[Section 2.4]{SC08}. 
On the other hand, if
\begin{equation*}
	L(x,y,t) \leq \tilde{b}(x,y) + \lambda |t|^p
\end{equation*}
is valid for some $p \in (0, \infty)$ and $\lambda > 0$, $L$ is a \textbf{Nemitski loss of order} $\boldsymbol{p}$.
We apply these definitions also to ratio-based loss functions.
\begin{remark}
	A locally Lipschitz continuous loss function $L$ is a $P$-integrable Nemitski loss if and only if $\mathcal{R}_{L,P}(0) < \infty$, see e.g. \cite[p. 31]{SC08}.
\end{remark}
\begin{example}
	With a monotonically increasing function $u$ and $a + c \geq 0$, $\ell(r) = |r-1|$ defines a $P$-integrable Nemitski loss function for all probability measures $P$ on $\XY$.
	A rough calculation yields
	\begin{equation*}
		L(x,y,t) \leq \frac{u(|t|)+c}{a+c} + 1.
	\end{equation*}
\end{example}

\section{Risk} \label{Kap_Risiko}
For loss functions, it is common to consider the output space $Y$ to be a closed subset of $\R$.
This property ensures that one can write risks
\begin{equation*}
	\mathcal{R}_{L, P}(f) := \int_{X \times Y} L(x,y,f(x)) \; \mathrm{d}P(x,y)
\end{equation*}
of measurable functions $f: X \to \R$ as double integrals
\begin{equation*}
	\int_X \int_Y L(x,y,f(x)) \; \mathrm{d}P(y|x) \; \dd P_X(x).
\end{equation*}
Nevertheless, the same conclusion holds in our situation (recall $Y = (a,b), 0 \leq a < b \leq \infty$), as can be seen by using results e.g. from \cites[Chapter 26]{Bauer2001}[Chapter 8]{Cohn2013}[Chapter 10]{Dudley2002}.
This requires a measurable Polish space $(X, \mathcal{A})$.
Since $\R$ (endowed with the standard topology) is a Polish space (\cites[Chapter 26, p. 157]{Bauer2001}[Ex. 8.1.1 (a)]{Cohn2013}), its open subset $Y$ (endowed with its subspace topology) is Polish as well (\cites[Chapter 26, p. 157]{Bauer2001}[Prop. 8.1.2]{Cohn2013}).
In particular, $\XY$ is Polish (\cites[Chapter 26, p. 157]{Bauer2001}[Prop. 8.1.4]{Cohn2013}).
Endowing $Y$ with its Borel $\sigma$-algebra $\B(Y)$ induced by the subspace topology, $\XY$ with its product $\sigma$-algebra $\AB(Y)$ is a measurable space.
Letting $P$ be a probability measure on $\XY$, $(X \times Y, \AB(Y), P)$ is probability space.
Define the projections
$\pi_X: X \times Y \to X, 	\pi_Y: X \times Y \to Y$, which are measurable functions.
Let $\mathcal{C} := \pi_X^{-1}(\mathcal{A})$ denote the smallest sub-$\sigma$-algebra of $\AB(Y)$ such that $\pi_X$ is measurable.
At last, we denote the marginal distribution of $P$ on $X$ by $P_X := P \circ \pi_X^{-1}$.
Hence, the conditional probability distribution $P_{\pi_{Y} | \mathcal{C}}$ on $\mathcal{B}(Y) \times (X \times Y)$ exists (\cite[Thm. 10.2.2]{Dudley2002}).
Therefore, there also exist conditional distributions $P(\cdot | x)$ for $P$ and $x \in X$ (\cite[Thm. 10.2.1 (I)]{Dudley2002}),
which imply
\begin{equation*}
	\int_{X \times Y} g(x,y) \; \mathrm{d}P(x,y) = \int_X \int_Y g(x,y) \; \mathrm{d}P(y|x) \; dP_X(x)
\end{equation*}
for all $P$-integrable functions $g$ (\cite[Thm. 10.2.1 (II)]{Dudley2002}). 
From this, we see that the upper equality holds not only for integrable functions $g$, but for all functions whose integral $\int g \; \mathrm{d}P$ is defined.
That is whenever $\int g_+ \; \mathrm{d}P < \infty$ or $\int g_- \; \mathrm{d}P < \infty$ holds, where $g_+ := \max\{g, 0\}$ and $g_- := \max\{-g, 0\}$. \par 
Now, one can discuss how properties of $\ell$ and $u$ apply to $\RLP$.
Most properties of loss functions apply to the risk as well. 
Some, however, need additional assumptions.
The following results refer the loss functions' properties back to those of its components $\ell$ and $u$ and use \cite[Chapter 2.2]{SC08}.
\begin{lemma}
	Let $L$ be a ratio-based loss function and let $\ell$ and $u$ be continuous. 
	Let $(f_n)_{n \in \N}$ be a sequence in $\Lcal_0(P_X)$ which converges to $f \in \Lcal_0(P_X)$ in probability w.r.t. the marginal distribution $P_X$. 
	Then, the risk is lower semi-continuous, i.e. 
	\begin{equation*}
		\RLP(f) \leq \liminf_{n \to \infty} \RLP(f_n).
	\end{equation*}
\end{lemma}
\begin{proof}
	This follows directly from Lemma \ref{LemmaLstetig} and \cite[Lemma 2.15]{SC08}.
\end{proof}
To prove the continuity of the risk, we will assume the Nemitski property.
\begin{lemma}
	Let $L$ be a ratio-based loss function with continuous functions $\ell$ and $u$ and let $P$ be a distribution on $\XY$, such that $L$ is a $P$-integrable Nemitski loss function.
	\begin{enumerate}
		\item Then, $\RLP: L_\infty(P_X) \to [0, \infty)$ is well-defined and continuous.
		\item If $L$ is a Nemitski loss of order $p \in [1, \infty)$, $\RLP: L_p(P_X) \to [0, \infty)$ is well-defined and continuous.
	\end{enumerate}
\end{lemma}
\begin{proof}
	Combining Lemma \ref{LemmaLstetig} and \cite[Lemma 2.17]{SC08} gives the assertion.
\end{proof}
\begin{lemma}
	Let $L$ be a ratio-based loss function, let $\ell$ and $u$ be differentiable functions and let $P$ be a probability measure on $\XY$. 
	Let $L$ and $|L'|$ be $P$-integrable Nemitski loss functions.
	Then, $\RLP: L_\infty(P) \to [0, \infty)$ is Fr\'{e}chet differentiable and its derivative at $f \in L_\infty(P)$ is a bounded linear operator
	\begin{align*}
		\mathcal{R}'_{L,P}(f): L_\infty(P_X) &\to \R \\
		g &\mapsto \int_{\XY} g(x)L'(x,y,f(x)) \; \dd P(x,y).
	\end{align*}
\end{lemma}
\begin{proof}
	This is a direct consequence of Lemma \ref{LemmaLdiff} and \cite[Lemma 2.21]{SC08}.
\end{proof}
\begin{lemma}
	For a rb loss function $L$, let $Y = (0, \infty), u = \exp, c = 0$, and $\ell(r) = \widetilde{\ell}(r) + \widetilde{\ell}(r^{-1})-2\widetilde{\ell}(1)$ with $\widetilde{\ell}: (0, \infty) \to [0, \infty)$ being twice differentiable and fulfilling \eqref{Bedingung_Konvexität_tildeell} for all $r \in (0, \infty)$.
	Let $P$ be a probability measure on $\XY$.
	Then, $\RLP: \Lcal_0(P_X) \to [0, \infty]$ is convex.
\end{lemma}
\begin{proof}
	Combine Proposition \ref{Prop_konvexeVerlustfkt_tildeEll} with \cite[Lemma 2.13]{SC08}.
\end{proof}
Note that this result can be improved to strict convexity by claiming strict inequality in \eqref{Bedingung_Konvexität_tildeell}. 
\begin{lemma}
	Let $L$ be a ratio-based loss function.
	\begin{enumerate}
		\item Let $\ell$ and $u$ be Lipschitz continuous, $a+c >0$. $P$ denotes a distribution on $\XY$.
		Then, $\RLP: L_\infty(P_X) \to [0, \infty)$ is Lipschitz continuous.
		\item Let $Y = (0,1), u(t) = \frac{1}{1+\exp(-t)}, c > 0$, and let $\ell$ be locally Lipschitz continuous. 
		Let $P$ be a probability measure on $\XY$.
		Then, $\RLP: L_\infty(P_X) \to [0, \infty)$ is Lipschitz continuous.
		\item Let $a+c > 0$, $b < \infty$. Furthermore, let $\ell$ and $u$ be locally Lipschitz continuous and let $P$ be a distribution on $\XY$.
		Then, $\RLP: L_\infty(P_X) \to [0, \infty)$ is locally Lipschitz continuous.
	\end{enumerate}
\end{lemma}
\begin{proof}
	This follows directly from the combination of Lemmata \ref{LipschitzL}, \ref{LipschitzY01}, and \ref{LocalLipschitzL} with \cite[Lemma 2.19]{SC08}.
\end{proof}

\section{Examples} \label{Kap_Bsp}
In this chapter, we will give special representing functions $\ell: (0, \infty) \to [0, \infty)$. 
The argument of $\ell$, which we will usually denote by $r$ in the remainder of the paper, is positive.
We give definitions as well as plots in the following subsections resulting in an overview at the end.
Some of the following examples can already be found in the literature. 
However, they are put in our more general theoretical framework here.
We also propose some new ratio-based loss functions and briefly discuss why these loss functions may be useful. 
\subsection{Ratio-Based Loss Functions Using the Logarithm}
Since $\ell$ has by definition a positive argument, one can use the logarithm.
Additionally to the loss functions defined via \eqref{log(1+r)Loss} in Example \ref{Bsp_log(1+r)} and via \eqref{bsp_3_Loss} in Example \ref{bsp_konvex_Lipschitz_3}, we give the following representation functions.
In this context, some work is already done.
Many loss functions use the logarithm in a distance-based representation, especially with $c = 1$ (\cites[Chapter 4.2.9]{CiampiconiEtAl2024}[Chapters III.H, III.I]{JadonEtAl2022}[(9)]{LiEtAl2025}[Chapters 8.1.2, 8.2.3, 8.2.5, 8.2.8]{TervenEtAl2025}).
Because of the logarithm's rules, this can be expressed through a ratio-based approach, too.
\subsubsection{Squared Logarithmic Relative Loss}
\begin{equation} \label{logquadrFehler}
	\ell(r) := (\log(r))^2
\end{equation}
(cf. \cites[Chapter III.H]{JadonEtAl2022}[Chapter 8.1.2]{TervenEtAl2025}[Chapter 3.2]{Eigen2014})
\subsubsection{Absolute Logarithmic Relative Loss}
\begin{equation} \label{logabsFehler}
	\ell(r) = |\log(r)|
\end{equation}
(cf. \cites[Chapter 4.2.9]{CiampiconiEtAl2024}[Chapter III.I]{JadonEtAl2022}[(9)]{LiEtAl2025}[Chapter 8.2.3]{TervenEtAl2025})
\subsubsection{Huber-type Logarithmic Relative Loss}
To combine the advantages of the previous functions, i.e. differentiability of the squared logarithmic loss and hopefully better robustness properties of the absolute logarithmic loss, we introduce a Huber-type version which is continuous and continuously differentiable using a parameter $\alpha > 1$ by
\begin{equation} \label{mixedlogFehler}
	\ell(r) = 
	{\renewcommand{\arraystretch}{1.5}%
		\begin{cases}
			\log(\frac{1}{\alpha}) \log(r^2\alpha), & r \leq \alpha^{-1}, \\
			\log(r)^2, &\alpha^{-1} < r < \alpha, \\
			\log(\alpha)\log\Bigl(\frac{r^2}{\alpha}\Bigr), & \alpha \leq r.
	\end{cases}}
\end{equation} 
The corresponding rb loss function $L$ is obviously inspired by Huber's loss function (\cites[p. 71: Ex. 5.4]{Huber1981}[Chapter III.L]{JadonEtAl2022}[p. 2532]{KimScott2012}[pp. 104-105: Ex. 2, Fig. 1]{HampelEtAl1986}[p. 26]{MaronnaMartinYohai2006}).
To our knowledge, this loss function has not been proposed earlier. 
\begin{figure}[ht]
	\centering
	\begin{subfigure}{0.31\textwidth}
		\centering
		\includegraphics[width=\textwidth]{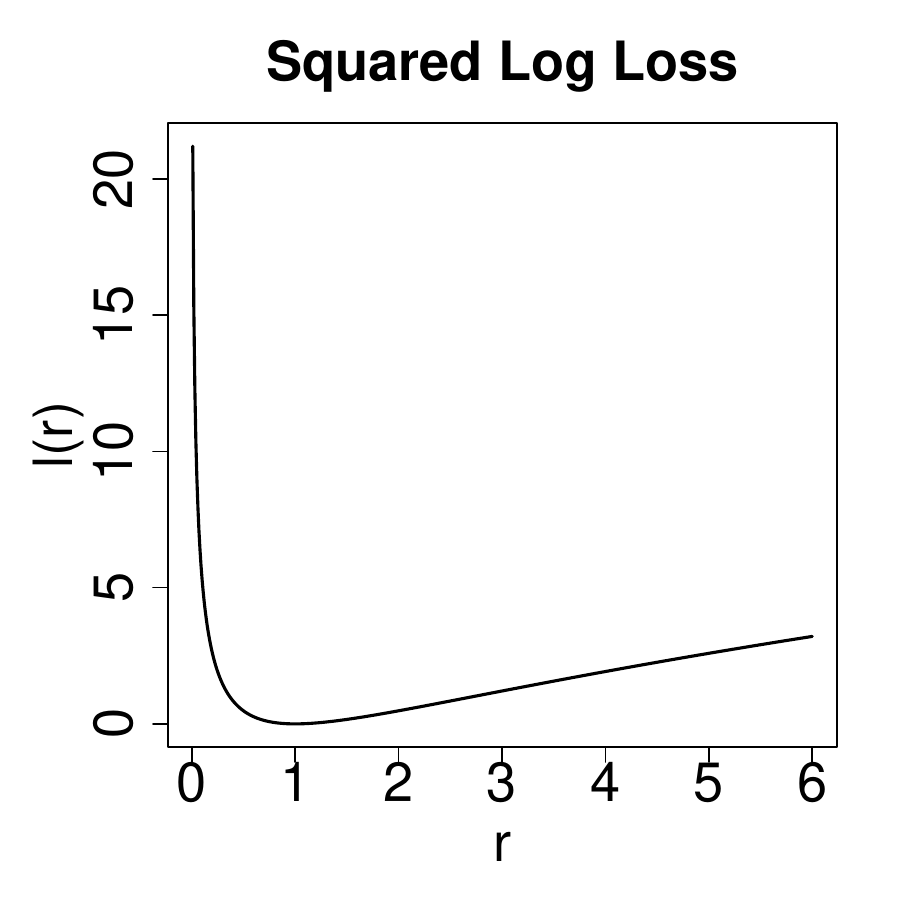}
	\end{subfigure}
	\hfill
	\begin{subfigure}{0.31\textwidth}
		\centering
		\includegraphics[width=\textwidth]{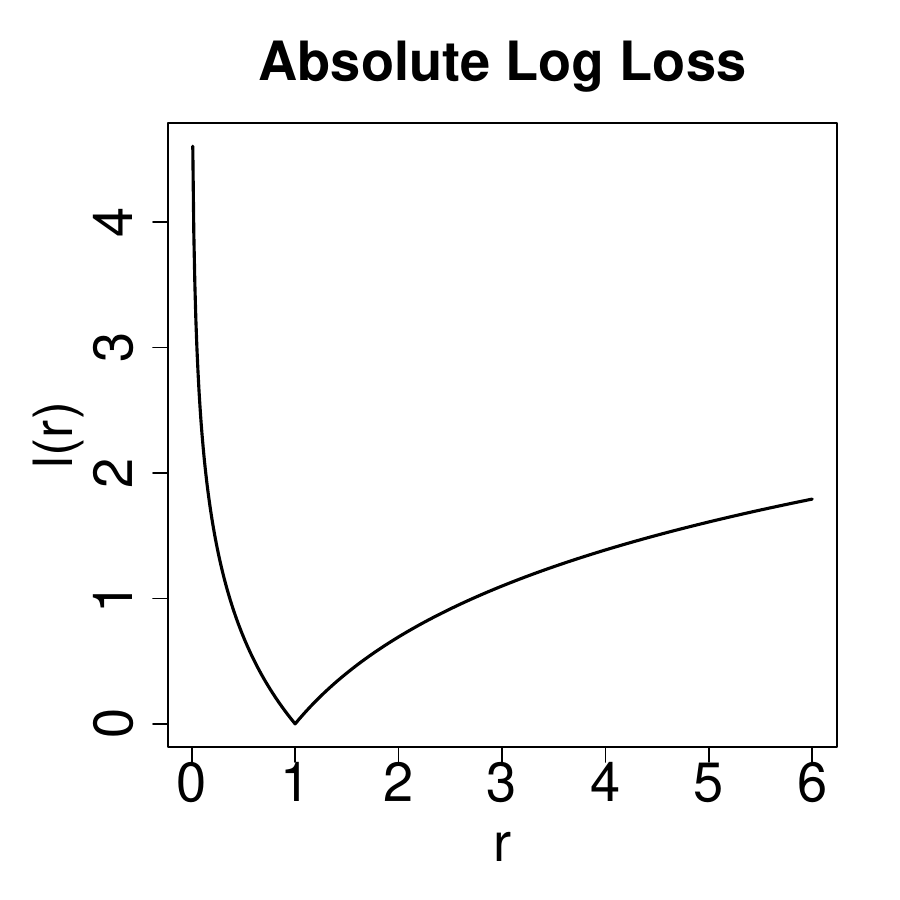}
	\end{subfigure}
	\hfill
	\begin{subfigure}{0.31\textwidth}
		\centering
		\includegraphics[width=\textwidth]{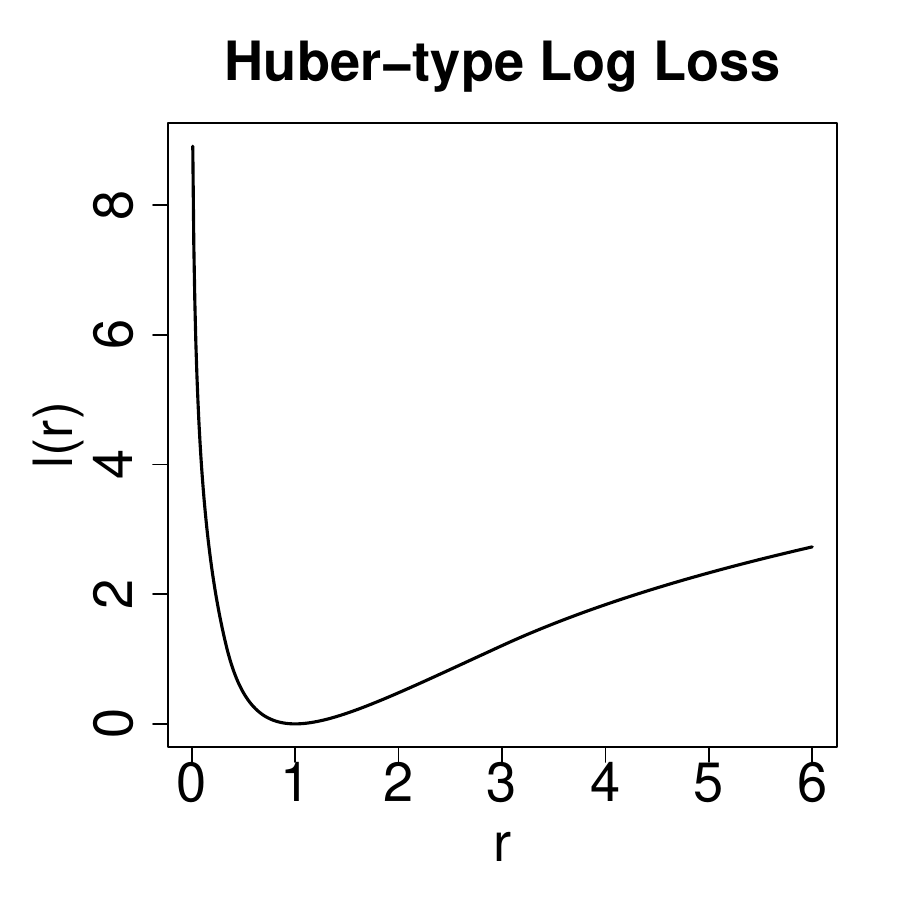}
	\end{subfigure}
	\caption{Plots of the representing functions $\ell$ using the logarithm; Huber-type logarithmic relative loss with parameter $\alpha = 3$}
\end{figure}

\subsection{Ratio-Based Loss Functions Using Logarithm and Hyperbolic Cosine}
A common distance-based loss function is log-cosh loss $L(x,y,t) = \psi(y-t) = \log(\cosh(y-t))$ (\cites[Chapter III.M]{JadonEtAl2022}[Chapter 3.1.9]{LiEtAl2025}[Chapter 4.2.8]{CiampiconiEtAl2024}).
We modify $\psi$ in various ways such that the minimum is attained in 1 to obtain a rb loss function. 

\subsubsection{Log-cosh Relative Loss}
\begin{equation} \label{logcoshFehler}
	\ell(r) = \log(\cosh(r-1))
\end{equation}
\subsubsection{Cosh-log Relative Loss}
\begin{equation} \label{coshlogFehler}
	\ell(r) = \cosh(\log(r)) - 1
\end{equation}
\subsubsection{Log-cosh-log Relative Loss}
\begin{equation} \label{logcoshlogFehler}
	\ell(r) = \log(\cosh(\log(r)))
\end{equation}
\begin{figure}[ht]
	\centering
	\begin{subfigure}{0.31\textwidth}
		\centering
		\includegraphics[width=\textwidth]{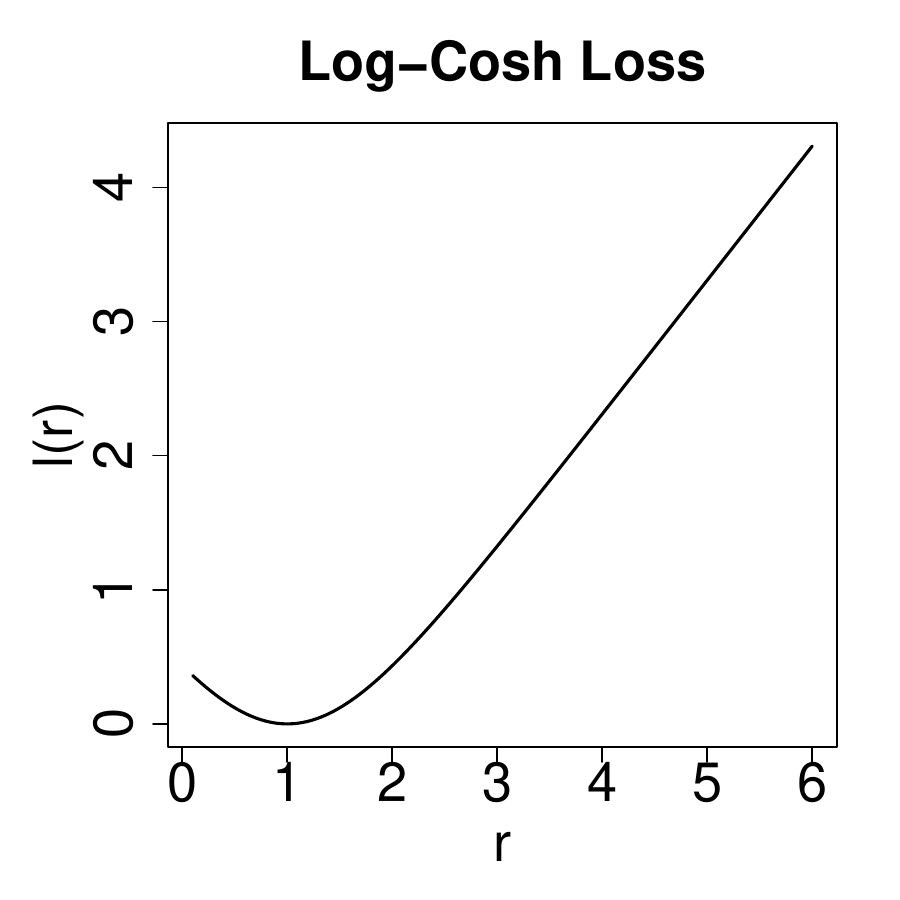}
	\end{subfigure}
	\hfill
	\begin{subfigure}{0.31\textwidth}
		\centering
		\includegraphics[width=\textwidth]{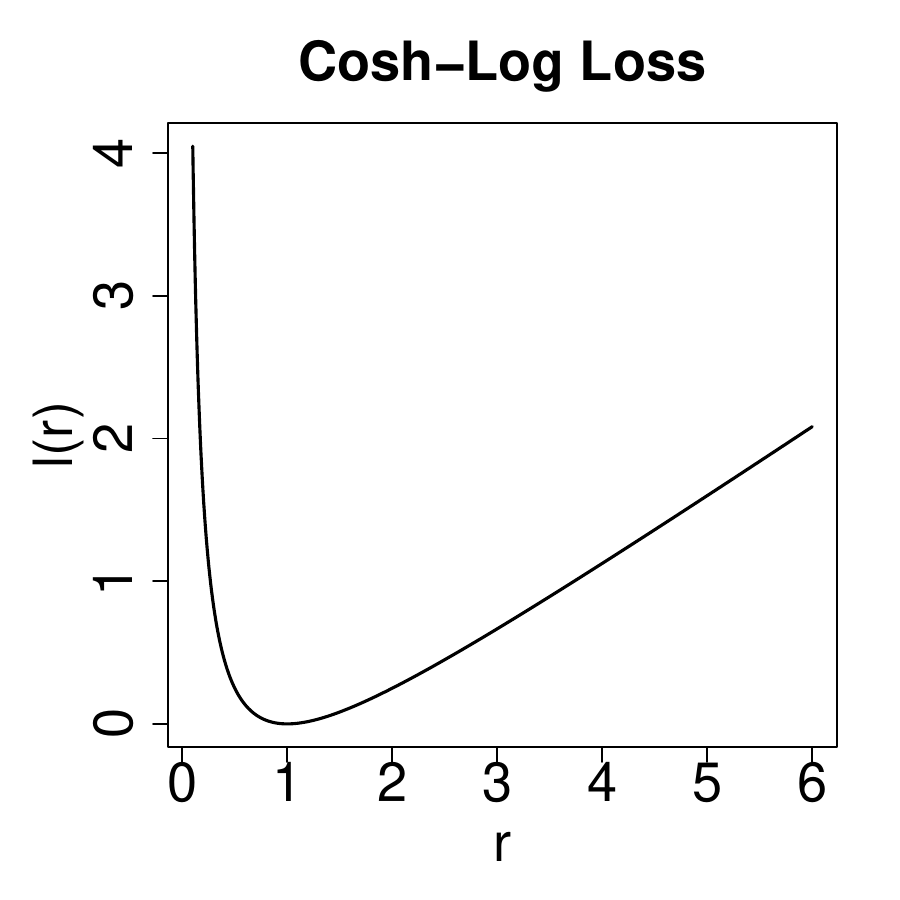}
	\end{subfigure}
	\hfill
	\begin{subfigure}{0.31\textwidth}
		\centering 
		\includegraphics[width=\textwidth]{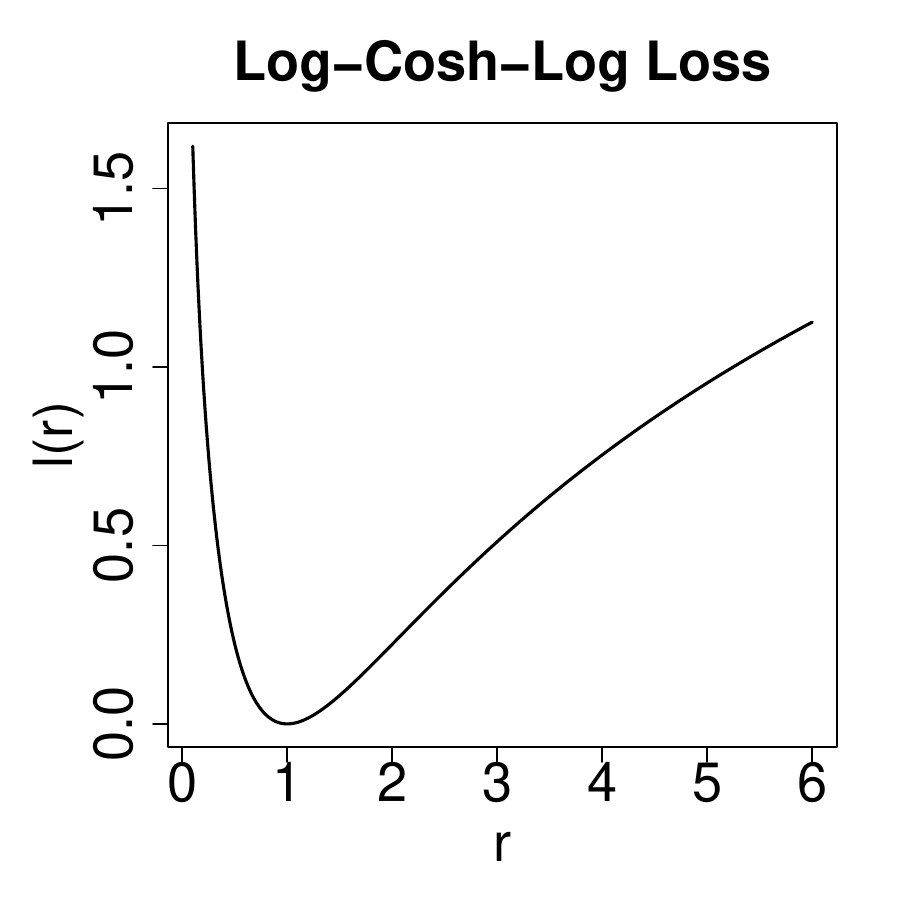}
	\end{subfigure}
	\caption{Plots of the representation functions $\ell$ using logarithm and hyperbolic cosine}
\end{figure}
\subsection{Maximum Loss}
Next, we define maximum loss (cf. \cites[Chapter 8.2.4]{TervenEtAl2025}[Chapter 4.3]{Eigen2014}[p. 15]{Ye2007})
\begin{equation} \label{maxFehler}
	\ell(r) = \max\{r, r^{-1}\} - 1 = \max\{|1-r|, |1-r^{-1}|\}.
\end{equation}
The second expression shows its connection to the later defined \textit{general relative error} (\cite[(8)]{ChenEtAl2016LPRE}).
\subsection{Logarithmic Pinball Loss} 
The following rb loss function $L$ using $\ell$ is based on the idea of the distance-based $\tau$-pinball loss function (see e.g. \cite[Ex. 2.43]{SC08}) and is convex, Lipschitz continuous, but not differentiable. 
For $\tau\in(0,1)$ we define the logarithmic pinball loss function via 
\begin{equation} \label{maxlogFehler}
	\ell(r) = \max\{\tau \log(r), (1-\tau) \log(r^{-1})\}.
\end{equation}
\begin{figure}[ht]
	\centering
	\begin{subfigure}{0.31\textwidth}
		\includegraphics[width=\textwidth]{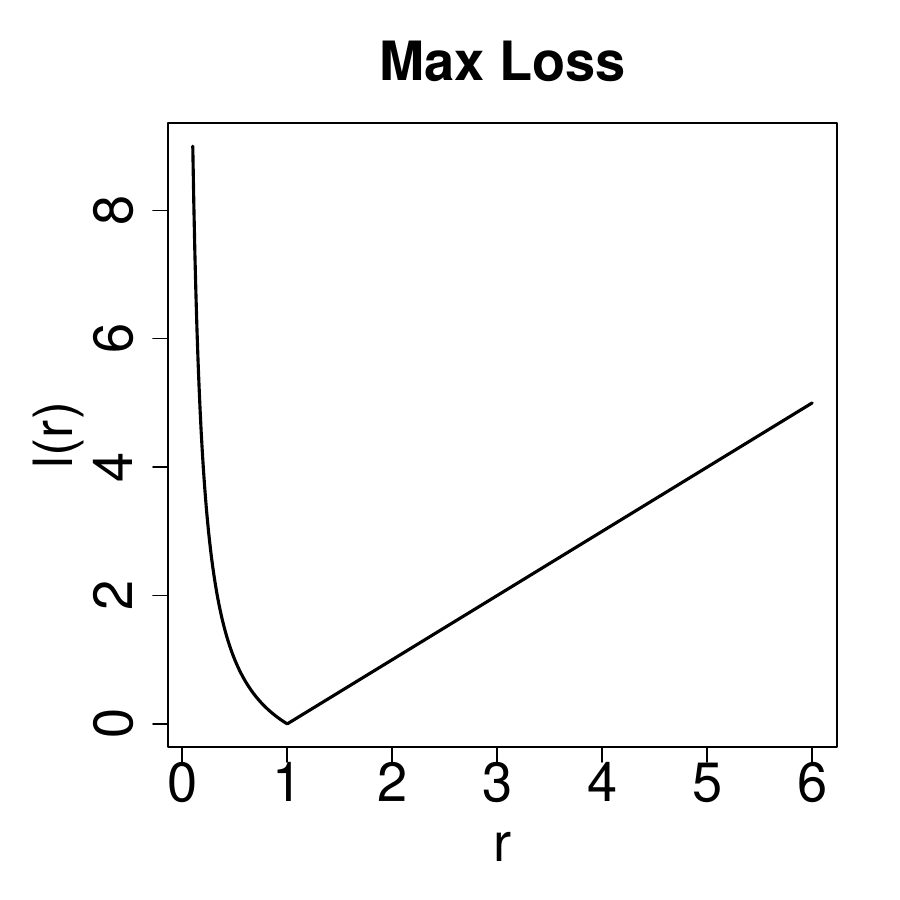}
		\caption{Plot of the representing function $\ell$ of maximum loss function from \eqref{maxFehler}}
	\end{subfigure}
	\hfill
	\begin{subfigure}{0.31\textwidth}
		\includegraphics[width=\textwidth]{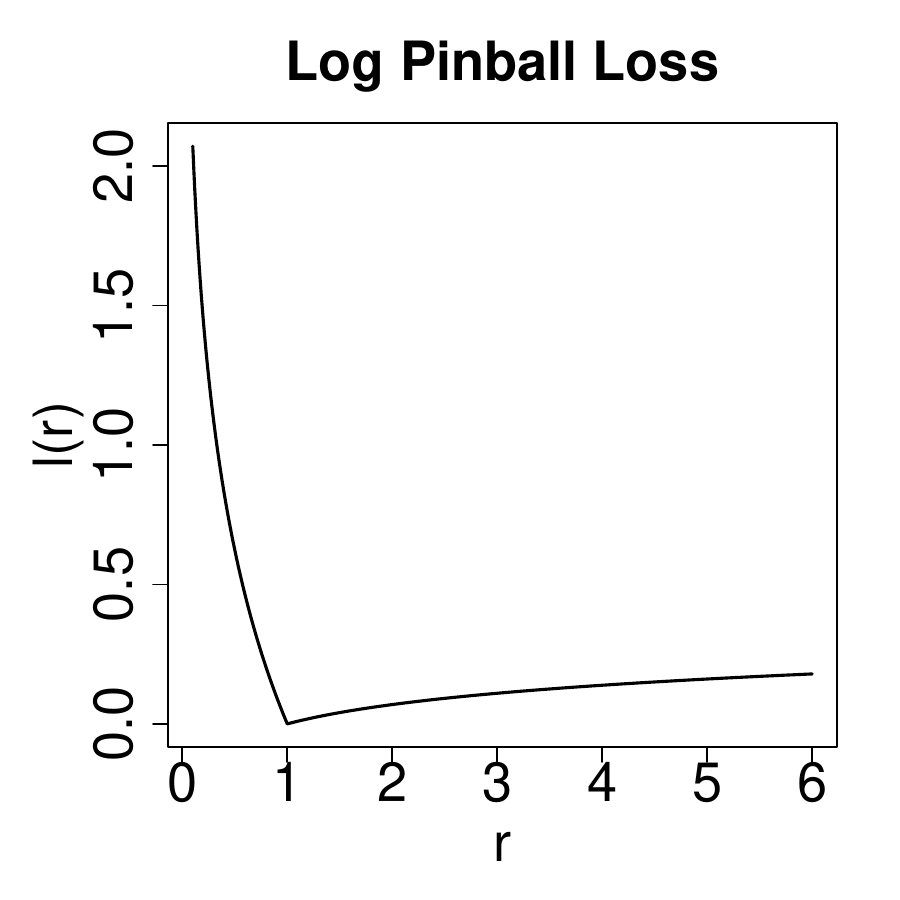}
		\caption{Plot of the representing function $\ell$ of logarithmic pinball loss with parameter $\tau = 0.1$}
	\end{subfigure}
	\hfill
	\begin{subfigure}{0.31\textwidth}
		\includegraphics[width=\textwidth]{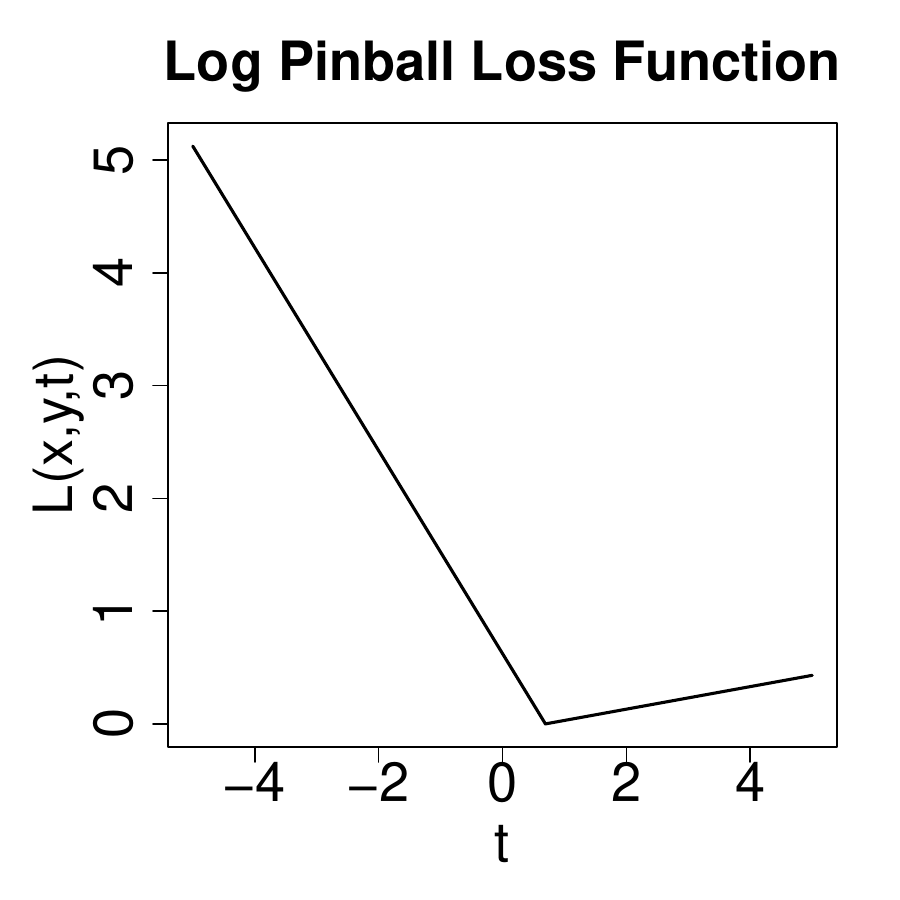}
		\caption{Plot of the logarithmic pinball loss function $L$ in variable $t$ for $y = 3$, $\tau = 0.1$, $c = 0$ and $u  =\exp$}
	\end{subfigure}
\end{figure}

\subsection{Absolute Relative Loss}
Inspired by \cites[Chapter 8.2.1]{TervenEtAl2025}[Chapter III.F]{JadonEtAl2022}[Chapter 4.3]{Eigen2014}[p. 13]{Ye2007}, an intuitive and easy way for calculations is the total percentage deviation 
\begin{equation} \label{absFehler}
	\ell(r) = |r-1|.
\end{equation}
\subsubsection{Squared Relative Loss}
Smoothing the loss function is done by taking the square (cf. \cite[Chapter 4.3]{Eigen2014}). 
This yields differentiability of $\ell$.
Still, we can derive total percentage deviation.
\begin{equation} \label{quadrFehler}
	\ell(r) = (r-1)^2
\end{equation}
\subsubsection{Huber-type Relative Loss}
Combining both absolute relative loss functions discussed above in the sense of Huber (\cite[Chapter 3.5: Ex. 5.4]{Huber1981}), we get 
\begin{equation} \label{mixedrelFehler}
	\ell(r) = 
	\begin{cases}
		2(\alpha^{-1}-1)(r-1)- (1-\alpha^{-1})^2, & r \leq \alpha^{-1}, \\
		(r-1)^2, & \alpha^{-1} < r < \alpha, \\
		2(\alpha -1)(r-1)- (\alpha-1)^2, & \alpha \leq r.
	\end{cases}
\end{equation}
As far as we know, this loss has not been studied in the literature. 
Obviously, this representing function $\ell$ is continuous and even has a continuous derivative.
\begin{figure}[ht]
	\centering
	\begin{subfigure}{0.31\textwidth}
		\centering
		\includegraphics[width=\textwidth]{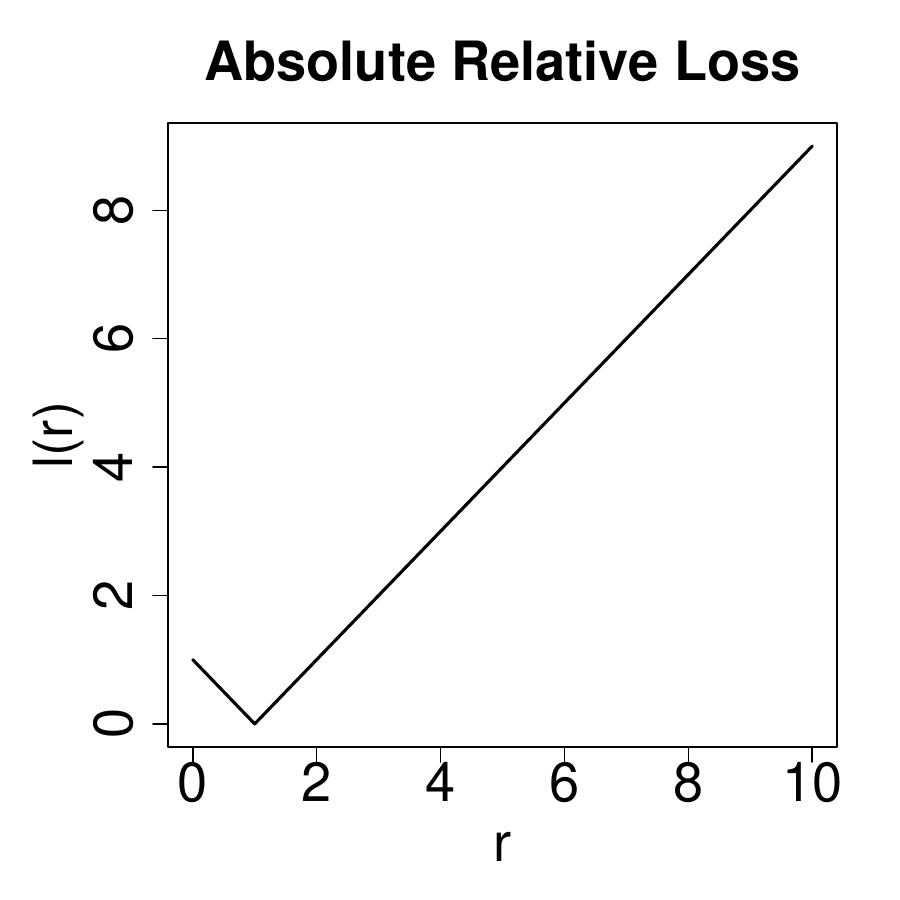}
	\end{subfigure}
	\hfill
	\begin{subfigure}{0.31\textwidth}
		\centering
		\includegraphics[width=\textwidth]{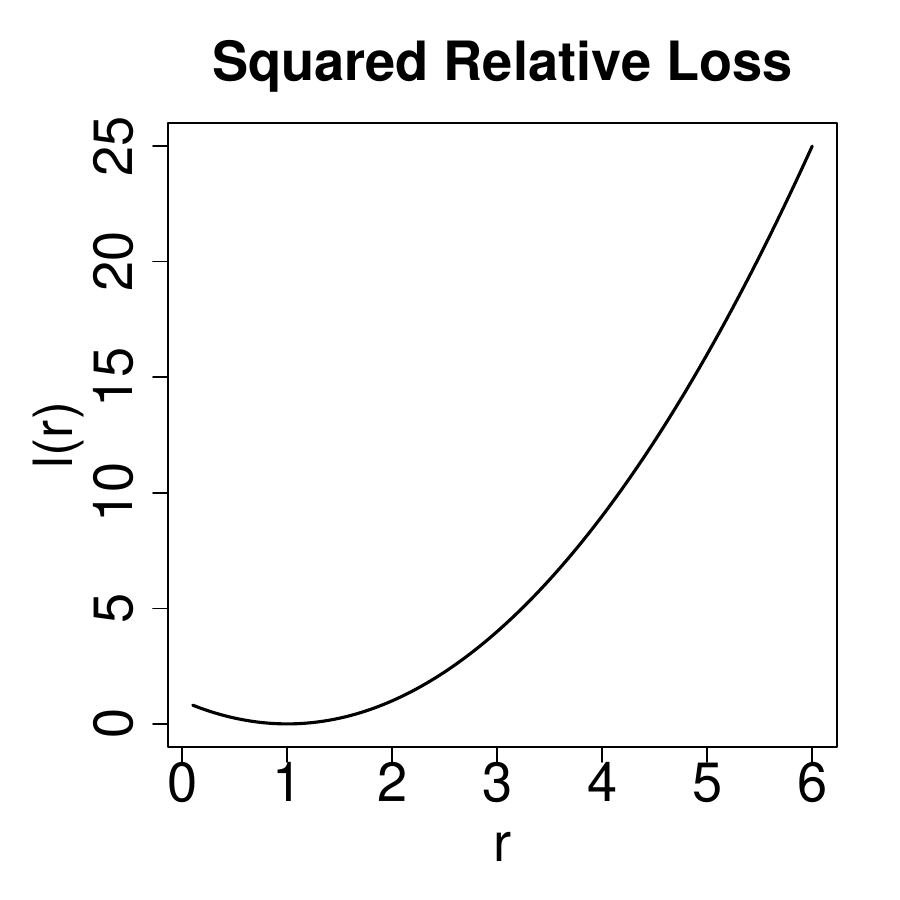}
	\end{subfigure}
	\hfill
	\begin{subfigure}{0.31\textwidth}
		\centering
		\includegraphics[width=\textwidth]{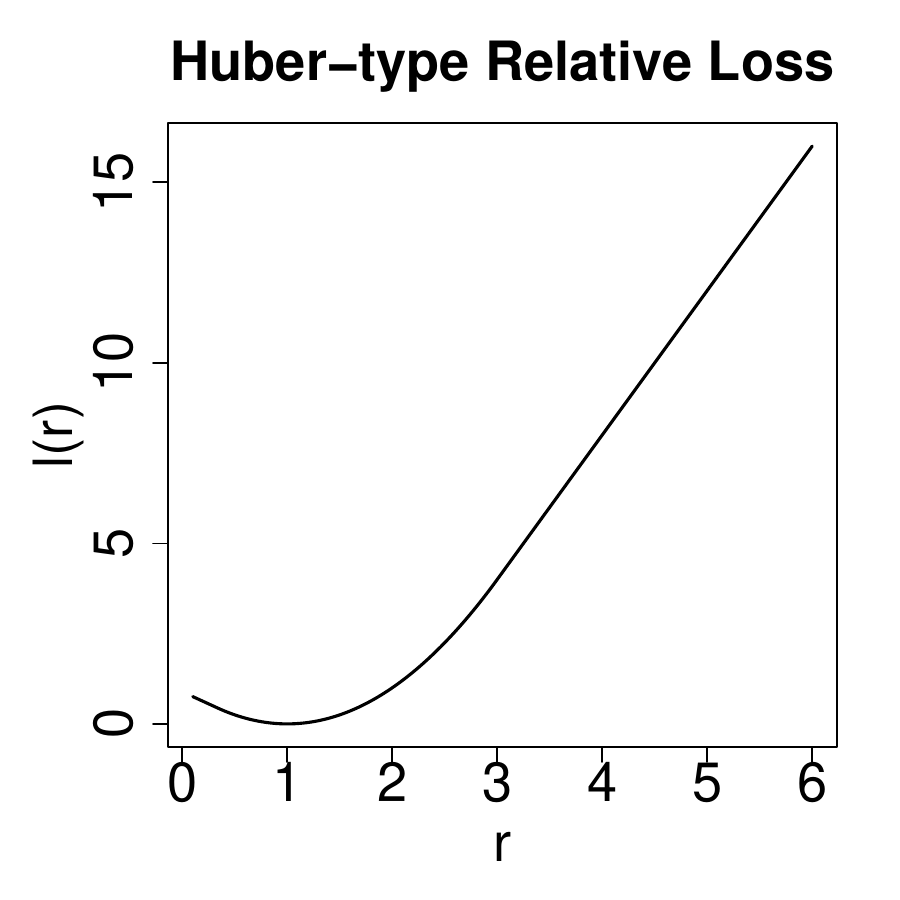}
	\end{subfigure}
	\caption{Plots of the representing functions $\ell$ of absolute relative, squared relative and Huber-type relative loss functions; Huber-type relative loss for parameter $\alpha = 3$}
\end{figure}  

\subsection{Inverse Absolute Relative Loss}
\begin{equation} \label{invabsFehler}
	\ell(r) = |r^{-1}-1|
\end{equation}
In contrast to the absolute loss, which penalizes overestimation hard, but neglected underestimation, this representation function $\ell$ behaves conversely.
Here, we focus on penalizing underestimation.
\subsubsection{Squared Inverse Relative Loss}
Again, differentiability can be achieved through taking the square:
\begin{equation} \label{invsqFehler}
	\ell(r) = (r^{-1}-1)^2.
\end{equation}
\subsubsection{Huber-type Inverse Relative Loss}
Advantages of both inverse loss functions given above are united in their Huber-type version with parameter $\alpha > 1$.
This is again inspired by Huber's loss function (\cite[Chapter 3.5: Ex. 5.4]{Huber1981}).
We requested continuity and continuous differentiability.
To our knowledge, this loss function has not been proposed in the literature.
\begin{equation} \label{mixedinvFehler}
	\ell(r) = 
	\begin{cases}
		2(\alpha-1)(r^{-1}-1)-(\alpha-1)^2, & r \leq \alpha^{-1}, \\
		(r^{-1}-1)^2, & \alpha^{-1} < r < \alpha, \\
		2(\alpha^{-1}-1)(r^{-1}-1)- (1-\alpha^{-1})^2, & \alpha \leq r.
	\end{cases}
\end{equation}
\begin{figure}[ht]
	\centering
	\begin{subfigure}{0.31\textwidth}
		\centering 
		\includegraphics[width=\textwidth]{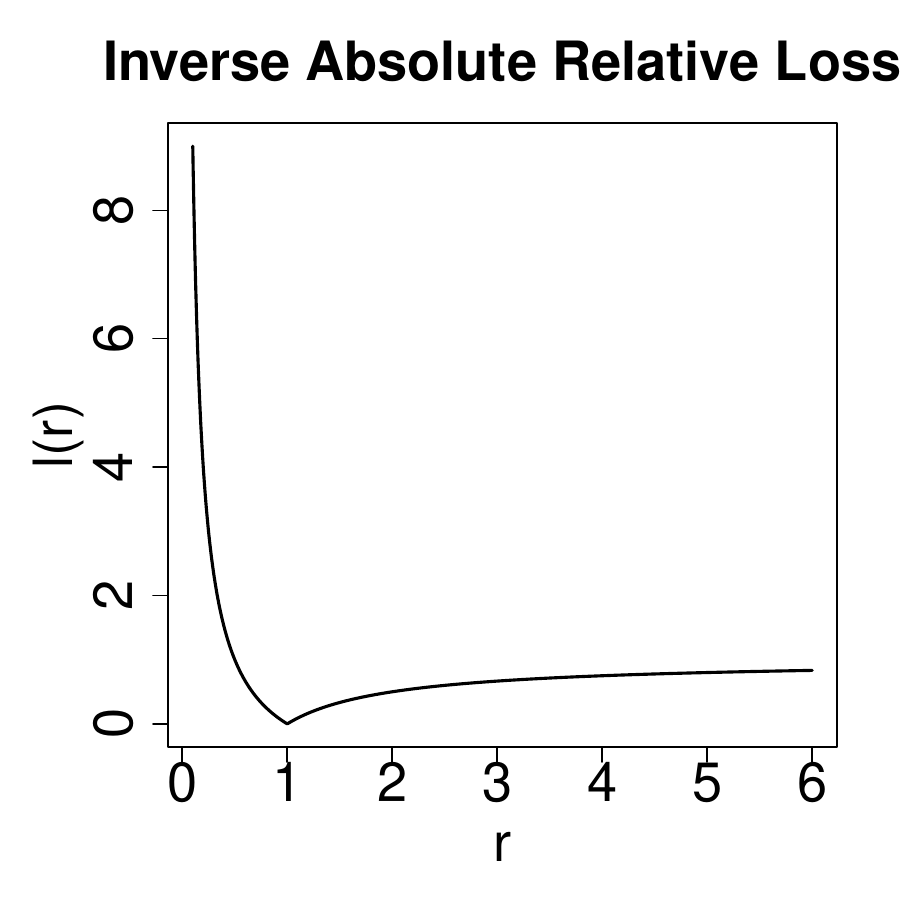}
	\end{subfigure}
	\hfill
	\begin{subfigure}{0.31\textwidth}
		\centering
		\includegraphics[width=\textwidth]{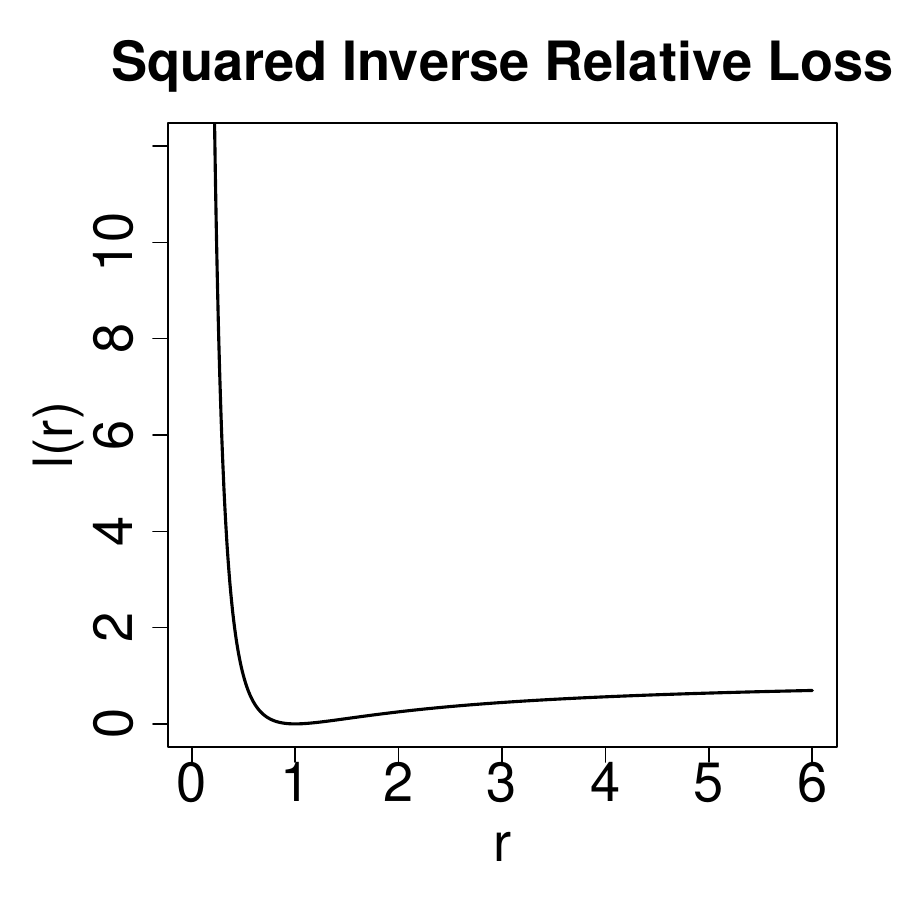}
	\end{subfigure}
	\hfill
	\begin{subfigure}{0.31\textwidth}
		\centering
		\includegraphics[width=\textwidth]{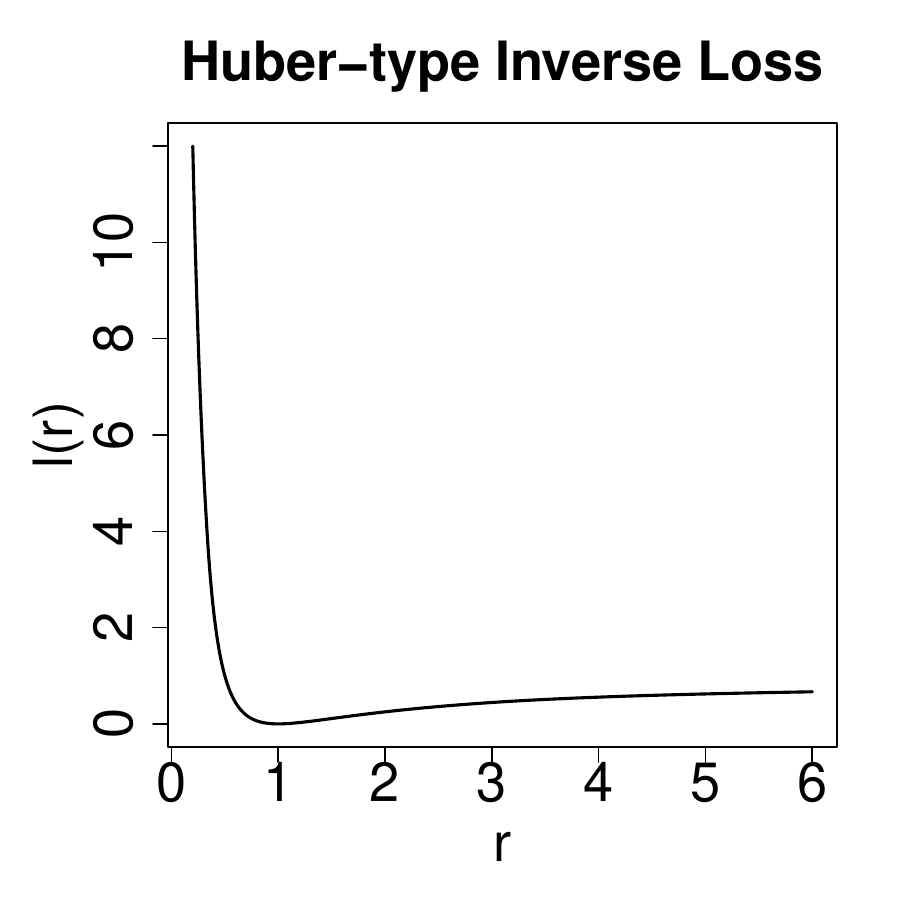}
	\end{subfigure}
	\caption{Plots of the representing functions $\ell$ of inverse relative loss functions; Huber-type inverse relative loss function for parameter $\alpha = 3$}
\end{figure}

\subsection{Least Absolute Relative Loss}
\begin{equation} \label{leastabsFehler}
	\ell(r) = |r - r^{-1}| = |1-r| + |1-r^{-1}|
\end{equation}
This loss function is following the example of \emph{least absolute relative error} (LARE, see \cite[(3)]{ChenEtAl2010LARE}).
The first representation allows easy calculations, the second shows its connection to \textit{general relative errors} \eqref{generalFehler} (\cite[(8)]{ChenEtAl2016LPRE}).
\subsubsection{Smooth Least Absolute Relative Loss}
Taking the square yields a differentiable function $\ell$:
\begin{equation} \label{smoothleastabsFehler}
	\ell(r) = (r-r^{-1})^2.
\end{equation}
\subsubsection{Huber-type Least Absolute Relative Loss}
\begin{equation} \label{mixedleastrelFehler}
	\ell(r) = 
	\begin{cases}
		\frac{2(\alpha^2-1)}{\alpha}(r^{-1}-r) - \frac{(1-\alpha^2)^2}{\alpha^2} & r \leq \alpha^{-1}, \\
		(r-r^{-1})^2, & \alpha^{-1} < r < \alpha, \\
		\frac{2(\alpha^2-1)}{\alpha}(r-r^{-1})- \frac{(1-\alpha^2)^2}{\alpha^2}, & \alpha \leq r.
	\end{cases}
\end{equation}
Each LARE version's advantages are combined in the above function $\ell$ inspired by Huber loss (\cite[Chapter 3.5: Ex. 5.4]{Huber1981}) using a parameter $\alpha > 1$.
To our knowledge, this loss function has not been proposed in the literature.
\begin{figure}[ht]
	\centering
	\begin{subfigure}{0.31\textwidth}
		\centering
		\includegraphics[width=\textwidth]{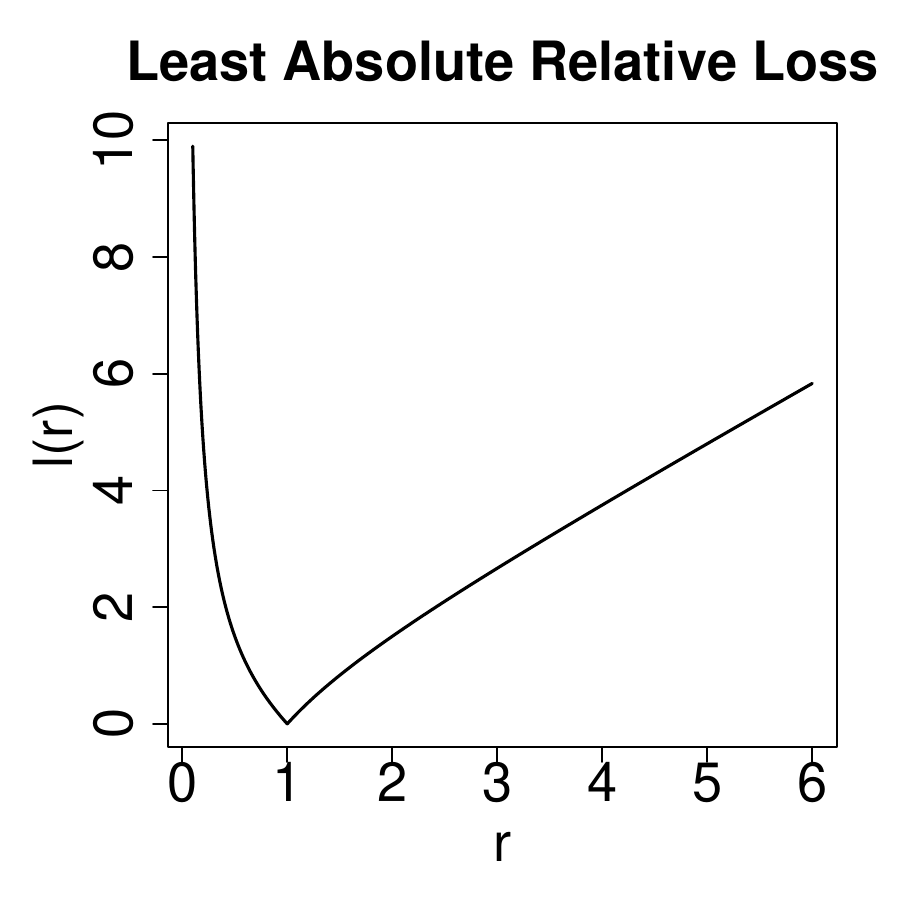}
	\end{subfigure}
	\hfill
	\begin{subfigure}{0.31\textwidth}
		\centering
		\includegraphics[width=\textwidth]{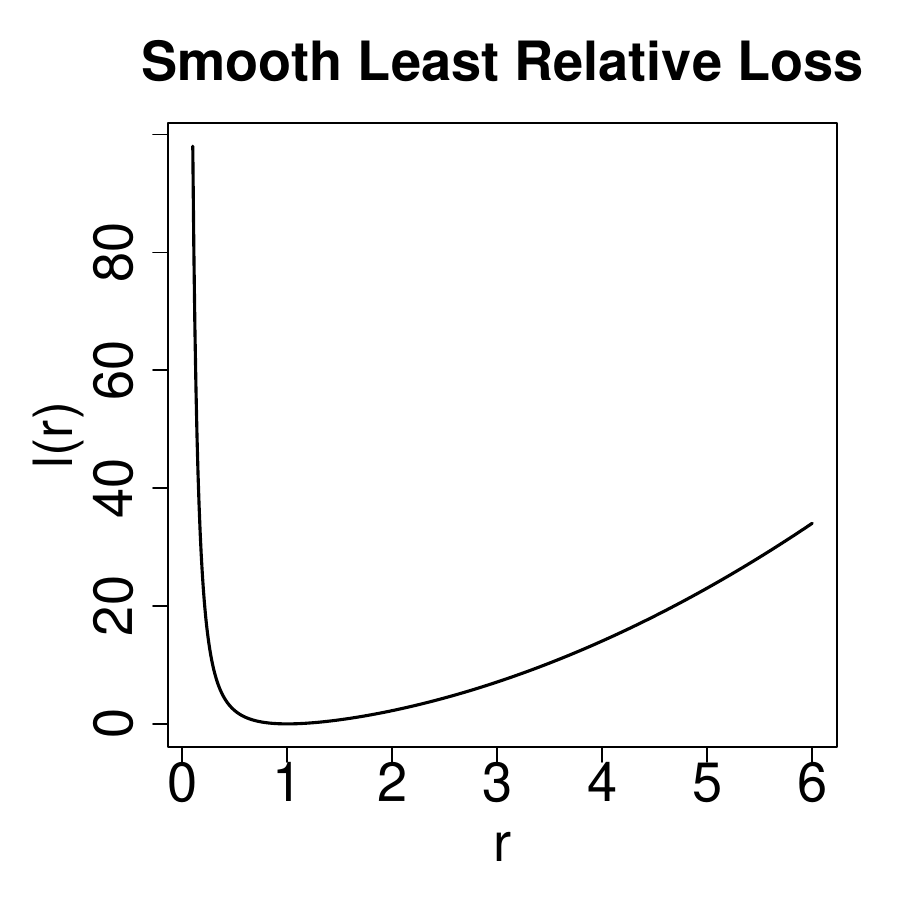}
	\end{subfigure}
	\hfill
	\begin{subfigure}{0.31\textwidth}
		\centering
		\includegraphics[width=\textwidth]{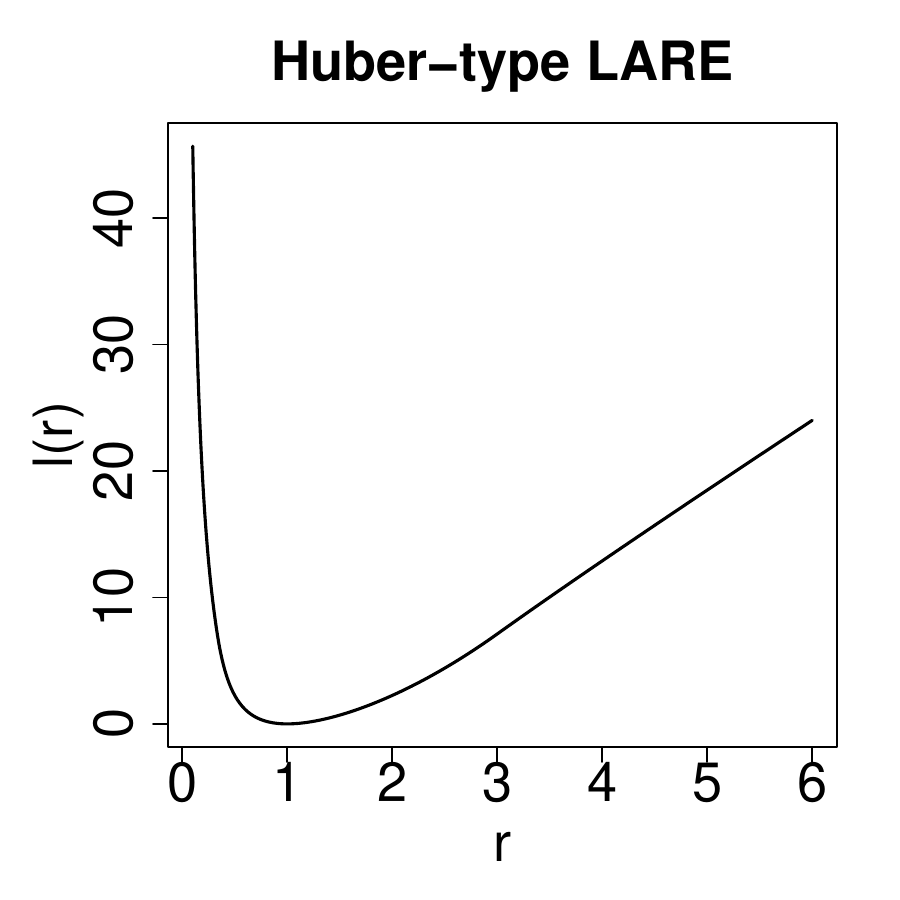}
	\end{subfigure}
	\caption{Plots of the representing functions $\ell$ of LARE loss functions; Huber-type least absolute relative loss function's representation $\ell$ with parameter $\alpha = 3$}
\end{figure}

\subsection{Least Product Relative Loss}
\begin{equation} \label{leastprodFehler}
	\ell(r)= |1-r|\cdot |r^{-1}-1| = r + r^{-1} - 2
\end{equation}
Taking the product of $|1-r|$ and $|r^{-1}-1|$ instead of the sum, gives the rb loss function, which is known in the literature as the \emph{least product relative error} (LPRE, cf. \cites[(3)]{ChenEtAl2016LPRE}[(3)]{MingEtAl2025}[(2.2)]{YangEtAl2023}). 
The first representation is essential for the loss function's name; the last one gives an easier formulation for calculations. 
\subsection{General Relative Loss}
Generalizing the concept of the previous loss functions, one gets the rb loss which is known as the \emph{general relative error} (\cite[(8)]{ChenEtAl2016LPRE})
\begin{equation} \label{generalFehler}
	\ell(r) = g(|1-r|, |r^{-1}-1|).
\end{equation}
The function  $g$ is not defined in \cite[Chapter 2.3]{ChenEtAl2016LPRE}, unless it is supposed to satisfy \anf{certain regularity conditions} that are not further specified.
For our purposes, we assume a measurable function
\begin{equation*}
	g: [0, \infty) \times [0, \infty) \to [0, \infty).
\end{equation*} 
Some of the prior examples already discussed are sum, product, and maximum as possible mappings $g$.
Some obvious additional examples are $g(a,b) = a^2+b^2$, $g(a,b) = \sqrt{a^2 + b^2}$, and $g(a,b) = \sqrt{a+b}$,
which give the representing functions
\begin{align}
	\ell(r) &= (1-r)^2 + (r^{-1}-1)^2, \label{generalsqFehler}\\
	\ell(r) &= \sqrt{(1-r)^2 + (r^{-1}-1)^2}, \label{generalsqrtsqFehler}\\
	\ell(r) &= \sqrt{|1-r|+|r^{-1}-1|}. \label{generalsqrtFehler}
\end{align}
These representing functions penalize underestimation and overestimation in the same manner.
An extreme version $g(a,b) = a + \exp(b)$, suggested in \cite[Chapter 2.3]{ChenEtAl2016LPRE}, penalizes underestimation more than overestimation.
Thus, we get (recall $\ell(1) = 0$) 
\begin{equation} \label{generalexpFehler}
	\ell(r) = |1-r| + \exp(|r^{-1}-1|) - 1.
\end{equation}
Note that using a general relative error does not automatically imply ratio-symmetry.
Applying both error terms, $|1-r|$ and $|r^{-1}-1|$, does not indicate that they are treated equally.
\begin{figure}[ht]
	\centering
	\begin{subfigure}{0.31\textwidth}
		\centering
		\includegraphics[width=\textwidth]{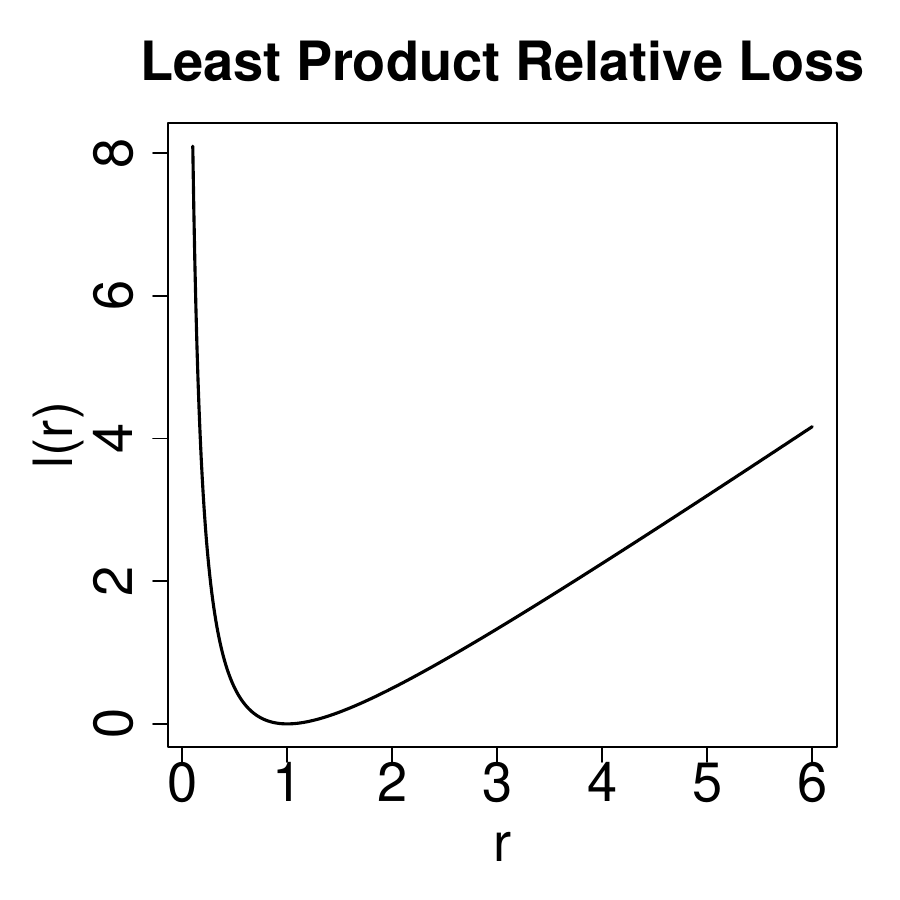}
		\caption{Plot of LPRE's representing function $\ell$ from \eqref{leastprodFehler}}
	\end{subfigure}
	\hfill
	\begin{subfigure}{0.63\textwidth}
		\centering
		\includegraphics[width=\textwidth]{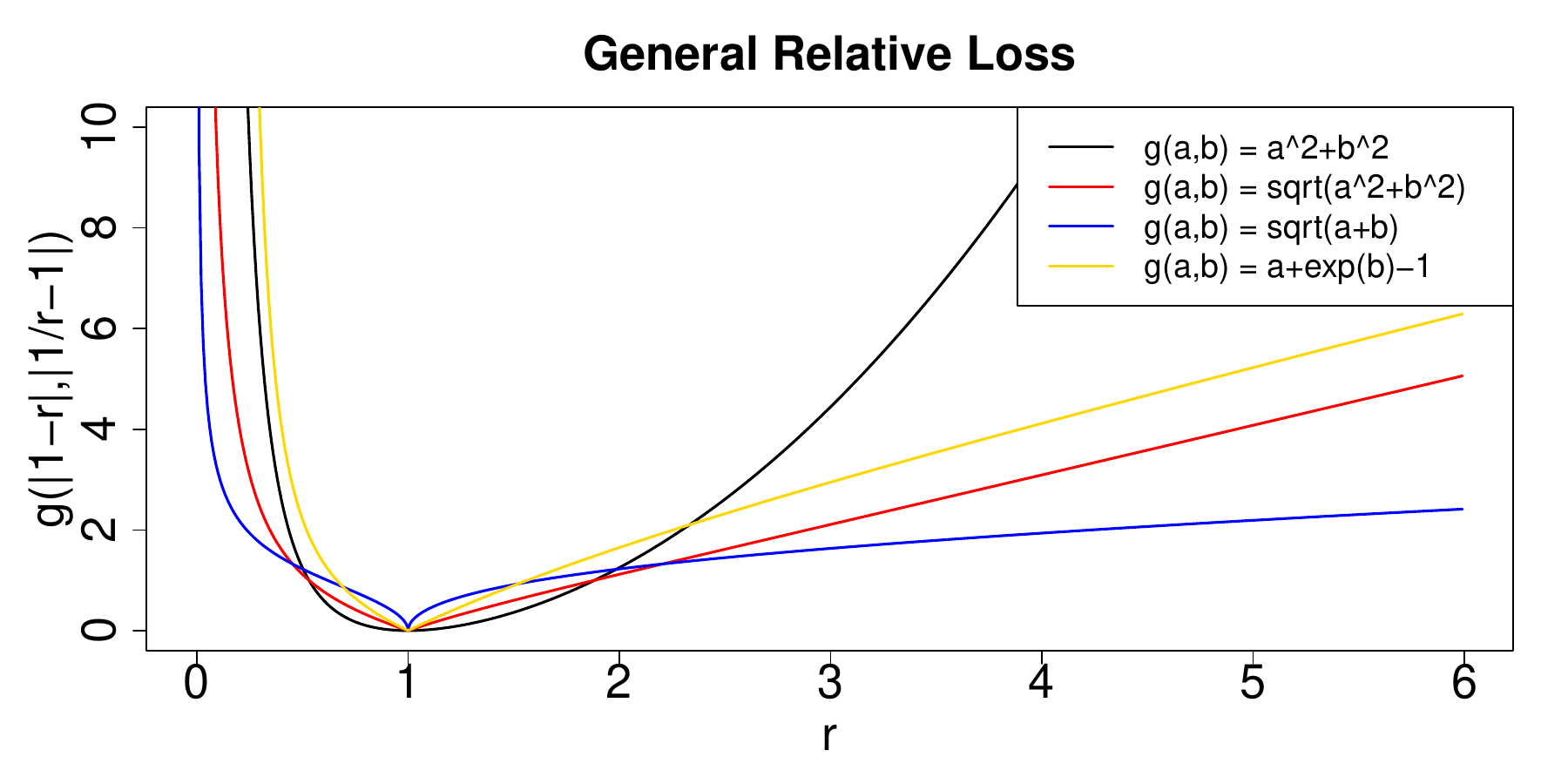}
		\caption{Plots of GRE's representing functions. Black: $\ell$ from \eqref{generalsqFehler}; Red:  $\ell$ from \eqref{generalsqrtsqFehler}; Blue: $\ell$ from \eqref{generalsqrtFehler}; Yellow: $\ell$ from \eqref{generalexpFehler}}
	\end{subfigure}
	\caption{Plots of LPRE's und GRE's representation functions $\ell$}
\end{figure}

\subsection{Insensitive Relative Loss}
As small deviations from the true value should not be penalized too hard or perhaps not at all, we adapt the classical distance-based  $\varepsilon$-insensitive loss, see e.g. \cite{Vapnik1995} and \cite{SmolaSchölkopf2004} to the ratio-based setting. 
We give two different versions.
The insensitive part can be defined ratio-symmetrically by adapting a ratio-symmetric representation function $\ell$. 
Basically, $\varepsilon$ does not need to have an upper bound, however, it makes sense to assume a somehow \emph{small} insensitive area around the minimum. 
Therefore, $\varepsilon$ should be significantly smaller than 1 in most applications.
\begin{align} 
	\ell(r) &= \max\{0, \max\{r, r^{-1}\} - 1 - \varepsilon \}, \hspace*{-2cm} &&\varepsilon \in (0,1), \label{insensitive1} \\
	\ell(r) &= \max\{0, r^{-1}+r-2-\varepsilon\}, \hspace*{-2cm} &&\varepsilon \in (0,1). \label{insensitive2}
\end{align}
\begin{figure}[H]
	\centering
	\begin{subfigure}{0.31\textwidth}
		\centering
		\includegraphics[width=\textwidth]{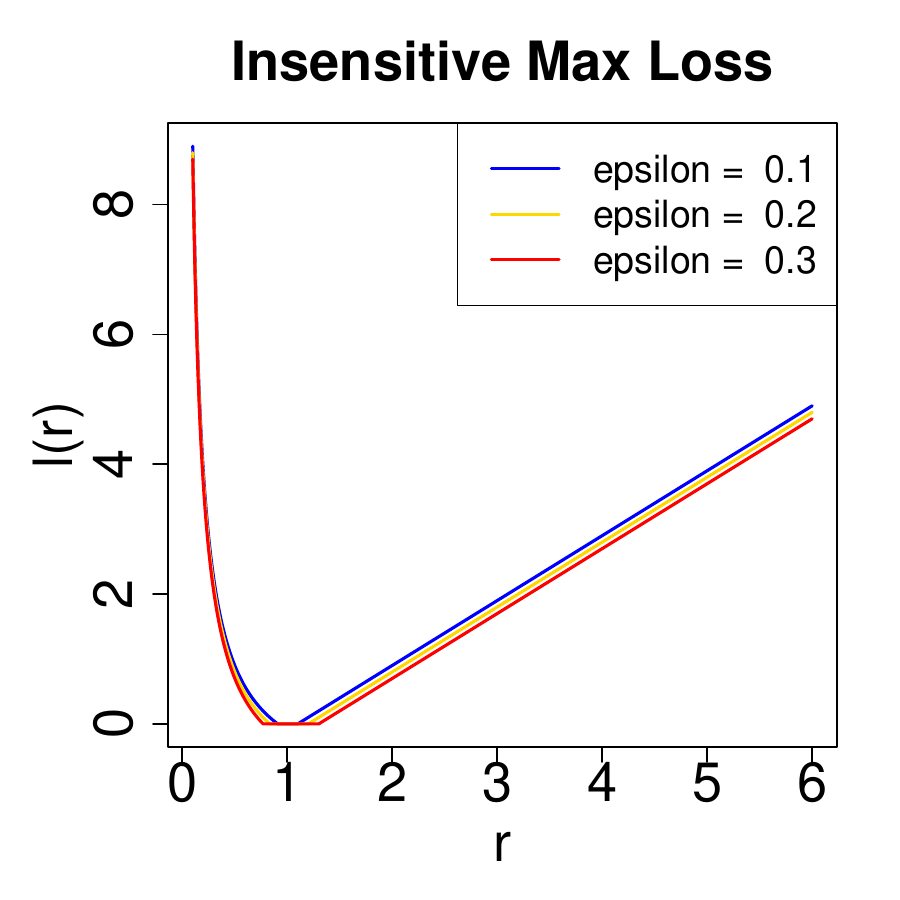}
		\caption{Plot of the representation functions $\ell$ of insensitive loss based on maximum loss for different $\varepsilon$}
	\end{subfigure}
	\hspace*{1cm}
	\begin{subfigure}{0.31\textwidth}
		\centering
		\includegraphics[width=\textwidth]{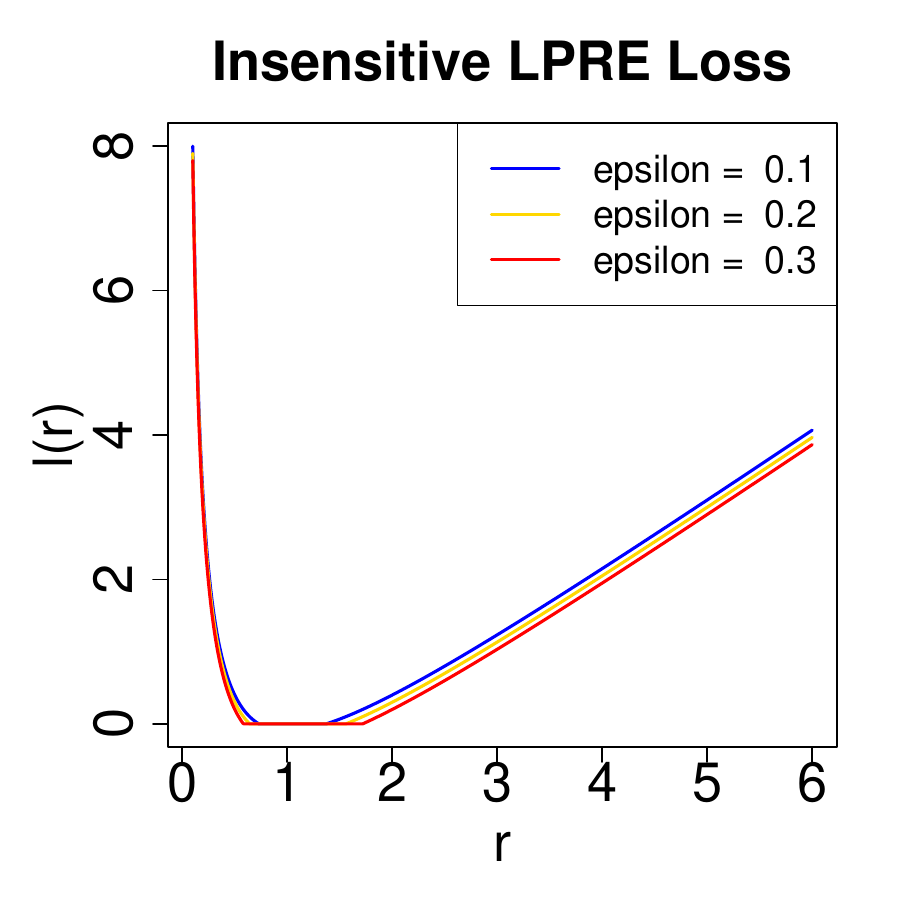}
		\caption{Plot of the representation functions $\ell$ of insensitive loss based on LPRE for various insensitivities $\varepsilon$}
	\end{subfigure}
\end{figure}

\subsection{Robust Relative Loss} \label{robustLossFunctions} 
Still using maximum loss and LPRE, we define more robust ratio-based loss functions with parameter $\alpha > 1$ inspired by Hampel's loss (\cites[Chapter 2.6: Ex. 1]{HampelEtAl1986}[Chapter 4.8]{Huber1981}[p. 2533]{KimScott2012}).
For each representation function, we give a robust and a robust plus insensitive version, which is obviously inspired by the famous $\varepsilon$-insensitive distance-based loss function (see e.g. \cite{Vapnik1995}), combining two favourable properties.
\begin{equation} \label{robustinsens1}
	\ell(r) = 
	\begin{cases}
		\max\{0, \max\{r, r^{-1}\} - 1 - \varepsilon \}, & r \in (\alpha^{-1}, \alpha), \\
		\alpha-1-\varepsilon, & r \not\in (\alpha^{-1}, \alpha).
	\end{cases}
\end{equation}
We require $\alpha > 1 + \varepsilon$ with $\varepsilon \geq 0$.
Here, $\varepsilon = 0$ means that there is no insensitivity around the minimum.
\begin{figure}[H] 
	\centering
	\begin{subfigure}{0.31\textwidth}
		\centering
		\includegraphics[width=\textwidth]{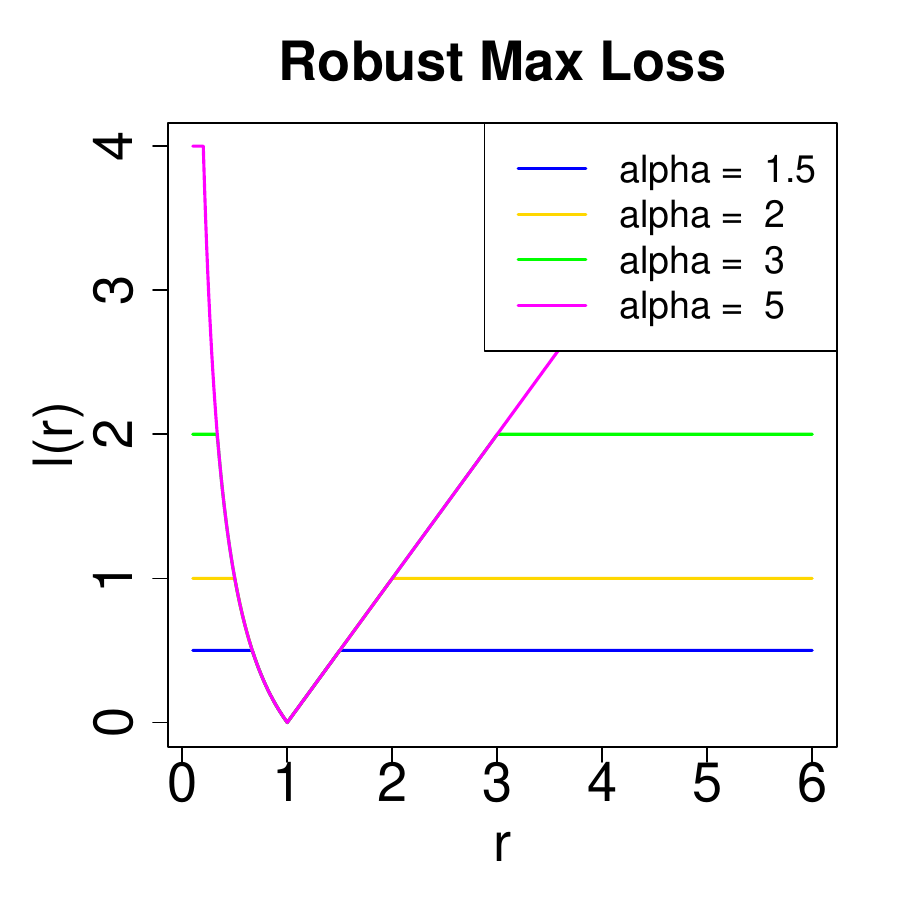}
	\end{subfigure}
	\hspace*{1cm}
	\begin{subfigure}{0.31\textwidth}
		\centering
		\includegraphics[width=\textwidth]{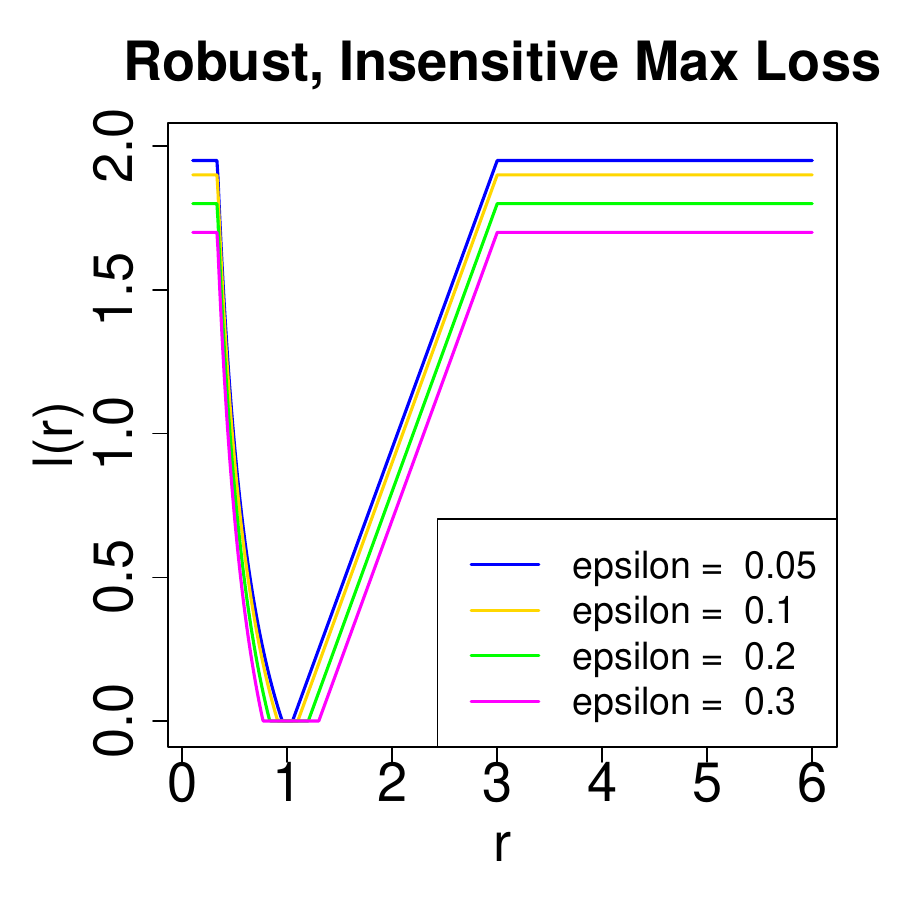}
	\end{subfigure}
	\caption{Plots of the representing functions $\ell$ of robust loss based on maximum loss function; Left: loss without insensitivity (i.e. $\varepsilon = 0$) for different choices of parameter $\alpha$; Right: robust loss functions with different insensitivity values $\varepsilon$ and parameter $\alpha = 3$}
\end{figure}
The second version is based on LPRE
\begin{equation} \label{robustinsens2}	
	\ell(r) = 
	\begin{cases}
		\max\{0, r^{-1}+r-2-\varepsilon\}, & r \in (\alpha^{-1}, \alpha), \\
		\alpha^{-1}+\alpha-2-\varepsilon, & r \not\in (\alpha^{-1}, \alpha).
	\end{cases}
\end{equation}
The parameters are required to satisfy $\alpha > 1 + \varepsilon, \varepsilon \geq 0$. 
In the sensitive setting, $\varepsilon = 0$ holds. \par
\begin{figure}[ht]
	\centering
	\begin{subfigure}{0.31\textwidth}
		\centering
		\includegraphics[width=\textwidth]{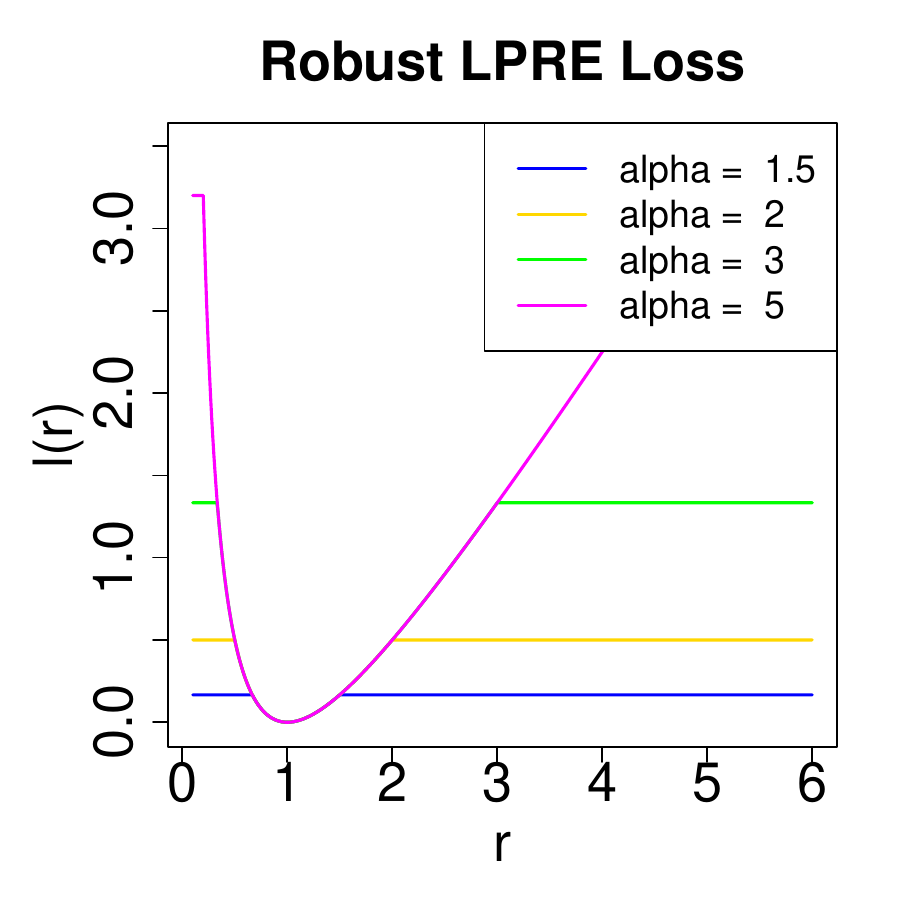}
	\end{subfigure}
	\hspace*{1cm}
	\begin{subfigure}{0.31\textwidth}
		\centering
		\includegraphics[width=\textwidth]{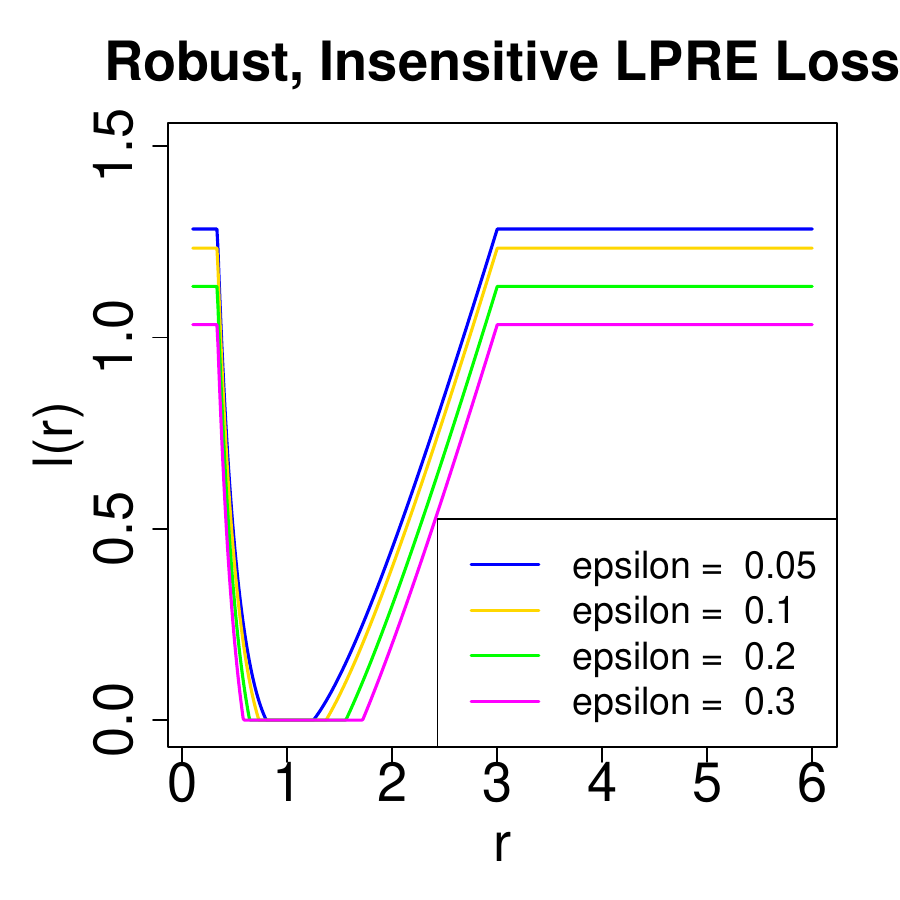}
	\end{subfigure}
	\caption{Plots of the representing functions $\ell$ of robust loss based on LPRE; Left: loss with $\varepsilon = 0$ (i.e. without insensitivity around 0) for different values of parameter $\alpha$; Right: $\ell$ for various insensitivity choices $\varepsilon$ and parameter $\alpha = 3$}
\end{figure}

\subsection{Smooth Robust Relative Loss Functions}
Instead of clipping the loss function at some point, \cite{FuEtAl2024} give a general framework to (smoothly) flatten general unbounded loss functions.
Taking some of the previous unbounded functions $\ell$, one can flatten these by using
\begin{equation*}
	\widehat{\ell}(r) := \frac{1}{\lambda} \biggl(1-\frac{1}{1+b\ell(r)}\biggr)
\end{equation*}
for some $\lambda, b > 0$.
Most properties still hold after this transformation, i.e. if $\ell$ is ratio-symmetric, continuous, differentiable or Lipschitz continuous (either locally or globally), $\widehat{\ell}$ is as well.
Note that convexity of $\ell$ does in general not transfer to $\widehat{\ell}$.
Using an unbounded loss $\ell$, $\lim_{r \to 0, r > 0} \ell(r) = \infty = \lim_{r \to \infty}\ell(r)$ holds.
However, $\widehat{\ell}$ is bounded as
\begin{equation*}
	\lim_{r \to 0, r > 0} \widehat{\ell}(r) = \frac{1}{\lambda} = \lim_{r \to \infty} \widehat{\ell}(r).
\end{equation*}
As \eqref{logcoshlogFehler} provides some nice properties (see Tables \ref{TablePropertiesl} and \ref{EigenschaftenL}, respectively), we modify this one to keep most of them.  
We get
\begin{equation} \label{robustlogcoshlogFehler}
	\ell(r) = \frac{1}{\lambda}\biggl(1 - \frac{1}{1+b\log(\cosh(\log(r)))}\biggr)
\end{equation}
with parameters $\lambda, b > 0$.
Here, $\lambda$ can reduce the impact of noise whereas $b$ pays attention to usual points and especially controls the loss function's compactness and growth (see \cite{FuEtAl2024}).
Indeed, \cite{FuEtAl2024} not only used some parameter $b$ but a nonnegative function $b(r)$. 
One could think about expanding the examples to nonnegative functions $b(r)$ for future research.
Since we have used differentiable log-cosh-log loss function before, \cite{FuEtAl2024}'s robust version was already smooth.\par
\begin{figure}[ht]
	\centering
	\begin{subfigure}{0.31\textwidth}
		\centering
		\includegraphics[width=\textwidth]{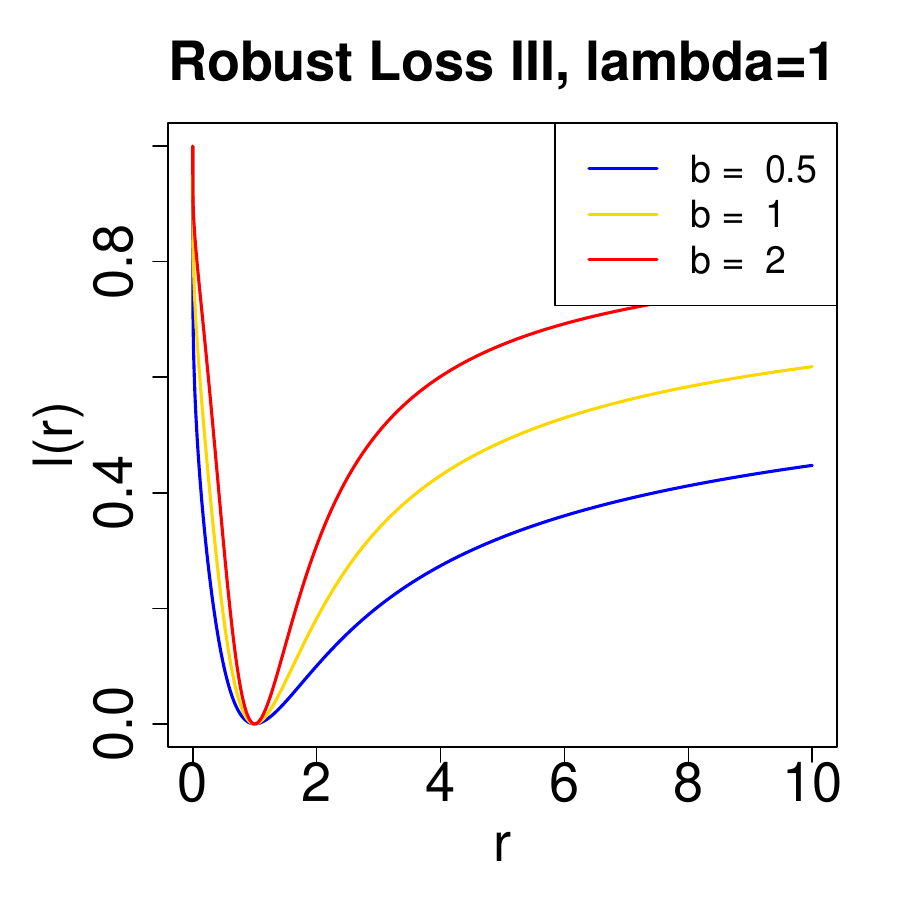}
	\end{subfigure}
	\hspace*{1cm}
	\begin{subfigure}{0.31\textwidth}
		\centering
		\includegraphics[width=\textwidth]{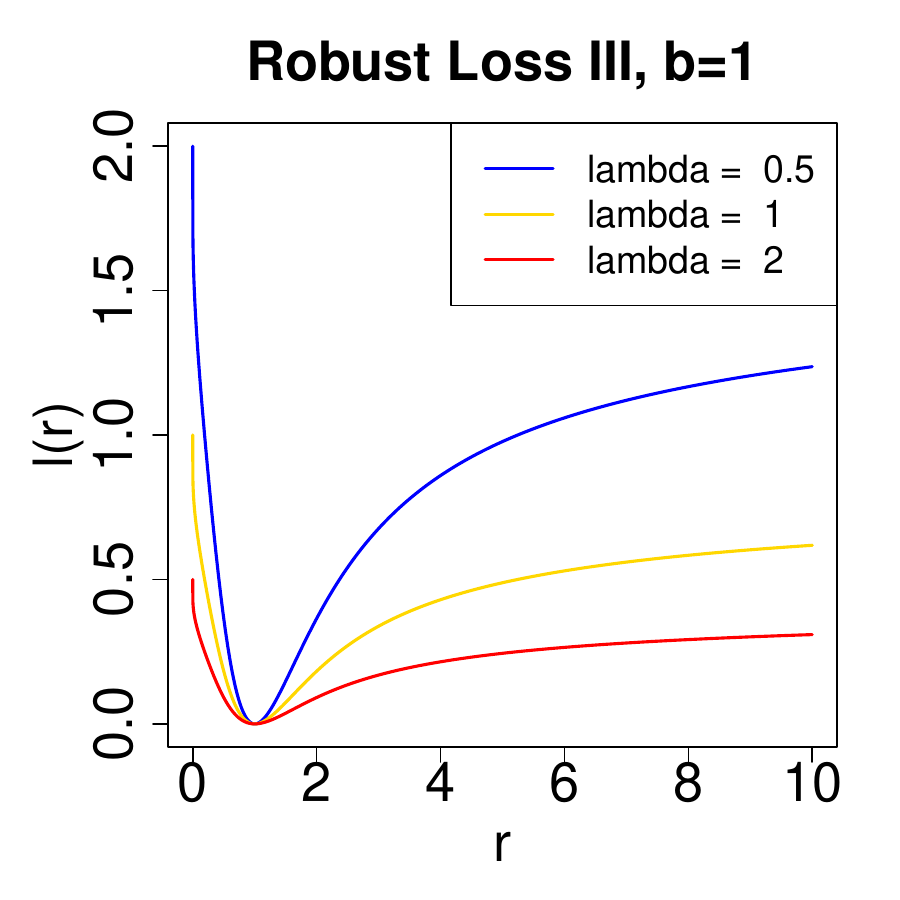}
	\end{subfigure}
	\caption{Plots of the representing functions $\ell$ of \cite{FuEtAl2024}'s robust loss functions based on log-cosh-log loss for certain choices of parameters $\lambda$ and $b$}
\end{figure}
Now, we extend the idea of Hampel's piecewise defined loss function (\cite[Chapter 2.6: Ex. 1]{HampelEtAl1986}) to smoothed versions of (smooth) LARE by adding some part(s) before cutting the loss. 
A perfect transfer of Hampel's loss (\cite[Chapter 2.6: Ex. 1]{HampelEtAl1986}) can be achieved when the limits $\lim_{r \to 0, r>0} \ell(r)$ and $\lim_{r \to \infty} \ell(r)$ agree.
This can be done in the following manner. 
We requested continuity and continuous differentiability. 
As far as we know, the following loss functions have not been proposed in the literature. 
Like Hampel's loss function, the first version deals with three parameters $1 < \alpha < \beta < \gamma$.
\begin{equation} \label{smoothrobustFehler4}
	\ell(r) = 
	\begin{cases}
		\frac{(\alpha^2-1)\beta\gamma}{\alpha(\beta\gamma+1)(\beta - \gamma)} \Bigl((\beta - \frac{1}{\beta})^2-(\gamma-\frac{1}{\gamma})^2\Bigr) - (\alpha - \frac{1}{\alpha})^2, & r \leq \frac{1}{\gamma},  \\[6pt]
		\frac{(\alpha^2-1)\beta\gamma}{\alpha(\beta\gamma+1)(\beta - \gamma)}\biggl(\!\!\Bigl(\!(\frac{1}{r}-r)-(\gamma-\frac{1}{\gamma})\!\Bigr)^2 + (\beta - \frac{1}{\beta})^2-(\gamma-\frac{1}{\gamma})^2\!\biggr)- (\alpha - \frac{1}{\alpha})^2, & \frac{1}{\gamma} < r \leq \frac{1}{\beta}, \\[6pt]
		2\frac{\alpha^2-1}{\alpha}(\frac{1}{r}-r) - (\alpha-\frac{1}{\alpha})^2, &\frac{1}{\beta} < r \leq \frac{1}{\alpha}, \\[6pt]
		(r-\frac{1}{r})^2, & \frac{1}{\alpha} < r < \alpha, \\[6pt]
		2\frac{\alpha^2-1}{\alpha}(r-\frac{1}{r}) - (\alpha-\frac{1}{\alpha})^2, & \alpha \leq r < \beta, \\[6pt]
		\frac{(\alpha^2-1)\beta\gamma}{\alpha(\beta\gamma+1)(\beta - \gamma)}\biggl(\!\!\Bigl(\!(r - \frac{1}{r})-(\gamma-\frac{1}{\gamma})\Bigr)^2 + (\beta - \frac{1}{\beta})^2-(\gamma-\frac{1}{\gamma})^2 \!\biggr) - (\alpha - \frac{1}{\alpha})^2, & \beta \leq r < \gamma, \\[6pt]
		\frac{(\alpha^2-1)\beta\gamma}{\alpha(\beta\gamma+1)(\beta - \gamma)} \Bigl((\beta - \frac{1}{\beta})^2-(\gamma-\frac{1}{\gamma})^2\Bigr) - (\alpha - \frac{1}{\alpha})^2, & \gamma \leq r.
	\end{cases}
\end{equation}
A simpler version reduces the number of parameters to two $1 < \alpha < \beta$ keeping ratio-symmetry and smooth robustness.
\begin{equation} \label{smoothrobustFehler5}
	\ell(r) = 
	\begin{cases}
		\frac{\alpha^2-1}{(\alpha-\beta)(\alpha \beta +1)} (\beta^2-1)((\alpha-\alpha^{-1})-(\beta - \beta^{-1})), & r \leq \beta^{-1}, \\
		\frac{\alpha^2-1}{(\alpha-\beta)(\alpha \beta +1)} \biggl(\beta(r-r^{-1})^2+2(\beta^2-1)(r-r^{-1})+\frac{(\beta^2-1)(\alpha^2-1)}{\alpha}\biggr), & \beta^{-1} < r \leq \alpha^{-1}, \\
		(r-r^{-1})^2, & \alpha^{-1} < r < \alpha, \\
		\frac{\alpha^2-1}{(\alpha-\beta)(\alpha \beta +1)} \biggl(\beta (r-r^{-1})^2-2(\beta^2-1)(r-r^{-1})+\frac{(\beta^2-1)(\alpha^2-1)}{\alpha}\biggr), & \alpha \leq r < \beta, \\
		\frac{\alpha^2-1}{(\alpha-\beta)(\alpha \beta +1)} (\beta^2-1)((\alpha-\alpha^{-1})-(\beta-\beta^{-1})), & \beta \leq r.
	\end{cases}
\end{equation}
\begin{figure}[ht]
	\centering
	\begin{subfigure}{0.31\textwidth}
		\centering
		\includegraphics[width=\textwidth]{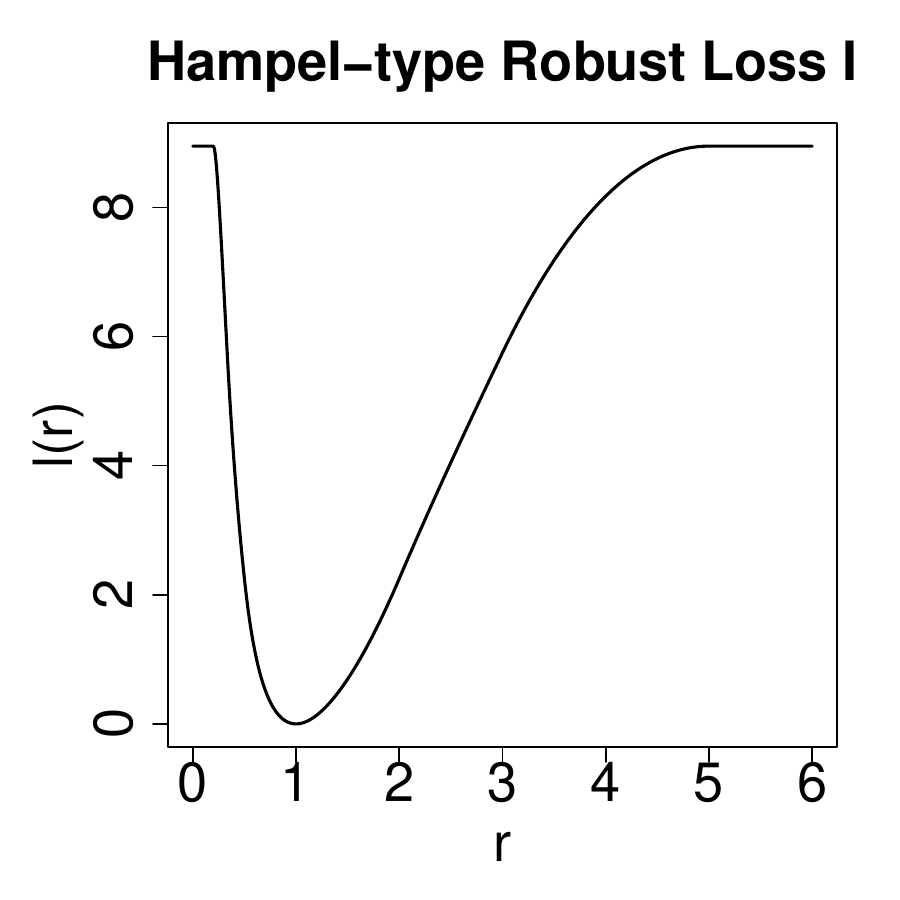}
	\end{subfigure}
	\hspace*{1cm}
	\begin{subfigure}{0.31\textwidth}
		\centering
		\includegraphics[width=\textwidth]{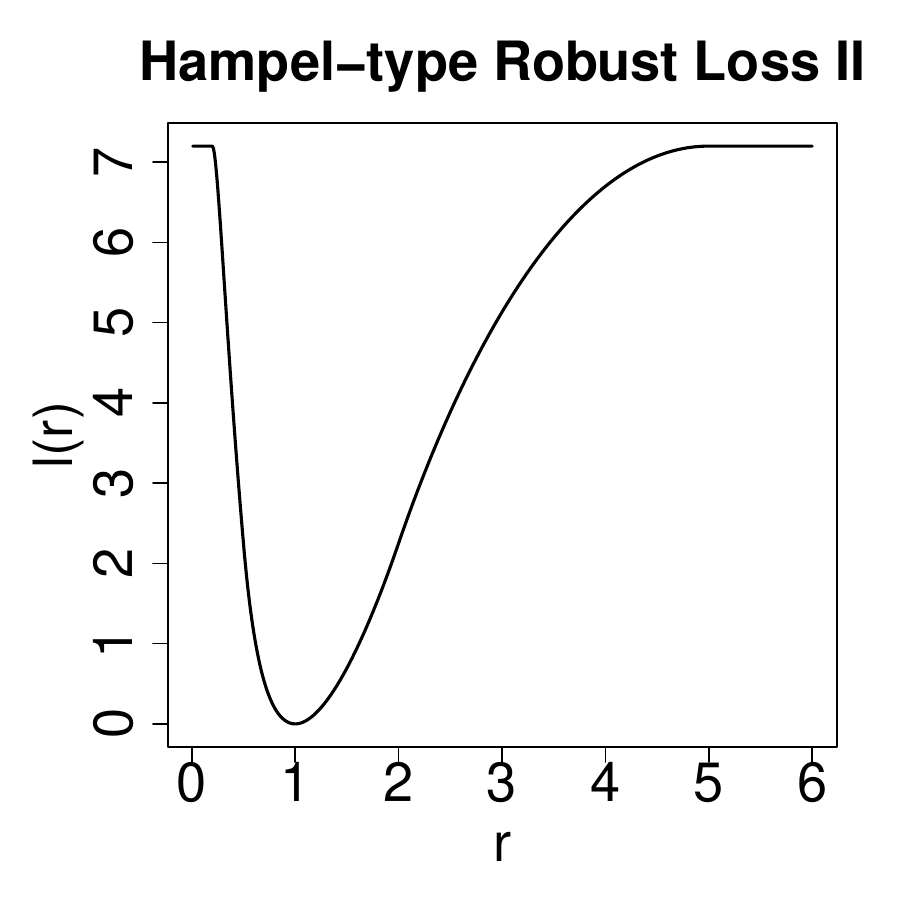}
	\end{subfigure}
	\caption{Plots of the representing functions $\ell$ for Hampel-type loss function; Left: $\ell$ from \eqref{smoothrobustFehler4} with three parameters $\alpha = 2, \beta = 3$, and $\gamma = 5$; Right: $\ell$ from \eqref{smoothrobustFehler5} for two parameters $\alpha = 2$ and $\beta = 5$}
\end{figure}

\subsection{Weighted Relative Loss}
Finally, let us once again focus on rb loss functions that do not require ratio-symmetry, i.e. on rb loss function that can be used if
overestimation and underestimation by the same factor are considered as harmful in a different manner.
In distance-based setting, for example, LINEX (\cite{ChangHung2007}) and BLINEX (\cite{TangEtAl2021}) loss as well as the classical $\tau$-pinball loss function, where $\tau\in(0,1)$ (see e.g. 
\cite{KoenkerBassett1978,Koenker2005, SteinwartChristmann2011} and 
\cite[Chapter III.N]{JadonEtAl2022}), are used when asymmetry is favorable. 
Here, we want to achieve asymmetry only by a parameter $\tau > 0$ like in the distance-based pinball loss.
In this case, $\tau < 1$ focuses on underestimation, whereas $\tau > 1$ highlights overestimation. 
\begin{align} 
	\ell(r) &= 
	\begin{cases} \label{gewFehler1}
		\tau^{-1} (\max\{r, r^{-1}\}-1), & r < 1, \\
		\tau (\max\{r, r^{-1}\}-1), & r \geq 1.
	\end{cases} \\
	\ell(r) &=
	\begin{cases} \label{gewFehler2}
		\tau^{-1} (r+r^{-1}-2), & r < 1, \\
		\tau (r+r^{-1}-2), & r \geq 1.
	\end{cases} \\
	\ell(r) &= 
	\begin{cases} \label{gewFehler3}
		\tau^{-1}(r-r^{-1})^2, & r < 1, \\
		\tau (r-r^{-1})^2, & r \geq 1.
	\end{cases}
\end{align}
\begin{figure}[ht]
	\centering
	\begin{subfigure}{0.31\textwidth}
		\centering
		\includegraphics[width=\textwidth]{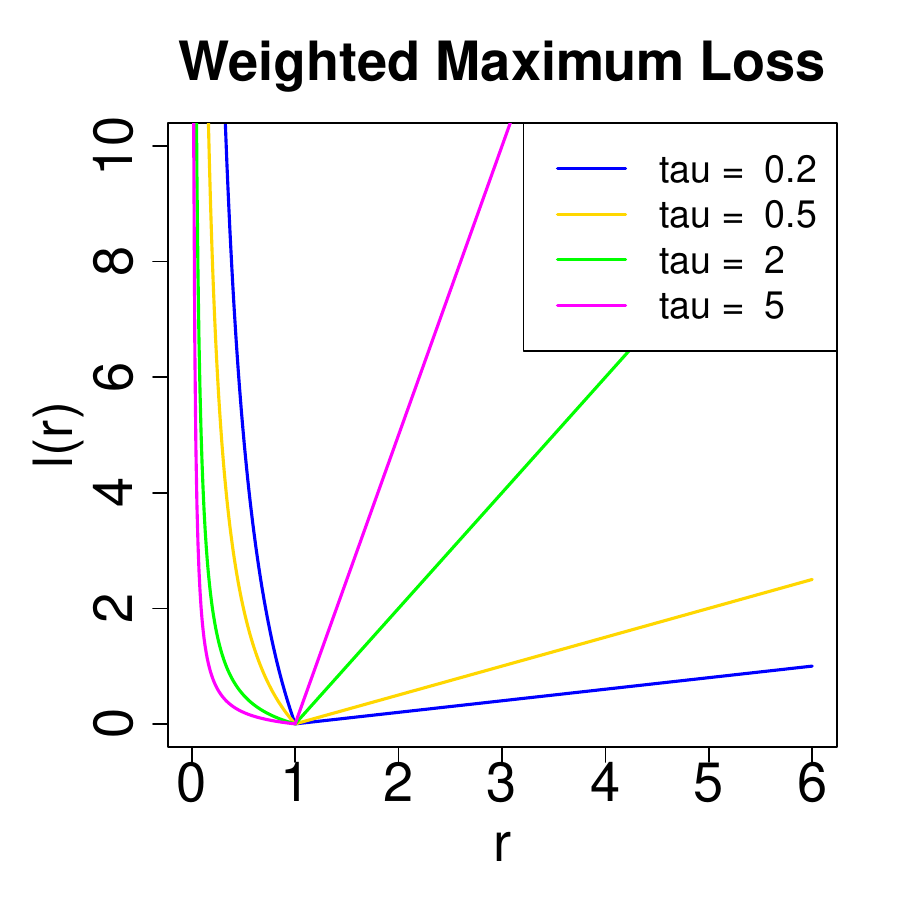}
	\end{subfigure}
	\hfill
	\begin{subfigure}{0.31\textwidth}
		\centering
		\includegraphics[width=\textwidth]{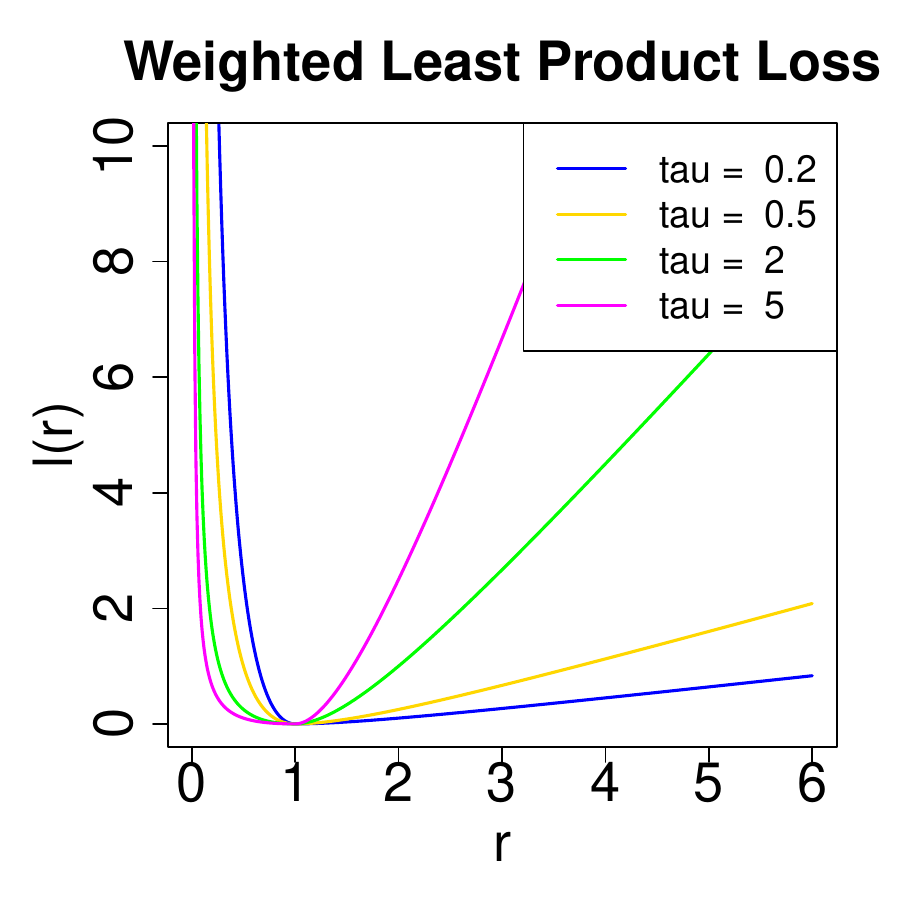}
	\end{subfigure}
	\hfill
	\begin{subfigure}{0.31\textwidth}
		\centering
		\includegraphics[width=\textwidth]{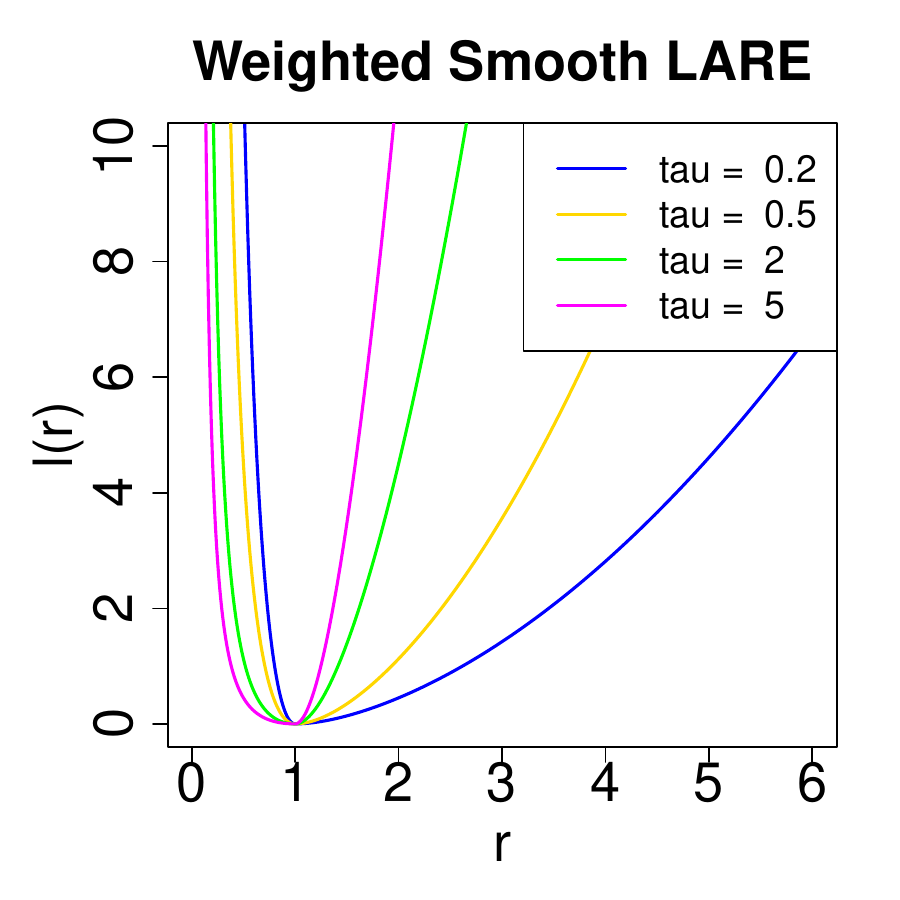}
	\end{subfigure}
	\caption{Plots of the representing functions $\ell$ of weighted loss functions for various weights $\tau$}
\end{figure}

\subsection{Overview} \label{overview_examples}
Table \ref{TablePropertiesl} gives some properties of the representing functions $\ell$ corresponding to ratio-based loss functions which we considered in this section.
\begin{longtable}[c]{|c||c|c|c|c|c|c|}
	\caption{Properties of ratio-based representation functions $\ell$. The symbol $"\checkmark"$ indicates that $\ell$ has this property, whereas the symbol $"\xmark"$ indicates that $\ell$ does not have this property for at least some choices of the parameter (if existing).} \label{TablePropertiesl} \\
	
	&&&&\multicolumn{2}{c|}{Lipschitz}&\\
	$\ell$ & ratio-symmetry & convex & continuous & locally & globally & differentiable \\
	\hline
	\endfirsthead
	&&&&\multicolumn{2}{c|}{Lipschitz}&\\
	$\ell$ & ratio-symmetry & convex & continuous & locally & globally & differentiable \\
	\hline
	\endhead
	\eqref{log(1+r)Loss} & \checkmark & \xmark & \checkmark & \checkmark & \xmark & \checkmark\\
	\eqref{bsp_3_Loss} & \checkmark & \xmark & \checkmark & \checkmark & \xmark & \checkmark\\
	\eqref{logquadrFehler} & \checkmark & \xmark & \checkmark & \checkmark & \xmark & \checkmark  \\
	\eqref{logabsFehler} & \checkmark & \xmark & \checkmark & \checkmark & \xmark & \xmark \\
	\eqref{mixedlogFehler} & \checkmark &\xmark &\checkmark &\checkmark &\xmark &\checkmark \\
	\eqref{logcoshFehler} & \xmark & \checkmark & \checkmark & \checkmark & \checkmark & \checkmark \\
	\eqref{coshlogFehler} & \checkmark & \checkmark & \checkmark & \checkmark & \xmark & \checkmark \\
	\eqref{logcoshlogFehler} & \checkmark & \xmark & \checkmark & \checkmark & \xmark & \checkmark \\
	\eqref{maxFehler}  & \checkmark &\checkmark & \checkmark& \checkmark & \xmark & \xmark \\
	\eqref{maxlogFehler} & \xmark & \xmark & \checkmark & \checkmark & \xmark & \xmark \\	
	\eqref{absFehler}  & \xmark &\checkmark& \checkmark& \checkmark & \checkmark & \xmark \\
	\eqref{quadrFehler} & \xmark & \checkmark & \checkmark & \checkmark & \xmark & \checkmark \\
	\eqref{mixedrelFehler} & \xmark & \checkmark & \checkmark & \checkmark & \checkmark & \checkmark \\
	\eqref{invabsFehler} & \xmark & \xmark & \checkmark & \checkmark & \xmark & \xmark \\
	\eqref{invsqFehler} & \xmark & \xmark & \checkmark & \checkmark & \xmark & \checkmark \\
	\eqref{mixedinvFehler} & \xmark & \xmark & \checkmark & \checkmark & \xmark & \checkmark \\
	\eqref{leastabsFehler} & \checkmark & \xmark & \checkmark & \checkmark & \xmark & \xmark \\
	\eqref{smoothleastabsFehler} & \checkmark &\checkmark & \checkmark & \checkmark & \xmark & \checkmark \\
	\eqref{mixedleastrelFehler} & \checkmark & \xmark & \checkmark & \checkmark & \xmark & \checkmark \\
	\eqref{leastprodFehler}  & \checkmark &\checkmark & \checkmark & \checkmark & \xmark & \checkmark \\
	\eqref{generalsqFehler} & \checkmark & \checkmark & \checkmark &\checkmark &\xmark & \checkmark \\
	\eqref{generalsqrtsqFehler} & \checkmark & \xmark & \checkmark &\checkmark & \xmark & \xmark \\
	\eqref{generalsqrtFehler} & \checkmark & \xmark & \checkmark &\xmark &\xmark & \xmark \\
	\eqref{generalexpFehler} & \xmark & \xmark & \checkmark & \checkmark & \xmark & \xmark \\
	\eqref{insensitive1}  & \checkmark &\checkmark & \checkmark& \checkmark& \xmark & \xmark \\
	\eqref{insensitive2}  & \checkmark &\checkmark&\checkmark &\checkmark & \xmark&\xmark \\
	\eqref{robustinsens1} & \checkmark &\xmark & \checkmark &\checkmark & \checkmark & \xmark \\
	\eqref{robustinsens2} & \checkmark &\xmark&\checkmark&\checkmark&\checkmark&\xmark\\
	\eqref{robustlogcoshlogFehler} & \checkmark & \xmark & \checkmark & \checkmark & \xmark & \checkmark \\
	\eqref{smoothrobustFehler4} & \checkmark & \xmark & \checkmark & \checkmark & \checkmark & \checkmark \\
	\eqref{smoothrobustFehler5} & \checkmark & \xmark & \checkmark & \checkmark & \checkmark & \checkmark \\
	\eqref{gewFehler1} & \xmark & \checkmark & \checkmark & \checkmark & \xmark & \xmark \\
	\eqref{gewFehler2} & \xmark & \checkmark & \checkmark & \checkmark & \xmark & \checkmark \\
	\eqref{gewFehler3} & \xmark & \checkmark & \checkmark & \checkmark & \xmark & \checkmark \\
\end{longtable}	
Please note, that margin-based or distance-based loss functions which are convex and simultaneously Lipschitz continuous can be interesting in the sense that they often yield existence \emph{and} uniqueness \emph{and} good statistical robustness properties of certain machine learning methods including general kernel based approaches, see e.g. \cite{ChristmannVanMessemSteinwart2009}. 
In contrast to that, many ratio-based loss functions proposed in the literature are not simultaneously convex and Lipschitz continuous. 
Note that the convexity of $\ell$ does in general not automatically transfer to the corresponding  ratio-based loss function.
Furthermore, some Huber-type and robust rb loss functions were used to combine favourable properties of several rb loss functions, especially continuous differentiability and better robustness results. 
However, the choice of parameters and a case-by-case analysis can lead to higher computation time.\par 
Combining the results from Table \ref{TablePropertiesl} with results from Chapter \ref{Kap_Eigenschaften} gives the following information about the loss function's properties  (i.e. with respect to the third variable $t$, cf. \cite[Chapter 2.2]{SC08})  when choosing a certain link function $u$.
If properties hold for both, $c = 0$ and $c > 0$, there is only one mark in the segment.
Otherwise, we split the entry with the first mark referring to assumption $c > 0$, whereas for the second mark, we assume a strictly rb loss function, i.e. $c = 0$. 
{\renewcommand{\arraystretch}{1.2}
	\begin{longtable}[c]{|c|c||c|c|c|c|c|}
		\caption{Properties of loss functions $L$ (w.r.t. the $3^{rd}$ argument $t$) depending on the chosen link function $u$ and on the representation function $\ell$. In each entry, symbols on the left hand side refer to $c>0$, while symbols on the right hand side refer to $c=0$. A single entry holds for both situations. The symbol \anf{\checkmark} indicates that $L$ has this property, whereas \anf{\xmark} indicates that $L$ does not have this property at least for some choices of the parameter (if existing).}
		\label{EigenschaftenL} \\
		
		&& \multicolumn{5}{c|}{Loss function $L$ (properties w.r.t. $t$) }\\
		$\ell$ & $u(t)$ & convex & continuous & locally Lipschitz & globally Lipschitz & differentiable \\
		\hline
		\endfirsthead
		&& \multicolumn{5}{c|}{Loss function $L$ (properties w.r.t. $t$) }\\
		$\ell$ & $u(t)$ & convex & continuous & locally Lipschitz & globally Lipschitz & differentiable \\
		\hline
		\endhead
		\eqref{log(1+r)Loss} & $\exp(t)$ & \xmark / \checkmark & \checkmark & \checkmark & \checkmark & \checkmark \\
		\eqref{log(1+r)Loss} & $\frac{1}{1+\exp(-t)}$ & \xmark & \checkmark & \checkmark & \checkmark & \checkmark \\
		\hline
		\eqref{bsp_3_Loss} & $\exp(t)$ & \xmark / \checkmark & \checkmark & \checkmark & \checkmark & \checkmark\\
		\eqref{bsp_3_Loss} & $\frac{1}{1+\exp(-t)}$ & \xmark & \checkmark & \checkmark  & \checkmark  & \checkmark  \\
		\hline
		\eqref{logquadrFehler} & $\exp(t)$ & \xmark / \checkmark & \checkmark & \xmark & \xmark & \checkmark \\
		\eqref{logquadrFehler} & $\frac{1}{1+\exp(-t)}$ & \xmark & \checkmark & \checkmark / \xmark & \checkmark / \xmark & \checkmark\\
		\hline
		\eqref{logabsFehler} & $\exp(t)$ & \xmark / \checkmark & \checkmark & \checkmark & \checkmark & \xmark\\
		\eqref{logabsFehler} & $\frac{1}{1+\exp(-t)}$ &\xmark &\checkmark& \checkmark & \checkmark & \xmark \\
		\hline
		\eqref{mixedlogFehler} & $\exp(t)$ & \xmark / \checkmark & \checkmark& \checkmark & \checkmark & \checkmark\\
		\eqref{mixedlogFehler}& $\frac{1}{1+\exp(-t)}$ &\xmark &\checkmark & \checkmark & \checkmark & \checkmark\\
		\hline
		\eqref{logcoshFehler} & $\exp(t)$ & \xmark &\checkmark & \checkmark / \xmark & \xmark & \checkmark \\
		\eqref{logcoshFehler} & $\frac{1}{1+\exp(-t)}$ & \xmark & \checkmark & \checkmark / \xmark & \checkmark / \xmark & \checkmark \\
		\hline
		\eqref{coshlogFehler} & $\exp(t)$ & \xmark / \checkmark & \checkmark & \xmark & \xmark & \checkmark \\
		\eqref{coshlogFehler} &$\frac{1}{1+\exp(-t)}$ & \xmark & \checkmark & \checkmark / \xmark & \checkmark / \xmark & \checkmark\\
		\hline
		\eqref{logcoshlogFehler} & $\exp(t)$ & \xmark / \checkmark & \checkmark & \checkmark & \checkmark & \checkmark \\
		\eqref{logcoshlogFehler} & $\frac{1}{1+\exp(-t)}$ &\xmark & \checkmark & \checkmark & \checkmark & \checkmark\\
		\hline
		\eqref{maxFehler}  & $\exp(t)$ & \xmark / \checkmark & \checkmark& \xmark & \xmark & \xmark \\
		\eqref{maxFehler} &$\frac{1}{1+\exp(-t)}$ & \xmark & \checkmark & \checkmark / \xmark & \checkmark / \xmark & \xmark \\
		\hline
		\eqref{maxlogFehler} & $\exp(t)$ & \xmark / \checkmark & \checkmark & \checkmark & \checkmark & \xmark \\
		\eqref{maxlogFehler} & $\frac{1}{1+\exp(-t)}$ & \xmark & \checkmark & \checkmark & \checkmark & \xmark \\			
		\hline
		\eqref{absFehler}  & $\exp(t)$ &\xmark & \checkmark & \checkmark / \xmark & \xmark & \xmark \\
		\eqref{absFehler} &$\frac{1}{1+\exp(-t)}$ &\xmark & \checkmark & \checkmark / \xmark & \checkmark / \xmark & \xmark \\
		\hline
		\eqref{quadrFehler} & $\exp(t)$ &\xmark & \checkmark & \checkmark / \xmark & \xmark & \checkmark \\
		\eqref{quadrFehler} &$\frac{1}{1+\exp(-t)}$ &\xmark & \checkmark & \checkmark / \xmark & \checkmark / \xmark & \checkmark\\
		\hline
		\eqref{mixedrelFehler} & $\exp(t)$ & \xmark & \checkmark & \checkmark / \xmark & \xmark & \checkmark \\
		\eqref{mixedrelFehler} & $\frac{1}{1+\exp(-t)}$ &\xmark & \checkmark & \checkmark / \xmark & \checkmark / \xmark & \checkmark\\
		\hline
		\eqref{invabsFehler} & $\exp(t)$ & \xmark & \checkmark & \xmark & \xmark & \xmark \\
		\eqref{invabsFehler} & $\frac{1}{1+\exp(-t)}$ & \xmark & \checkmark & \checkmark & \checkmark / \xmark & \xmark \\
		\hline
		\eqref{invsqFehler} & $\exp(t)$ &\xmark & \checkmark& \xmark & \xmark &  \checkmark\\
		\eqref{invsqFehler} & $\frac{1}{1+\exp(-t)}$ &\xmark & \checkmark & \checkmark & \checkmark / \xmark & \checkmark\\
		\hline
		\eqref{mixedinvFehler} & $\exp(t)$ &\xmark & \checkmark & \xmark & \xmark &  \checkmark\\
		\eqref{mixedinvFehler} & $\frac{1}{1+\exp(-t)}$ &\xmark & \checkmark & \checkmark & \checkmark / \xmark & \checkmark\\
		\hline
		\eqref{leastabsFehler} & $\exp(t)$ & \xmark / \checkmark & \checkmark & \xmark & \xmark & \xmark \\
		\eqref{leastabsFehler} & $\frac{1}{1+\exp(-t)}$ & \xmark & \checkmark & \checkmark / \xmark & \checkmark / \xmark & \xmark \\
		\hline
		\eqref{smoothleastabsFehler} & $\exp(t)$ & \xmark / \checkmark & \checkmark & \xmark & \xmark & \checkmark \\
		\eqref{smoothleastabsFehler} & $\frac{1}{1+\exp(-t)}$ & \xmark & \checkmark & \checkmark / \xmark & \checkmark / \xmark & \checkmark\\
		\hline
		\eqref{mixedleastrelFehler} & $\exp(t)$ & \xmark / \checkmark & \checkmark & \xmark & \xmark & \checkmark\\
		\eqref{mixedleastrelFehler} &$\frac{1}{1+\exp(-t)}$ & \xmark & \checkmark & \checkmark / \xmark & \checkmark / \xmark & \checkmark\\
		\hline
		\eqref{leastprodFehler} & $\exp(t)$ & \xmark / \checkmark & \checkmark & \xmark & \xmark &  \checkmark \\
		\eqref{leastprodFehler}  &$\frac{1}{1+\exp(-t)}$ & \xmark &\checkmark & \checkmark / \xmark & \checkmark / \xmark & \checkmark\\
		\hline 
		\eqref{generalsqFehler} & $\exp(t)$ & \xmark / \checkmark & \checkmark & \xmark & \xmark & \checkmark \\
		\eqref{generalsqFehler} &$\frac{1}{1+\exp(-t)}$ & \xmark & \checkmark & \checkmark / \xmark & \checkmark / \xmark & \checkmark\\
		\hline
		\eqref{generalsqrtsqFehler} & $\exp(t)$ & \xmark / \checkmark & \checkmark & \xmark & \xmark & \xmark \\
		\eqref{generalsqrtsqFehler} &$\frac{1}{1+\exp(-t)}$ & \xmark & \checkmark & \checkmark / \xmark & \checkmark / \xmark & \xmark\\
		\hline
		\eqref{generalsqrtFehler} & $\exp(t)$ & \xmark & \checkmark& \xmark & \xmark & \xmark \\
		\eqref{generalsqrtFehler} & $\frac{1}{1+\exp(-t)}$ & \xmark & \checkmark & \xmark & \xmark & \xmark \\
		\hline
		\eqref{generalexpFehler} & $\exp(t)$ & \xmark / \checkmark & \checkmark& \xmark & \xmark & \xmark \\
		\eqref{generalexpFehler} & $\frac{1}{1+\exp(-t)}$ & \xmark & \checkmark & \checkmark / \xmark & \checkmark / \xmark & \xmark \\
		\hline
		\eqref{insensitive1} & $\exp(t)$ & \xmark / \checkmark & \checkmark & \xmark & \xmark & \xmark \\
		\eqref{insensitive1}  &$\frac{1}{1+\exp(-t)}$ & \xmark & \checkmark & \checkmark / \xmark & \checkmark / \xmark & \xmark \\ 
		\hline
		\eqref{insensitive2} & $\exp(t)$ & \xmark / \checkmark & \checkmark & \xmark & \xmark & \xmark \\
		\eqref{insensitive2}  &$\frac{1}{1+\exp(-t)}$ & \xmark & \checkmark & \checkmark / \xmark & \checkmark / \xmark & \xmark \\ 
		\hline
		\eqref{robustinsens1} & $\exp(t)$ & \xmark & \checkmark & \checkmark & \checkmark & \xmark \\
		\eqref{robustinsens1} &$\frac{1}{1+\exp(-t)}$ & \xmark & \checkmark & \checkmark & \checkmark & \xmark\\
		\hline
		\eqref{robustinsens2} & $\exp(t)$ & \xmark & \checkmark & \checkmark & \checkmark & \xmark \\
		\eqref{robustinsens2} &$\frac{1}{1+\exp(-t)}$ & \xmark & \checkmark & \checkmark & \checkmark & \xmark \\
		\hline
		\eqref{robustlogcoshlogFehler} & $\exp(t)$ & \xmark & \checkmark & \checkmark & \checkmark & \checkmark \\
		\eqref{robustlogcoshlogFehler} &$\frac{1}{1+\exp(-t)}$ & \xmark & \checkmark & \checkmark & \checkmark & \checkmark \\
		\hline
		\eqref{smoothrobustFehler4} & $\exp(t)$ & \xmark & \checkmark & \checkmark & \checkmark & \checkmark \\
		\eqref{smoothrobustFehler4} &$\frac{1}{1+\exp(-t)}$ & \xmark & \checkmark & \checkmark & \checkmark & \checkmark\\
		\hline
		\eqref{smoothrobustFehler5} & $\exp(t)$ & \xmark & \checkmark & \checkmark & \checkmark & \checkmark \\
		\eqref{smoothrobustFehler5} &$\frac{1}{1+\exp(-t)}$ & \xmark & \checkmark & \checkmark & \checkmark & \checkmark \\
		\hline
		\eqref{gewFehler1} & $\exp(t)$ & \xmark / \checkmark & \checkmark & \xmark & \xmark & \xmark \\
		\eqref{gewFehler1} & $\frac{1}{1+\exp(-t)}$ & \xmark & \checkmark & \checkmark / \xmark & \checkmark / \xmark & \xmark \\
		\hline
		\eqref{gewFehler2} & $\exp(t)$ & \xmark / \checkmark & \checkmark & \xmark & \xmark & \checkmark \\
		\eqref{gewFehler2} & $\frac{1}{1+\exp(-t)}$ & \xmark & \checkmark & \checkmark / \xmark & \checkmark / \xmark & \checkmark\\
		\hline
		\eqref{gewFehler3} & $\exp(t)$ & \xmark / \checkmark & \checkmark & \xmark & \xmark & \checkmark \\
		\eqref{gewFehler3} & $\frac{1}{1+\exp(-t)}$ & \xmark & \checkmark & \checkmark / \xmark & \checkmark / \xmark & \checkmark \\
	\end{longtable}
}
Note that log-cosh-log loss and Huber-type logarithmic loss function combined with $c = 0$, $Y = (0, \infty)$, $u = \exp$ give ratio-based loss functions which fulfill all properties, i.e. convexity, Lipschitz continuity, and differentiability (cf. Proposition \ref{Prop_BspKonvexeVerlustfkt_Integral} and \ref{Prop_Lipschitz_Integral}). 
This is also true for the rb loss functions in the Examples \eqref{log(1+r)Loss} and \eqref{bsp_3_Loss}, respectively, from Chapter \ref{Kap_Eigenschaften}.
Dropping the requirement of differentiability, in this scenario ($Y = (0, \infty)$, $u = \exp$, $c = 0$), logarithmic absolute loss \eqref{logabsFehler} and logarithmic pinball loss \eqref{maxlogFehler} at least yield convex and Lipschitz continuous loss functions. \par
Roughly speaking, Table \ref{EigenschaftenL} shows that in most cases convexity and Lipschitz continuity are incompatible properties of rb loss functions as convexity often requires $c = 0$, whereas Lipschitz continuity needs $c > 0$.

\section{Connection to Distance-based Loss Functions} \label{Kap_Distance}
The question arises whether ratio-based loss functions are part of the class of distance-based loss functions or vice versa.
The constant function $L(x,y,t)=0$, for sure, is an element in both classes.
Furthermore, $L \equiv 0$ is the only \textit{constant} function, which can be both, ratio-based and distance-based, as both definitions require a certain value to be mapped to zero.
Obviously, $L \equiv 0$ can be noted in a distance- and in a ratio-based way using representing functions $\psi \equiv 0$ and $\ell\equiv 0$. 
\par 
Next, we consider a strict ratio-based loss (i.e. $c = 0$).
Because $\frac{u(t)}{y} \in (0, \infty)$, $y \in Y \subseteq (0, \infty)$, and $u: \R \to Y$ is measurable, we obviously have
\begin{equation*}
	\ell\biggl(\frac{u(t)}{y}\biggr) = \ell \biggl(\exp\Bigl(\log\Bigl(\frac{u(t)}{y}\Bigr)\Bigr)\biggr) = \ell\Bigl(\exp(\log(u(t))-\log(y))\Bigl).
\end{equation*}
Define
\begin{equation*}
	\psi: \R \to [0, \infty), \; \psi := \ell \circ \exp.
\end{equation*}
For this function, $\psi(0) = 0$ is satisfied. 
Therefore, $\psi$ can serve as a representing function of a dis\-tance\--based loss.
Together with $\tilde{y} := -\log(y)$ and $\tilde{t} := -\log(u(t))$,
\begin{equation*}
	\ell\biggl(\frac{u(t)}{y}\biggr) = \ell\Bigl(\exp(\log(u(t))-\log(y))\Bigr) = \psi(\tilde{y}-\tilde{t}).
\end{equation*}
This also holds for $c > 0$ using transformations $\tilde{y} := -\log(y+c)$ and $\tilde{t} := -\log(u(t)+c)$:
\begin{equation*}
	\ell\biggl(\frac{u(t)+c}{y+c}\biggr) = \psi(\tilde{y}-\tilde{t}).
\end{equation*}
Conversely, starting with a distance-based loss, whose representing function is $\psi: \R \to [0, \infty), \psi(0) = 0$, one can write
\begin{equation*}
	\psi(y-t) = \psi(\log(\exp(y-t))) = \psi\biggl(\log\Bigl(\frac{\exp(y)}{\exp(t)}\Bigr)\biggr).
\end{equation*}
Additionally, we define 
\begin{equation*}
	\ell: (0, \infty) \to [0, \infty), \; \ell := \psi \circ \log,
\end{equation*}
fulfilling $\ell(1) = 0$.
Hence, $\psi$ is a representing function of a distance-based loss.
Therefore, $\tilde{y} := \exp(-y)$ and $\tilde{t} := \exp(-t)$ in $(0, \infty)$ satisfy
\begin{equation*}
	\psi(y-t) = \psi\biggl(\log\biggl(\frac{\exp(y)}{\exp(t)}\biggr)\biggr) = \ell\biggl(\frac{\exp(-t)}{\exp(-y)}\biggr) = \ell\biggl(\frac{\tilde{t}}{\tilde{y}}\biggr).
\end{equation*}
Even though, we can find an obvious connection between ratio-based and distance-based loss functions, one can in general \emph{not} write a ratio-based loss as a distance-based loss in its \emph{original} arguments $(x,y,t)$, but only in its transformed arguments and vice versa. 
Additionally, as  \cite[Chapter 2]{ChenEtAl2010LARE} emphasized, \anf{a transformation is a reasonable choice in some cases due to its theoretical simplicity. However, a linear relationship in the transformed model is not linear in the original one. And one need[s] to transform the analysis results back to the original measurement scale}.
Recalling the medical example from the introduction, a 2 kg deviation does not give any hint about the ratio of true and predicted weight.
Remember the comparison between adult and child here.
The same holds true if it is only known that 80\% of the true weight was captured by the prediction. 
Then, one can conclude neither the true weight nor how far the prediction derives from the true value. \par
In the first chapter, we also saw that distance-based loss function cannot handle a ratio-based approach.
Conversely, assuming a ratio-based loss function which fulfills the distance-based property $\ell(\frac{\hat{y}}{y}) = \ell(\frac{\hat{y}+\lambda}{y+\lambda})$ for all $\hat{y}, y \in Y$ and $\lambda \geq 0$, i.e. under- or overestimation by a certain constant is penalized equally, implies a constant function.  \par 
Lastly, we also highlight these differences with some examples. 
For this, at first, consider a systematic error of the predictions $\hat{y}_i$ for all $i \in \{1, \dots, n\}$, which estimates the $\xi$-th of the observed value ($\xi > 0, \xi \neq 1$), i.e. $\hat{y}_i = \xi y_i$.
Then, $y_i-\hat{y}_i = (1-\xi)y_i$ for all $i \in \{1, \dots, n\}$.
Hence, $\psi(y_i-\hat{y}_i)$ does depend on $y_i$ in distance-based setting, whereas $\ell\Bigl(\frac{\xi y_i}{y_i}\Bigr) = \ell(\xi)$ in (strict) ratio-based setting is constant for all $i \in \{1, \dots, n\}$. 
Second, we think of the systematic error $\hat{y}_i = y_i + \xi$ for $\xi \in \R, \xi \neq 0$.
In this case, the prediction either over- or underestimates the true value by a certain amount.
Therefore, $\psi(y_i - \hat{y}_i) = \psi(\xi)$  is constant, but the ratio-based expression $\ell\Bigl(\frac{y_i+\xi}{y_i}\Bigr) = \ell\Bigl(1 + \frac{\xi}{y_i}\Bigr)$ is not.\par
From this short comparison of ratio-based loss functions and distance-based loss functions, it is clear that \emph{both} classes of loss functions are interesting and that the real-life application should be taken into account to decide which type of loss function is more appropriate.

\section{Alternative Definition} \label{Kap_Alternative}
One can ask why we defined ratio-based losses in the proposed way.
Instead of \eqref{ratiobased} we could have also considered
\begin{equation} \label{ratiobased_inverse}
	L(x,y,t) := \ell\biggl(\frac{y+c}{u(t)+c}\biggr), \qquad (x,y,t) \in \XY \times \R, 
\end{equation}
using the inverse quotient. 
In this short section, we argue that in most cases our slightly different notion of an rb loss function \eqref{ratiobased} yields the same properties, but sometimes the results need slightly weaker assumptions and the proofs are shorter. Of course, this does not matter if the rb loss function is based on a representation function $\ell(r)=\widetilde{\ell}(r)+\widetilde{\ell}(r^{-1})-2\widetilde{\ell}(1)$, see
Proposition \ref{Prop_konvexeVerlustfkt_tildeEll}. \par 
If not otherwise mentioned, the loss function $L$ will be defined \emph{in this section only} via {(\ref{ratiobased_inverse})}. \par
\begin{lemma} \label{altLstetig}
	For a loss function $L$ defined through \eqref{ratiobased_inverse}, let $\ell$ and $u$ be continuous. Then $L$ is a continuous loss function.
\end{lemma}
If $L$ is twice differentiable  w.r.t. to its last argument, for all
fixed values of $(x,y)$, we obtain
\begin{align*}
	L'(x,y,t) &= -\ell'\biggl(\frac{y+c}{u(t)+c}\biggr) \frac{y+c}{(u(t)+c)^2} u'(t), \\
	L''(x,y,t) &= -\ell''\biggl(\frac{y+c}{u(t)+c}\biggr) \! \biggl(\frac{(y+c)u'(t)}{(u(t)+c)^2}\biggr)^{\!2} + \ell'\biggl(\frac{y+c}{u(t)+c}\biggr)\frac{(y+c)}{(u(t)+c)^3} \bigl(u''(t)(u(t)+c)-2(u'(t))^2\bigr).
\end{align*}
\begin{remark}
	Let $L$ be a loss function defined via \eqref{ratiobased_inverse}.
	If $\ell$ and $u$ are ($n$ times) differentiable functions, then $L$ a ($n$ times) differentiable loss function, $n \in \N$.
\end{remark}
Looking at the second derivative, one can hardly derive a criterion for convexity. 
However, as with the first definition, we refer to Proposition \ref{Prop_konvexeVerlustfkt_tildeEll}.
\begin{lemma} \label{altLLipschitz}
	Let $Y = (a,b), 0 \leq a < b < \infty$, $c \geq 0$ with $a+c>0$. 
	Let $\ell$ and $u$ be Lipschitz continuous functions. 
	In that case, $L$ defined by \eqref{ratiobased_inverse} is a Lipschitz continuous loss function. 
\end{lemma}
Compared to the result in Chapter \ref{Kap_Eigenschaften}, we have an additional restriction here.
Whereas Chapter \ref{Kap_Eigenschaften} allows $b = \infty$, this is not possible here.
In contrast, Lemmata \ref{Lemma_Lipschitz_beschrAbl} and \ref{Lemma_Lipschitz_beschrAbl2} as well as Proposition \ref{Prop_Lipschitz_Integral} still hold.
Additionally, for $Y = (0,1)$ we only need a locally Lipschitz continuous function $\ell$.
\begin{lemma} \label{altLLipschitzY01}
	Let $L$ be a loss function of type \eqref{ratiobased_inverse}.
	Let $Y = (0,1), u(t) = \frac{1}{1+\exp(-t)}$, and $c > 0$. 
	Moreover, let $\ell$ be locally Lipschitz continuous.
	Then, $L$ is a Lipschitz continuous loss with Lipschitz constant $|L|_1 \leq \frac{|\ell|_{I,1} |u|_1 (1+c)}{c^2}, I := \Bigl(\frac{c}{1+c}, \frac{1+c}{c}\Bigr)$.
\end{lemma}
\begin{lemma} \label{altLLokalLipschitz}
	Assume $L$ like in \eqref{ratiobased_inverse}.
	Let $Y = (a,b), 0 \leq a < b < \infty$, and $a+c>0$. 
	Assume $\ell$ and $u$ to be locally Lipschitz continuous. 
	Then, $L$ is a locally Lipschitz continuous loss function.
\end{lemma}
Altogether, most results from Chapter \ref{Kap_Eigenschaften} apply to loss functions defined by {(\ref{ratiobased_inverse})} as well.
Only the result about global Lipschitz continuity has a certain restriction.
Hence, our original definition \eqref{ratiobased} of a ratio-based loss function has some minor advantages when compared with the definition  in 
\eqref{ratiobased_inverse}.
Calculations are easier and shorter in Chapter \ref{Kap_Eigenschaften}. \par
Furthermore, the former definition helps to control the risk in an easier manner. Indeed, the ratio 
\begin{equation*}
	\frac{y+c}{u(t)+c} \in \biggl(\frac{a+c}{b+c}, \frac{b+c}{a+c}\biggr) \subsetneq (0,\infty)
\end{equation*}
is bounded if $Y = (a,b), 0 \leq a < b < \infty$, and $c \geq 0$ such that $a+c>0$.
Thus, a continuous representing function $\ell$ bounds the risk $\RLP(f)$ of a measurable function $f: X \to \R$.
Moreover, if $Y = (0, \infty)$ and $\ell$ is bounded, the risk of a measurable function is finite. \par
When it comes to the calculation of the risk $\RLP(0)$, where $Y = (0, \infty)$, we, however, have to consider the ratio
\begin{equation*}
	\frac{y+c}{u(0)+c} \in \Bigl(\frac{a+c}{u(0)+c}, \infty \Bigr)
\end{equation*}
which has a lower, but no upper bound.
Even assuming a globally Lipschitz continuous $\ell$ gives
\begin{equation*}
	\biggl|\ell\Bigl(\frac{y+c}{u(0)+c}\Bigr)-\ell(1)\biggr| \leq |\ell|_1 \biggl|\frac{y-u(0)}{u(0)+c}\biggr|,
\end{equation*}
which can become arbitrarily large when taking the supremum over $y \in Y = (0, \infty)$.
In general, we can therefore not conclude that $\RLP(0)< \infty$.
Hence, using the inverse ratio does in this setting in general not bound the $L$-risk.

\section{Conclusion}
Many machine learning methods and AI algorithms are based on three cornerstones: (i) an appropriate function space  $\mathcal{F}$ often called hypothesis space, (ii) the set of  probability measures $\mathcal{P}$, and (iii) a loss function $L$ which is used to define the risk functional. \par
In supervised learning, margin-based loss functions for classification and distance-based loss functions for regression and quantile regression have been investigated in great detail by many authors. Distance-based loss functions are -- not only but of particular -- interest, if the loss between output values and its predictions depends on their \emph{difference}. 
This is obvious under the classical signal plus noise assumption.\par 
This paper focused on loss functions for supervised learning,
which depend on the \emph{ratio} of output values and its predictions. 
Such ratio-based loss functions are of particular interest if a multiplicative error structure seems to be plausible. 
Though, relative errors have been of concern for several years (see e.g. \cite{ChenEtAl2010LARE, ChenEtAl2016LPRE, JadonEtAl2022, TervenEtAl2025}) for multiplicative models, a systematic investigation of ratio-based loss functions has not been done in the literature to our best knowledge.
Therefore, the goal of this survey article was to put ratio-based loss functions into a more general framework.
We proposed a general definition, investigated their properties and those of the corresponding risk functionals, and collected several examples from the literature and proposed some new ratio-based loss functions.
We also showed that in general neither ratio-based loss functions are a subset of distance-based loss functions nor vice versa, although constant loss functions (which are completely uninteresting for practical applications) obviously belong to both sets of loss functions.\par
The goal of this paper was not to deduce learning rates, robustness results, etc. for a certain machine learning method.
Instead, we hope that our survey paper will stimulate research on several machine learning algorithms based on ratio-based loss functions in various directions.  
Such research may include finding additional convex and Lipschitz continuous ratio-based loss functions as well as deducing learning rates, representer theorems, or statistical robustness results (such as qualitative robustness, influence functions, and bounds for the maxbias) for certain machine learning methods including kernel based methods and CNNs.
Moreover, research in terms of non-convex loss functions can be crucial for some ML methods, too.
One can also think about mitigating the assumptions on the link function $u$ as well as expanding the ratio-based definition to other output spaces.
We are currently focusing on applying ratio-based loss functions to kernel-based ML methods and want to describe the set of all convex and Lipschitz continuous rb loss functions.
However, this is beyond the scope of this paper.

\printbibliography[heading = bibintoc, title={Bibliography}] 
\appendix
\section{Proofs}
\begin{proof}[\upshape \textbf{Proof of Proposition \ref{PropNoConvexLoss}}]
	Assume, that $L$ is a convex rb loss function.
	Let $(x,y) \in \XY$.
	Then, for all $\hat{t} \in \R$,
	\begin{equation} \label{monotoneAbb}
		s \mapsto \frac{L(x,y,s)-L(x,y,\hat{t})}{s-\hat{t}} = \frac{L(x,y,\hat{t})-L(x,y,s)}{\hat{t}-s}
	\end{equation}
	is increasing.
	Furthermore, there are $t < t' \in \R$ satisfying $L(x,y,t) < L(x,y,t')$.
	Define 
	\begin{equation*}
		q := \frac{L(x,y,t')-L(x,y,t)}{t'-t} > 0.
	\end{equation*}
	Let $t'' > t'$ and $\lambda = \frac{t'-t''}{t-t''} \in (0,1)$. Then $t' = \lambda t + (1-\lambda)t''$.
	Since $L$ is a convex loss function,  
	\begin{align*}
		L(x,y,t') &\leq \lambda L(x,y,t) + (1-\lambda) L(x,y,t'') \\
		\Leftrightarrow  \qquad \frac{L(x,y,t')-L(x,y,t'')}{t'-t''} &\geq \frac{L(x,y,t)-L(x,y,t'')}{t-t''} \overset{\eqref{monotoneAbb}}{\geq} \frac{L(x,y,t)-L(x,y,t')}{t-t'} = q.
	\end{align*}
	Hence,
	\begin{align*}
		L(x,y,t')-L(x,y,t'') &\leq q(t'-t'') \\
		\Leftrightarrow \qquad L(x,y,t'') &\geq L(x,y,t')+q(t''-t').
	\end{align*}
	As $q > 0$ and $t'' > t'$, $\lim_{t \to \infty} L(x,y,t) = \infty$ follows, giving a contradiction to $\lim_{t \to \infty} L(x,y,t) < \infty$ discussed in advance of Proposition \ref{PropNoConvexLoss}.
	Therefore, there is no convex loss $L$ in this setting.
\end{proof}
\begin{proof}[\upshape \textbf{Proof of Proposition \ref{Prop_konvexeVerlustfkt_tildeEll}}]
	Before proving the result, we show that $\ell$ is well-defined, i.e. $\ell \geq 0$.
	As each $r \in (0, \infty)$ can be written as $r = \e^t = u(t)$ with $t = \log(r) \in \R$, we define $f(t) := \widetilde{\ell}(\e^t)$.
	Then, 
	\begin{equation*}
		f''(t) = \widetilde{\ell}''(\e^t)\cdot (\e^t)^2 + \widetilde{\ell}'(\e^t) \cdot \e^t = \e^t \biggl(\widetilde{\ell}'(\e^t) + \e^t\widetilde{\ell}''(\e^t)\biggr).
	\end{equation*}
	As \eqref{Bedingung_Konvexität_tildeell} holds, $f''(t) \geq 0$ for all $t \in \R$, i.e. $f$ is convex.
	Therefore,
	\begin{equation*}
		\widetilde{\ell}(1) = f(0) \leq \frac{1}{2}f(t) + \frac{1}{2}f(-t) = \frac{1}{2} \widetilde{\ell}(r) + \frac{1}{2}\widetilde{\ell}(r^{-1})
	\end{equation*}
	gives the assertion. \par
	For $y \in Y$, we now define $q(t) := \frac{\e^t}{y}$, then $q'(t) = q(t)$ and the derivative of $\frac{1}{q(t)}$ is $-\frac{1}{q(t)}$.
	Hence,
	\begin{equation*}
		L''(x,y,t) = q(t) \Bigl(\widetilde{\ell}'(q(t)) + q(t)\widetilde{\ell}''(q(t))\Bigr) + \frac{1}{q(t)} \biggl(\widetilde{\ell}'\Bigl(\frac{1}{q(t)}\Bigr)+\frac{1}{q(t)} \widetilde{\ell}''\Bigl(\frac{1}{q(t)}\Bigr)\biggr).
	\end{equation*}
	Since $q(t), \frac{1}{q(t)} \in (0, \infty)$ for all $t \in \R$ and \eqref{Bedingung_Konvexität_tildeell} holds, $L''(x,y,t) \geq 0$.
	Thus, $L$ is a convex loss.
\end{proof}
\begin{proof}[\upshape \textbf{Proof of Proposition \ref{Prop_BspKonvexeVerlustfkt_Integral}}]
	Differentiating $\widetilde{\ell}$ yields $\widetilde{\ell}'(r) = \frac{g(r)}{r}$ by the Fundamental theorem of calculus. 
	Hence,
	\begin{equation*}
		\widetilde{\ell}'(r) + r\widetilde{\ell}''(r) = \frac{g(r)}{r} + r\Bigl(\frac{g'(r)}{r}-\frac{g(r)}{r^2}\Bigr) = g'(r).
	\end{equation*}
	As $g$ is an increasing function, $g'(r) \geq 0$ for all $r \in (0, \infty)$ implies \eqref{Bedingung_Konvexität_tildeell}. 
	According to Proposition \ref{Prop_konvexeVerlustfkt_tildeEll}, $L$ is a convex ratio-based loss function.
\end{proof}
\begin{proof}[\upshape \bfseries Proof of Lemma \ref{LipschitzL}] 
	For $(x,y) \in X \times Y$ and $t_1, t_2 \in \R$, we have
	\begin{equation*}
		|L(x, y, t_1) - L(x, y, t_2)| = \biggl| \ell\biggl( \frac{u(t_1) + c}{y + c}\biggr) -  \ell\biggl( \frac{u(t_2) + c}{y + c}\biggr)\biggr| \leq |\ell|_1 \cdot \frac{1}{|y+c|} \cdot |u|_1 \cdot|t_1 - t_2|.
	\end{equation*}
	$|\ell|_1$ and $|u|_1$ denote the Lipschitz constants of its particular functions.
	Hence, because $y \geq a$,
	\begin{equation*}
		\sup_{(x,y) \in X \times Y} |L(x, y, t_1) - L(x, y, t_2)| \leq |\ell|_1 \cdot |u|_1 \cdot \frac{1}{a + c} \cdot |t_1 - t_2|. \qedhere
	\end{equation*}
\end{proof}
\begin{proof}[\upshape\bfseries Proof of Lemma \ref{Lemma_Lipschitz_beschrAbl2}]
	Assume $L$ is a differentiable Lipschitz continuous loss function with uniform Lipschitz constant $K \geq 0$, i.e. there exists a constant $K \ge 0$ such that for all $t, t' \in \R$
	\begin{equation*}
		\sup_{(x,y) \in X \times Y} |L(x,y,t)-L(x,y,t')| \leq K |t-t'|.
	\end{equation*}
	Then, 
	\begin{equation*}
		|L'(x,y,t_0)| = \biggl|\lim_{t\to t_0} \frac{L(x,y,t_0)-L(x,y,t)}{t-t_0}\biggr| \leq \lim_{t\to t_0} \frac{\sup_{(x,y) \in \XY}|L(x,y,t_0)-L(x,y,t)|}{|t-t_0|} \leq K,\\
	\end{equation*}
	which yields a contradiction to $|L'|$ being unbounded.
	Thus, $L$ is not a Lipschitz continuous loss.
\end{proof}
\begin{proof}[\upshape \textbf{Proof of Proposition \ref{Prop_Lipschitz_Integral}}]
	According to Lemma \ref{Lemma_Lipschitz_beschrAbl}, $L$ is a Lipschitz continuous loss function because of
	\begin{equation*}
		|L'(x,y,t)| = \Bigl|g\Bigl(\frac{\e^t}{y}\Bigr)-g\Bigl(\frac{y}{\e^t}\Bigr)\Bigr| \leq 2M. \qedhere
	\end{equation*}
\end{proof}
\begin{proof}[\upshape \textbf{Proof of Lemma \ref{LipschitzY01}}]
	The ratio of interest $\frac{u(t)+c}{y+c}$ is an element in $I$. 
	Therefore,
	\begin{equation*}
		|L(x,y,t)-L(x,y,t')| = \biggl|\ell\biggl(\frac{u(t)+c}{y+c}\biggr)-\ell\biggl(\frac{u(t')+c}{y+c}\biggr)\biggl| \leq \frac{|\ell|_{I,1}}{y+c} |u(t)-u(t')|.
	\end{equation*}
	Hence, one can conclude Lipschitz continuity of the loss function:
	\begin{equation*}
		\sup_{(x,y) \in \XY} |L(x,y,t)-L(x,y,t')| \leq \frac{|\ell|_{I,1}|u|_1}{c} |t-t'|. \qedhere
	\end{equation*}
\end{proof}
\begin{proof}[\upshape\bfseries Proof of Lemma \ref{LocalLipschitzL}]
	Let $d \geq 0$ and $t_1, t_2 \in [-d, d]$.
	Since $u$ is monotonic, 
	\begin{equation*}
		0 < \check{m} \leq \frac{u(t_i)+c}{y+c} \leq \hat{m} < \infty
	\end{equation*}
	holds for $y \in Y$ and $i \in \{1,2\}$.
	Here, 
	\begin{equation*}
		\check{m} := \frac{\min\{u(-d), u(d)\}+c}{b + c} \in (0, \infty), \hspace*{2cm} \hat{m} := \frac{\max\{u(-d), u(d)\}+c}{a + c} \in (0, \infty),  
	\end{equation*}
	depending on $d$.
	As $u: \R \to Y$ is locally Lipschitz continuous, with $|u|_{d, 1}$ denoting the Lipschitz constant of $u$ on $[-d,d]$,
	\begin{equation*}
		|u(t_1) - u(t_2)| \leq |u|_{d,1} |t_1 - t_2|
	\end{equation*}
	holds. 
	Because $\ell: (0, \infty) \to [0, \infty)$ is locally Lipschitz continuous on $[\check{m}, \hat{m}] =: M \subseteq (0, \infty)$,
	\begin{equation*}
		\biggl|\ell\biggl(\frac{u(t_1)+c}{y+c}\biggr)-\ell\biggl(\frac{u(t_2)+c}{y+c}\biggr)\biggr| \leq |\ell|_{M,1} \biggl|\frac{u(t_1)-u(t_2)}{y+c}\biggr| \leq |\ell|_{M,1} \frac{|u|_{d,1}}{|y+c|} |t_1 - t_2| 
	\end{equation*}
	follows.
	As a result, the interval $M$ and its corresponding Lipschitz constant only depend on $d$.
	Thus,
	\begin{equation*}
		\sup_{(x,y) \in \XY} |L(x,y,t_1)-L(x,y,t_2)| \leq \frac{|\ell|_{M,1}|u|_{d,1}}{a+c} \cdot |t_1-t_2|. \qedhere
	\end{equation*}
\end{proof}
\begin{proof}[\upshape\bfseries Proof of Lemma \ref{LocalLipschitzL2}]
	Let $d \geq 0$ and $t_1, t_2 \in [-d,d]$. 
	Since $u$ is a monotone function,  
	\begin{equation*}
		0 < \frac{u(t_i)+c}{y+c} \leq m < \infty
	\end{equation*} 
	for $y \in Y$ and $i \in \{1,2\}$ with
	\begin{equation*}
		m :=  \frac{\max\{u(-d), u(d)\}+c}{a + c} \in (0, \infty)
	\end{equation*}
	depending on $d$.	
	As the continuation $\ell: [0, \infty) \to [0, \infty)$ is locally Lipschitz continuous, 
	\begin{align*}
		\biggl|\ell\biggl(\frac{u(t_1)+c}{y+c}\biggr)-\ell\biggl(\frac{u(t_2)+c}{y+c}\biggr)\biggr| \leq |\ell|_{[0,m],1} \cdot \biggl|\frac{u(t_1)-u(t_2)}{y+c}\biggr| 
	\end{align*}
	follows. 
	Because $u$ is locally Lipschitz continuous as well,
	\begin{align*}
		\biggl|\frac{u(t_1)-u(t_2)}{y+c}\biggr| \leq \frac{|u|_{d,1}}{|y+c|} |t_1 - t_2|.
	\end{align*}
	Therefore, 
	\begin{equation*}
		\sup_{(x,y) \in \XY} |L(x,y,t_1) - L(x,y,t_2)| \leq \frac{|\ell|_{[0,m],1} \cdot |u|_{d,1}}{a+c} |t_1-t_2|. \qedhere
	\end{equation*}
\end{proof}
\begin{proof}[\upshape\bfseries Proof of Lemma \ref{altLLipschitz}]
	Let $(x,y) \in \XY$ and $t_1, t_2 \in \R$.
	Since $\ell$ is Lipschitz continuous,
	\begin{equation*}
		\Bigl|L(x, y, t_1) - L(x, y, t_2)\Bigr| = \Bigl| \ell\biggl( \frac{y + c}{u(t_1) + c}\Bigr) -  \ell\Bigl( \frac{y + c}{u(t_2) + c}\Bigr)\Bigr| \leq |\ell|_1 \cdot |y + c|\cdot |u|_1 \cdot |t_1 - t_2| \cdot \frac{1}{(a+c)^2}.
	\end{equation*}
	As $0 < y < b < \infty$, we have
	\begin{equation*}
		\sup_{(x,y) \in X \times Y} |L(x, y, t_1) - L(x, y, t_2)| \leq |\ell|_1 \cdot |u|_1 \frac{b + c}{(a+c)^2} \cdot |t_1 - t_2|. \qedhere
	\end{equation*}
\end{proof}
\begin{proof}[\upshape\bfseries Proof of Lemma \ref{altLLipschitzY01}]
	We now focus on the ratio $\frac{y+c}{u(t)+c} \in I$.
	Since $u$ itself is globally Lipschitz continuous,
	\begin{equation*}
		|L(x,y,t)-L(x,y,t')| \leq |\ell|_{I,1} \Bigl| \frac{y+c}{u(t)+c} - \frac{y+c}{u(t')+c} \Bigr| \leq \frac{|\ell|_{I,1} |u|_1}{c^2} (y+c)|t-t'|
	\end{equation*}
	follows. Taking the supremum yields
	\begin{equation*}
		\sup_{(x,y) \in \XY} |L(x,y,t)-L(x,y,t')| \leq \frac{|\ell|_{I,1} |u|_1 (1+c)}{c^2}|t-t'|. \qedhere
	\end{equation*}
\end{proof}
\begin{proof}[\upshape\bfseries Proof of Lemma \ref{altLLokalLipschitz}]
	Let $d \geq 0, t, t' \in [-d,d]$.
	As $u$ is a monotone function,
	\begin{equation*}
		\frac{y+c}{u(t)+c}, \frac{y+c}{u(t')+c} \in [\check{M}_d, \hat{M}_d] =: M_d \subseteq (0, \infty)
	\end{equation*}
	with $\check{m}_d = \min\{u(-d), u(d)\}$, 
	$\hat{m}_d := \max\{u(-d), u(d)\}$, 
	$\check{M}_d := (a+c)/(\hat{m}_d+c)$,
	and
	$\hat{M}_d := (b + c)(\check{m}_d + c)$.
	Since $u$ is locally Lipschitz continuous, $|u(t)-u(t')| \leq |u|_{d,1} |t - t'|$.
	The local Lipschitz continuity of $\ell$ yields
	\begin{equation*}
		\biggl|\ell\biggl(\frac{y+c}{u(t)+c}\biggr)-\ell\biggl(\frac{y+c}{u(t')+c}\biggr)\biggr| \leq |\ell|_{M_d, 1} \biggl|\frac{y+c}{u(t)+c}-\frac{y+c}{u(t')+c}\biggr| \leq |\ell|_{M_d,1} |u|_{d,1} \frac{b+c}{(\check{m}_d+c)^2} |t-t'|.
	\end{equation*}
	Thus,
	\begin{equation*}
		\sup_{(x,y) \in \XY} |L(x,y,t)-L(x,y,t')| \leq|\ell|_{M_d,1} |u|_{d,1} \frac{b+c}{(\check{m}_d+c)^2} |t-t'|. \qedhere
	\end{equation*}
\end{proof}
	
\end{document}